\documentclass[twoside]{article}

\usepackage[accepted]{aistats2026}

\usepackage{verbatim}
\usepackage{booktabs}       
\usepackage{nicefrac}      
\usepackage{amsmath,amsthm,amssymb}
\usepackage{mathtools}
\usepackage{todonotes}
\usepackage{enumitem}
\usepackage{bm}
\usepackage{bbm}
\usepackage{dsfont}
\usepackage{todonotes}
\usepackage{booktabs}
\usepackage{float}
\usepackage{multirow}
\usepackage{subcaption}

\usepackage[utf8]{inputenc}

\usepackage[
  backend=biber,
  style=authoryear,     
  sorting=ynt,           
  sortcites=true,        
  giveninits=true,   
  maxcitenames=2,        
  mincitenames=1,
  maxbibnames=99,        
  uniquename=init,
  doi=false,isbn=false,url=false,
  eprint=false,
  hyperref=true
]{biblatex}

\renewbibmacro{in:}{}

\let\citep\parencite
\let\citet\textcite

\DeclareCiteCommand{\parencite}[\mkbibparens]
  {\usebibmacro{prenote}}
  {\bibhyperref{%
     \printnames{labelname}%
     \setunit{\nameyeardelim}%
     \printlabeldateextra}}
  {\multicitedelim}
  {\usebibmacro{postnote}}

\DeclareCiteCommand{\textcite}
  {\usebibmacro{prenote}}
  {\bibhyperref{%
     \printnames{labelname}%
     \setunit{\nameyeardelim}%
     \printlabeldateextra}}
  {\multicitedelim}
  {\usebibmacro{postnote}}

\usepackage{bm}
\usepackage{amsmath}
\usepackage{amssymb}
\usepackage{mathtools}
\usepackage{stmaryrd}

\usepackage{amsthm}

\usepackage[linesnumbered, ruled, vlined]{algorithm2e}
\definecolor{bluepurple}{RGB}{0, 0, 139}
\SetCommentSty{myCommentStyle}

\usepackage{algpseudocode}

\usepackage[mathscr]{eucal}
\usepackage{thmtools, thm-restate}
\usepackage{upgreek}
\let\what\widehat
\usepackage{relsize}
\SetKwProg{Fn}{\textsc{Function}}{}{end}

\newcommand{\bigint}{\mathlarger{\mathlarger{\int}}}
\newcommand{\mdpr}{\ensuremath{\mu}}
\newcommand{\mdpp}{\ensuremath{p}}
\newcommand{\kernforce}{r^\mathrm{f}}

\newcommand{\diff}{\ensuremath{\mathrm{d}}}
\newcommand{\myeqref}[1]{Eq.~\eqref{#1}}
\newcommand{\wrt}{w.r.t. }
\newcommand{\mdp}{\ensuremath{M}}
\newcommand{\empmdp}{\ensuremath{\hat M}}
\newcommand{\Mdp}{\ensuremath{\mathbf{M}}}
\newcommand{\mdq}{\ensuremath{\widetilde{{M}}}}
\newcommand{\mdqp}{\ensuremath{\Phi}}
\newcommand{\mdqr}{\ensuremath{\theta}}
\newcommand{\loss}{\ensuremath{r^\mathrm{exp}}}
\newcommand{\expb}[1]{\ensuremath{\exp\left\{  #1 \right\}}}

\DeclareMathOperator*{\glr}{GLR}

\DeclareMathOperator{\var}{Var}
\DeclareMathOperator{\dir}{Dir}
\let\underbrace\LaTeXunderbrace

\usepackage{etoolbox}
\DeclarePairedDelimiterX\norm[1]\lVert\rVert{
\ifblank{#1}{\:\cdot\:}{#1}
}

\DeclarePairedDelimiterX\Set[1]\{\}{%

#1
}
\DeclarePairedDelimiterXPP\lnorm[1]{}\lVert\rVert{_2}{\ifblank{#1}{\:\cdot\:}{#1}}
\DeclarePairedDelimiterXPP\pnorm[2]{}\lVert\rVert{_{#1}}{\ifblank{#2}{\:\cdot\:}{#2}}

\DeclarePairedDelimiterXPP\Prob[1]{\mathbb{P}}(){}{

#1}

\usepackage{hyperref}
\usepackage{cleveref} 
\usepackage{xcolor}
\definecolor{Green}{HTML}{228B22}
\hypersetup{colorlinks=true,citecolor=Green, urlcolor=Green, linkcolor=Green}

\newcommand{\algoname}{ 
 \hyperlink{REPIT}{PIPS}}

\newcommand{\closure}[0]{\mathrm{cl}}

\newtheorem{theorem}{Theorem}
\newtheorem{lemma}[theorem]{Lemma}
\newtheorem{proposition}[theorem]{Proposition}

\newtheorem{assumption}{Assumption}

\newcommand{\indp}[1]{\ensuremath{\mathbf{1}\!\left( #1\right)}}
\newcommand{\C}{\subset}

\newcommand{\bP}{\ensuremath{\mathrm{Pr}}} 

\newcommand{\bpol}{\ensuremath{\pi^\mathrm{exp}}} 
\newcommand{\bpolmix}{\ensuremath{\tilde{\pi}^\mathrm{exp}}} 
\newcommand{\bB}{\mathbb{B}}
\newcommand{\bE}{\mathbb{E}}

\newcommand{\bR}{\mathbb{R}}

\newcommand{\bN}{\mathbb{N}}

\newcommand{\cN}{\mathcal{N}}

\newcommand{\cX}{\mathcal{X}}
\newcommand{\cY}{\mathcal{Y}}

\newcommand{\cC}{\mathcal{C}}
\DeclareMathOperator{\KL}{KL}

\newcommand{\cS}{\mathcal{S}}
\newcommand{\cR}{\mathcal{R}}

\newcommand{\cA}{\mathcal{A}}
\newcommand{\cZ}{\mathcal{Z}}

\newcommand{\cL}{\mathcal{L}}

\newcommand{\cH}{\mathcal{H}}

\newcommand{\cO}{\mathcal{O}}
\newcommand{\cP}{\mathcal{P}}

\newcommand{\cE}{\mathcal{E}}

\newcommand{\veps}{\varepsilon}

\newcommand{\iid}{{ \it i.i.d }}

\newcommand{\impl}{\Longrightarrow}

\DeclareMathOperator*{\argmax}{argmax}
\DeclareMathOperator*{\argmin}{argmin}

\DeclareMathOperator{\vol}{{vol}}
\DeclareMathOperator{\alt}{\mathrm{Alt}}

\newcommand{\myexp}[1]{\ensuremath{{e}^{#1}}}

\newcommand{\muh}{\ensuremath{{\hat{\mu}}}} 
\newcommand{\mut}{\ensuremath{{\tilde{\mu}}}} 
\newcommand{\maxp}[2]{\ensuremath{\max\left( #1, #2 \right)}}

\DeclarePairedDelimiterX\innerp[2]{\langle}{\rangle}{#1,#2}

\newcommand{\lp}{\ensuremath{\left(}}
\newcommand{\rp}{\ensuremath{\right)}}
\newcommand{\lb}{\ensuremath{\left\{}}
\newcommand{\rb}{\ensuremath{\right\}}}
\newcommand{\lsb}{\ensuremath{\left[}}
\newcommand{\llvert}{\ensuremath{\left \lvert}}
\newcommand{\rrvert}{\ensuremath{\right \rvert}}
\newcommand{\rsb}{\ensuremath{\right]}}
\newcommand{\wt}{\ensuremath{\widetilde}}
\newcommand{\wh}{\ensuremath{\widehat}}

\let\geq\geqslant
\let\leq\leqslant

\newcommand{\ceq}{\ensuremath{\triangleq}}
\newcommand{\sect}{\Alph{section}}

\makeatletter
\newcommand{\arabicsectionlabell}[1]{%
  \protected@write\@auxout{}{%
    \string\newlabel{#1@arabic}{{\arabic{section}}{\thepage}}%
  }%
}
\makeatother

\begin{document}

\runningauthor{Cyrille Kone, Kevin Jamieson}

\twocolumn[

\aistatstitle{Optimal Posterior Sampling for Policy Identification in Tabular Markov Decision Processes}

\aistatsauthor{ Cyrille Kone \And Kevin Jamieson }

\aistatsaddress{ Univ. Lille, CNRS, Inria, Centrale Lille, \\ UMR 9189-CRIStAL, F-59000 Lille, France \And University of Washington,\\  Seattle, USA  } ]

\begin{abstract}
 We study the $(\varepsilon, \delta)$-PAC policy identification problem in finite-horizon episodic Markov Decision Processes. Existing approaches provide finite-time guarantees for approximate settings ($\varepsilon>0$) but suffer from high computational cost, rendering them hard to implement, and also suffer from suboptimal dependence on $\log(1/\delta)$. We propose a randomized and computationally efficient algorithm for best policy identification that combines posterior sampling with an online learning algorithm to guide exploration in the MDP. Our method achieves asymptotic optimality in sample complexity, also in terms of posterior contraction rate, and runs in $O(S^2AH)$ per episode, matching standard model-based approaches. Unlike prior algorithms such as MOCA and PEDEL, our guarantees remain meaningful in the asymptotic regime and avoid sub-optimal polynomial dependence on $\log(1/\delta)$.
Our results provide both theoretical insights and practical tools for efficient policy identification in tabular MDPs. 
\end{abstract}

\section{INTRODUCTION}
\label{sec:intro}
In online reinforcement learning (RL), an agent sequentially interacts with an unknown Markov decision process (MDP) by observing trajectories composed of transitions and rewards generated by the environment.
We are interested in the setting where interactions occur in episodes of finite length $H$: an episodic finite-horizon setting \citep{10.5555/528623}. Put formally, the MDP is described by $\mdp \ceq (\cS, \cA, H, \mdpp, \mdpr, s_\text{init})$ where $\cS$ is the state space of size $S$, $\cA$ is the action space of size $A$, $\mdpp$ and $\mdpr$ denote the transition and reward kernels, and $s_\mathrm{init}$ is the initial state. 

Each episode starts in the initial state $s_\mathrm{init}$. At stage $h \in [H]$, when in state $s_h \in \cS$, the agent selects an action $a_h \in \cA$ according to its strategy, collects a (random) reward $r_h$, and transitions to the (random) next state $s_{h+1}$ according to the environment dynamics. After $H$ steps, the environment resets to $s_\mathrm{init}$. The agent must learn to efficiently explore the environment in order either to maximize cumulative rewards across episodes \citep{JMLR:v11:jaksch10a} or, after an exploration phase, to identify a “good” policy that maximizes the expected sum of rewards per episode \citep{10.1145/180139.181019, NIPS1998_99adff45}. In this work, we focus on the latter objective, studied in the celebrated $(\veps, \delta)$-PAC (Probably Approximately Correct) policy identification setting, where the agent explores adaptively and eventually stops to recommend a policy it believes to be $\veps$-optimal. The goal is to keep the exploration phase as short as possible while guaranteeing the correctness of the recommendation.

In this framework, the agent follows adaptive exploratory policies and stops at a (random) time $\tau$, when it recommends a candidate policy $\hat{\pi}_\tau$ believed to be optimal. The objective is to minimize the exploration time while ensuring PAC correctness. This setting is a pure exploration problem, since performance is evaluated only through the final recommended policy, not during the exploration phase.

An algorithm for $(\veps,\delta)$-PAC policy identification typically has three components : (i) an adaptive exploration policy $\rho_t$ to gather new observations at episode $t$, (ii) a recommendation rule $\hat \pi_t$ representing the guess of the learning agent for the best policy and (iii) a stopping rule $\tau$ specifying when to terminate interaction with the environment.

Early work on $(\veps, \delta)$-PAC policy identification adapted algorithms from the regret-minimization setting by incorporating adaptive stopping rules, and established worst-case guarantees on the sample complexity \citep{10.5555/3042573.3042791,kaufmann21a,pmlr-v139-menard21a}. However, subsequent results \citep{pmlr-v178-wagenmaker22a, pmlr-v201-tirinzoni23a} proved such rates to be inherently suboptimal from a problem-dependent perspective. Recent research has therefore shifted toward identifying the exact complexity term governing policy identification. In this regime, racing-style elimination algorithms have been proposed \citep{NEURIPS2024_28795419, NEURIPS2022_39c60dda, pmlr-v178-wagenmaker22a, NEURIPS2022_27bf08fe}. These approaches often rely on sophisticated algorithmic designs or exhaustive enumeration of deterministic policies, which leads to significant computational challenges. Another line of work has introduced tracking-based algorithms \citep{pmlr-v139-marjani21a, NEURIPS2021_d9896106}; while promising, these methods are also computationally costly. Importantly, none of these approaches achieves instance-dependent optimality when $\delta$ is small.

In this work, we revisit the problem of $(\veps, \delta)$-PAC policy identification in the tabular episodic setting with a focus on the instance-dependent complexity when $\delta$ is small. We propose an algorithm that relies on posterior sampling and avoids costly policy enumeration and optimization steps required by many prior methods. 

 \subsection{Related Work}
 \label{sec:related_work}
 The earliest results on $(\veps, \delta)$-PAC policy identification focused on establishing worst-case sample complexity bounds, i.e., the dependence on $S, A, H, \veps$, and $\delta$ \citep{10.1145/180139.181019, NIPS1998_99adff45, 10.5555/3042573.3042791, 10.1007/s10994-013-5368-1, 10.5555/3174304.3175320, 10.5555/2969442.2969555, NEURIPS2018_bb03e43f, pmlr-v125-agarwal20b, pmlr-v139-menard21a}. Notably, the most recent of these results matches the lower bound $O\left(\tfrac{SAH^3}{\veps^2}\log(1/\delta)\right)$ proven by \citet{pmlr-v132-domingues21a}. The corresponding algorithms are often derived from optimistic methods originally designed for regret minimization \citep{pmlr-v70-azar17a} or from online-to-batch and regret-to-PAC reductions \citep{NEURIPS2018_d3b1fb02, pmlr-v201-tirinzoni23a}. While such worst-case guarantees provide a complete characterization of performance in the hardest environments, they do not capture the true, instance-dependent complexity of PAC policy identification. In particular, many MDP instances can be solved with far fewer samples than the worst-case bounds would suggest.

A series of works has recently attempted to characterize the true instance-dependent complexity of $(\veps, \delta)$-PAC policy identification. The seminal work of \citet{pmlr-v178-wagenmaker22a} notably observed that algorithms relying on the optimism principle, and, more generally, any algorithm satisfying good regret properties, cannot be instance-optimal in the $(\veps, \delta)$-PAC setting for all MDPs. The authors then introduced MOCA, a racing-type algorithm with action elimination, and showed that, with high probability, its sample complexity is upper-bounded by $C_\mathrm{MOCA}(\mdp, \veps) \log(1/\delta) + \text{poly}(\log(1/\delta), \log(1/\veps), SAH)/\veps,$ where $C_\mathrm{MOCA}(\mdp, \veps)$ scales with $H^2\sum_{h=1}^H \min_{\rho} \max_{s,a}\tfrac{1}{(\veps \lor \Delta_h(s,a))^2 w^\rho_h(s,a)}\:,$ and $w_h^\rho(s,a)$ is the probability that $(s,a)$ is visited at step $h$ when $\rho$ is played during an episode. $C_\mathrm{MOCA}(\mdp, \veps)$ is therefore identified as the leading complexity term when $\veps$ is small or when the suboptimality gaps are small. However, due to second-order terms, it may cease to be the dominating term when $\delta$ is small \citep{pmlr-v178-wagenmaker22a}.

\citet{NEURIPS2022_27bf08fe} proposed PEDEL, a policy-elimination algorithm that estimates policy values using a sophisticated G-optimal design. For tabular MDPs the sample complexity of PEDEL scales with $C_\mathrm{PEDEL}(\mdp, \veps) \ceq H^4 \sum_{h=1}^H {\min_\rho }\max_{\pi \in \Pi_\mathrm{det}} \sum_{s,a} \tfrac{w^\pi_h(s,a)^2}{w_h^\rho(s,a) (\veps \lor \Delta(\pi))^2}$ where $\Pi_\mathrm{det}$ is the set of deterministic (Markovian) policies and $\Delta(\pi)$ is the policy gap. See \citet{pmlr-v201-tirinzoni23a} for further discussion of the different gaps in $(\veps, \delta)$-PAC policy identification. 
By design, PEDEL requires enumerating all deterministic policies, which for finite-horizon MDPs is of size $A^{SH}$. The identified complexities for PEDEL and MOCA are generally not comparable.
 
\citet{al-marjani23a} introduced an algorithm whose identified complexity ranges between MOCA and PEDEL but suffers from a highly suboptimal polynomial dependence on $\log(1/\delta)$. 

\citet{NEURIPS2024_28795419} improved upon PEDEL by estimating not the individual policy values but the differences in policy values with respect to a reference policy. However, like PEDEL, their algorithm still requires enumerating the set of deterministic policies. 

\citet{NEURIPS2022_39c60dda} introduced EPRL, a fully adaptive action-elimination algorithm for MDPs with deterministic transitions. The sample complexity of EPRL also scales with the sum of the inverse squared suboptimality gaps of suboptimal reachable actions. The authors proved a lower bound on the sample complexity of a $(\veps, \delta)$-PAC algorithm for deterministic MDPs, showing that EPRL is optimal up to $O(H^2)$ terms. Importantly, this optimality guarantee holds only in the deterministic setting and does not extend to stochastic MDPs.
 
\citet{pmlr-v139-marjani21a} studied the instance-dependent complexity of $(\veps, \delta)$-PAC policy identification with a \emph{generative model} for $\veps = 0$ and stated the optimal complexity as a solution to an intensive min–max program, similar to the techniques popularized in \citet{garivier_optimal_2016}. The authors proposed an algorithm that solves a proxy of this problem. Applying the same technique with a \emph{forward model}, \citet{NEURIPS2021_d9896106} proposed MDP-NaS and proved a sample complexity asymptotically (as $\delta \to 0$) equivalent to $C(\mdp)\log(1/\delta)$, where $C(\mdp)$ is worse than $C_\text{MOCA}(\mdp, 0)$.

To date, all existing algorithms for policy identification either suffer from computational intractability, exhibit a suboptimal dependence on the leading term involving $\log(1/\delta)$ as $\delta \to 0$, or use regret minimization algorithms for their sampling rules.

\paragraph{Posterior Sampling and Top-Two Methods.}
Posterior sampling has emerged as a powerful tool for exploration in bandits and MDPs. Algorithms such as Top-Two Thompson Sampling \citep{russo_simple_2016, jourdan_top_2022, zhaoqi_peps} achieve optimal exploration in bandits by maintaining uncertainty over the true model. They appear to be statistically optimal, computationally efficient, and empirically competitive with other approaches. In RL, particularly for regret minimization, posterior sampling algorithms have become popular for designing efficient and tractable methods \citep{osband2013moreefficientreinforcementlearning, NIPS2017_3621f145, NEURIPS2019_451ae867, NEURIPS2022_45e15bae}. Yet, to the best of our knowledge, no posterior sampling algorithm has been analyzed for best policy identification in RL. We aim to close this gap by proposing an optimal and tractable algorithm for the episodic tabular setting. However, we emphasize that our goal is not merely to establish a sample complexity bound for an existing regret-minimization algorithm, as such algorithms are known to be suboptimal from an instance-dependent perspective (see above). Instead, we aim to design an instance-optimal algorithm that retains the computational efficiency of posterior sampling.

 \subsection{Main Contributions}
\label{sec:contributions}
In this paper, we focus on the theoretical analysis of the high-confidence regime (small $\delta$) and the exact identification case ($\varepsilon = 0$): 
\begin{itemize} 
    \item \emph{Algorithmic framework.} We introduce\algoname{}, a top-two, posterior-sampling-based algorithm for PAC policy identification, which breaks the policy-enumeration bottleneck by using posterior-driven exploration coupled with an online learner.   

    \item \emph{Theoretical guarantees.} We establish optimal instance-dependent sample complexity bounds in the high-confidence and exact identification regimes. In addition, we prove sharp posterior contraction results, providing a refined characterization of how uncertainty about the optimal policy decreases with the number of episodes. We further discuss the extension of \algoname{} to the $\veps>0$ case for which a simple modification of\algoname{} is proposed.  

    \item \emph{Empirical validation.} We complement our theory with experiments showing that our algorithm outperforms existing baselines, both in terms of sample efficiency and computational scalability.  
\end{itemize}
\section{PROBLEM FORMULATION}
\label{sec:setting}
We consider an episodic, finite-horizon Markov decision process (MDP) with finite state space $\cS$ of size $S$ and action space $\cA$ of size $A$. The MDP is defined by the tuple $\mdp  \ceq (\cS, \cA,  \{p_h\}_{h=1}^H, \{R_h\}_{h=1}^H, s_\mathrm{init})$ where $p_h$ and $R_h$ denote the transition and reward kernels at step $h$, and $s_\mathrm{init}$ is the initial state. 
A (Markovian) policy is a sequence $\pi = \{\pi_h\}_{h=1}^H$ with $\pi_h(\cdot  |  s)$ a distribution over actions given state $s \in \cS$. Under policy $\pi$, at each stage $h \in [H]$ the agent draws $a_h \sim \pi_h(s_h)$, receives reward $r_h \sim R_h(s_h, a_h)$ with mean $\mu_h(s_h, a_h) = \mu(s,a,h)$, and transitions to $s_{h+1} \sim p_h(\cdot  |  s_h, a_h) \equiv p(\cdot \mid s,a,h)$. A policy is called \emph{deterministic} if each $\pi_h(\cdot  |  s)$ assigns all probability mass to a single action denoted $\pi_h(s)$ or $\pi(s,h)$. 

\paragraph{Learning Problem.} Given the MDP $\mdp$, for any policy $\pi = \{\pi_h\}_{h=1}^H$, the $Q$-function is defined at each state-action as $Q_h^{\pi}(s, a) \ceq \bE_\mdp^{\pi}\bigl[\sum_{k=h}^H r_k  |  s_h=s, a_h=a\bigr]\:,$ i.e., the expected accumulated reward when starting from $(s,a)$ at step $h$ and following $\pi$ thereafter. The corresponding state-value function is $V_h^\pi(s) \ceq \bE_\mdp^{\pi} \bigl[\sum_{k=h}^H r_k  |  s_h = s \bigl] $ and we write $V_0^\pi \ceq V_1^\pi(s_\mathrm{init})$ for the expected return when starting from the initial state.  Let $\Pi$ denote the set of all policies (deterministic and stochastic). The learner’s goal is to identify an optimal policy
$\pi^\star \in \argmax_{\pi \in \Pi } V_0^\pi$ using as few episodes as possible. For an MDP $\Mdp$, we let $\Pi^\star(\mdp)$ (resp. $\Pi_\mathrm{det}^\star(\mdp)$) denote the set of optimal (resp. deterministic optimal) policies, and drop the dependence on $\mdp$ when clear from context. Finally, we denote by $\{\cH_t\}_{t \geq 1}$ the filtration of the learning process, i.e., the information collected up to the end of the $t$-th episode.  
\begin{assumption}
The MDP $\mdp$ admits a unique optimal policy $\pi^\star$. 
\end{assumption} 

We assume normally distributed rewards, i.e. $R_h(s,a) \sim \cN(\mu(s,a, h), 1)$, and we further assume without loss of generality that the expected reward $\mu(s,a, h) \in (0, 1)$. Any tabular MDP in this parametric set is uniquely identified by its transition probabilities $p\ceq \{p_{h}\}_{h=1}^H$ and expected reward functions $\mu \ceq \{ \mu_h\}_{h=1}^H$. Maintaining a distribution over these parameters is therefore equivalent to maintaining a distribution over MDP models. 
In the sequel, when the context is clear, we identify an MDP by its transition probabilities $p$ and expected rewards $\mu$. We denote by $\Mdp$ the set of such parametric MDPs. 

\paragraph{Empirical MDP.} For $t \geq 1$, let $n_t(s' \mid s,a,h) \,\ceq\, \sum_{i=1}^t \indp{ s_h^i = s,\, a_h^i = a,\, s_{h+1}^i = s'}$ be the number of transitions to state $s'$ after taking action $a$ in state $s$ at step $h$ during the first $t$ episodes. We also define $n_t(s,a, h) \,\ceq\, \sum_{s'} n_t(s'\mid s,a,h) \,=\, \sum_{k=1}^{t-1} \indp{ s_h^k = s,\, a_h^k = a},$ the total number of visits to $(s,a,h)$ up to the end of the $t$-th episode. The empirical transition probability at $(s,a,h)$ is then given by
$$ \hat p_t(s' | s,a,h) \,\ceq\,
\begin{cases}
\dfrac{n_t(s'\mid s,a, h)}{n_t(s,a, h)}, & \text{if } n_t(s,a, h) > 0, \\[1.2ex]
\dfrac{1}{S}, & \text{otherwise}.
\end{cases} $$
We denote by $\muh_t(s,a,h)$ the empirical mean reward at $(s,a,h)$. The empirical reward and transition kernels $(\hat p_t, \muh_t)$ define a unique MDP in $\Mdp$, which we denote by $\empmdp_t$. For any policy $\pi \in \Pi$, we write $(\hat V_{t,h}^\pi)$ and $(\hat Q_{t, h}^\pi)$ for the value and $Q$-functions of $\pi$ in $\empmdp_t$. 

\paragraph{Prior and Posterior over MDPs.}
Our distribution over MDPs is defined over the parameters that uniquely determine the model. For transitions, independent Dirichlet priors are considered, which, by conjugacy, yield Dirichlet posteriors. For rewards, we assign independent uniform priors on $(0,1)$, leading to truncated normal posteriors. After observing an episode $\{(s_h^t, a_h^t, r_h^t, s_{h+1}^t)\}_{h=1}^H$, the posterior is updated via Bayes’ rule. Direct computation gives the closed form  
\begin{equation}
\label{eq:posterior_def}
   \!\!\bP(p, \mu | \cH_{t-1})
    = \!\!\prod_{s,a,h}\!\!f_{s,a,h}(\mu_h(s,a))  g_{s,a,h}(p_h(\cdot  |  s,a)),\!\!
    \end{equation}
    where $f_{s,a,h}$ is the density of density of the truncated normal $\cN_{[0,1]}\!\bigl(\hat\mu_{t}(s,a,h), \tfrac{1}{n_{t}(s,a, h)}\bigr),$ and $g_{s,a,h}$ is the density of $\dir\!\bigl(\bm u + n_{t}(\cdot \mid s,a,h)\bigr)$, the Dirichlet with parameter $\bm + n_{t}(\cdot \mid s,a,h),$ for $\bm\in \bR^d_+$, that we will specify latter.  

Our goal is not to develop a full Bayesian model; the posterior primarily serves as an algorithmic tool to guide exploration. 

\paragraph{Information-theoretic Lower Bounds.} 
For two probability distributions $P,Q$ supported on a discrete space $\cS$, the Kullback–Leibler (KL) divergence is  
$\KL(P \,\Vert\, Q) \ceq \sum_{s \in \cS} P(s) \log \frac{P(s)}{Q(s)}$.  
For two MDPs $\mdp, \mdp'$, we say that $\mdp$ is absolutely continuous with respect to $\mdp'$ (denoted $\mdp \ll \mdp'$) if for every $(s,a,h)$, $p_h(\cdot  |  s,a) \ll p_h'(\cdot  |  s,a)$ and $R_h(s,a) \ll R_h'(s,a)$. Given an MDP $\mdp$, define  
$ \alt(\mdp) \,\ceq\, \bigl\{ \mdp' \in \Mdp : \mdp \ll \mdp' \;\; \text{and} \;\; \Pi^\star(\mdp) \cap \Pi^\star(\mdp') = \emptyset \bigr\},$ the set of alternative MDPs $\mdp'$ such that $\mdp$ is absolutely continuous with respect to $\mdp'$ but where no policy optimal in $\mdp$ remains optimal in $\mdp'$. We denote by $(\KL(\mdp \,\|\, \mdp'))_{s,a,h}$ the KL divergence between the transition and reward kernels of $\mdp$ and $\mdp'$. For each $(s,a,h)$, introducing $\mdp_{s,a,h} \ceq R_h(s,a) \otimes p_h(\cdot | s,a)$, 
\begin{align*}
   \KL\!(\!\mdp\Vert \mdp\!')_{s,a,h} \!&= \KL\!\bigl(\mdp_{s,a,h} \Vert \mdp_{s,a,h}'\bigr) \\
    \!&= \KL\!\bigl(R_{s,a,h} \Vert R_{s,a,h}'\!\bigr) +  \KL\!\bigl(\!p_{s,a,h} \Vert p_{s,a,h}'\!\bigr)\!,
\end{align*}
where $p_{s,a,h} \ceq p_h(\cdot | s,a)$ and $R_{s,a,h} \ceq R_h(s,a)$. 

\indent
An algorithm is said to be {$\delta$-PAC for best policy identification (BPI)} if its stopping time $\tau$ and recommendation rule $\hat\pi_\tau$ satisfy $$ \bP_{\mdp}\!\bigl(\tau < \infty, \; \hat\pi_\tau \neq \pi^\star\bigr) \;\leq\; \delta, \;\; \forall\, \mdp \in \Mdp \:. $$
Let $\Omega_\mdp$ denote the set of visitation probabilities induced by Markovian policies on $\mdp$, $\Omega_\mdp \ceq
   \{ \{w_h^\pi(s,a)\}_{s,a,h} :  \pi \in \Pi \}$. 

From the information-contraction principle and change-of-distribution techniques (as popularized in \citet{garivier_optimal_2016}), the following holds.  
\begin{restatable}{theorem}{lbdSc}
\label{thm:lbd-sc}
Any $\delta$-PAC algorithm with stopping time $\tau$ satisfies, 
\begin{align*}
	&\qquad\qquad\bE_\mdp[\tau]
    \,\geq\, \Gamma_\mdp^{-1} \log \frac{1}{2.4\delta}, \;\; \text{where} \\
    &\Gamma_\mdp \ceq 
    \sup_{w \in \Omega_\mdp} 
    \inf_{\mdq \in \alt(\mdp)} 
    \Biggl[ \sum_{s,a,h} w_h(s,a) \KL\!(\mdp \Vert \mdq)_{(s,a,h)} \!\Biggr]. 	
	\end{align*}
\end{restatable}
 \medskip
\noindent
We can express $\Gamma_\mdp$ as  
\begin{equation}
\label{eq:game}
    \Gamma_\mdp = 
    \max_{\bpol} 
    \!\!\!\inf_{\mdq \in \alt(\mdp)} 
    \!\!\bE_\mdp^{\bpol}\!\lsb\sum_{h=1}^H \KL(\mdp \Vert \mdq)_{s_h,a_h,h}  \rsb, \!
\end{equation}
where the expectation is taken over trajectories generated by policy $\rho$ with dynamics governed by $\mdp$ and the KL terms act as deterministic rewards. The problem in Eq.\eqref{eq:game} can be interpreted as the value of a two-player zero-sum game: the learner chooses an exploration policy $\bpol$ to maximize the accumulated information (measured by KL terms), while an adversary selects an alternative MDP $\mdq$ (with different optimal policies than in $\mdp$) to minimize it. 

The quantity in Eq.\eqref{eq:game} not only characterizes the sample complexity of BPI but also governs the contraction rate of the posterior distribution of any adaptive algorithm for BPI.  
\begin{restatable}{theorem}{postLbdMain}
\label{thm:opt_post_contra}
Fix an MDP $\mdp \in \Mdp$ with optimal policy $\pi^\star$ and let $\nu_t$ denote the posterior distribution defined in \myeqref{eq:posterior_def}. For any adaptive algorithm for best policy identification, with probability one,  
$$ \limsup_{t \to \infty} 
    -\tfrac{1}{t} 
    \log \bP_{\wt\mdp \sim \nu_t \,|\, \cH_{t-1}}
        \!\bigl(\pi^\star \notin \Pi^\star(\mdq)\bigr) 
    \;\leq\; \Gamma_\mdp \:. $$
\end{restatable}

In words, for any adaptive algorithm, the posterior probability of misidentifying the optimal policy decays at most as $\myexp{-t \Gamma_\mdp}$ asymptotically in $t$. While analogous results are known in the multi-armed bandit setting \citep{russo_simple_2016}, our result is novel in the context of tabular MDPs (see Appendix~\ref{appx:posterior_convergence} for a full proof). 
\section{POLICY IDENTIFICATION VIA POSTERIOR SAMPLING}
\label{sec:algo}
Our algorithm uses the posterior distribution to guide exploration in the environment. At episode $t$, the learner computes the optimal policy of the empirical MDP $\empmdp_{t}$, denoted by $\hat \pi_t$, based on all data observed up to episode $t$. The core idea is then to construct a \emph{challenger} MDP $\mdq_{t}$, sampled from an inflated posterior distribution, such that $\hat \pi_t$ is not optimal in $\mdq_{t}$. This challenger highlights uncertainty about the optimal policy and directs further exploration.  

Formally, let $1/\eta_t$ be the inflation parameter. We define the density of the inflated posterior distribution over MDPs as  
\begin{equation*}
\nu^{\eta_t}_t (p, \mu) 
    \,\ceq\,  
    \prod_{s,a,h} f^t_{s,a,h}(\mu(s,a, h)) \,g^t_{s,a,h}(p(\cdot  |  s,a, h)) \:, 	
\end{equation*}
where $f^t_{s,a,h}$ is the density of inflated $N_{[0,1]}(\muh_{t}(s,a, h), \frac{1}{\eta_t n_{t}(s,a, h)})$ and $g^t_{s,a,h}$ is the density of the inflated Dirichlet $ \dir\!\lp \bm u + \eta_t n_{t}(\cdot \mid s,a,h)\rp$,
and we set the prior hyperparameter to $\bm u = (1,\dots,1)$. The challenger MDP $\mdq_{t}$ is generated via conditional posterior sampling: we draw candidate MDPs $\mdq_{t, 1}, \mdq_{t, 2}, \dots$ \iid from $\nu^{\eta_t}_t$ and select the first one for which $\hat\pi_t$ is not optimal. Concretely, for each sampled MDP $\mdq_{t, k}$, we compute its optimal $Q$-function $\wt Q_{t, h}^{k, \star}$ by backward induction \citep{10.5555/528623} and then check whether $\hat\pi_t$ is optimal for this $Q$-function. As soon as a sample is found where $\hat\pi_t$ is suboptimal, this instance is designated as the challenger MDP $\mdq_{t}$. 

This sampling procedure naturally favors randomness in insufficiently explored state-action pairs, where deviations from the empirical model are most likely to produce a different optimal policy. Consequently, the challenger MDP $\mdq_{t}$ serves as an informative alternative to the current estimate, concentrating exploration on the state-action pairs that are most relevant for distinguishing optimal from suboptimal policies. Once the challenger is obtained, an online RL learner directs exploration toward the regions where the empirical model and the challenger differ the most, as measured by their KL divergence.

\paragraph{Exploration Policy.} 
 The learner maintains an online RL reward-maximization strategy over the space of exploration policies to ensure that sufficient evidence is collected to distinguish between the empirical MDP $\empmdp$ and the challenger MDP $\mdq$. An initial exploration policy $\bpol_1$ is given. At the end of each episode, the exploration policy is updated via a policy improvement step performed by a mirror descent learner. By letting $\mdq_{t} \ceq (\mdqr_{t}, \mdqp_{t})$, we define the exploration reward kernel as 
 
 \begin{equation}
\label{eq:main-loss-def}
\begin{aligned}
&\quad \loss_t(s,a, h) \,\ceq\, 
  \frac{\bigl(\muh_{t}(s,a, h)-\mdqr_{t}(s,a,h)\bigr)^2}{2} \;\; +\;\\
&\;\;\sum_{s'} \hat p_{t}\!(s'\mid s,a,h)
  \log\frac{\hat p_{t}(s' \mid s,a,h)}%
           {\maxp{\mdqp_{t}(s'\mid s,a,h)}{\myexp{-t^\alpha}}} \: .
\end{aligned}
\end{equation}
which corresponds to the KL divergence between the marginals $\mdq_{t}(s,a, h)$ and $\empmdp_{t}(s,a,h)$ with transition in $\mdq_{t}$ clipped to $\myexp{-t^\alpha}$ to avoid singularities and control the magnitude of the reward. An optimistic policy evaluation step estimates the $Q$-function of $\bpol_t$ in the MDP with reward kernel $\loss_t$ and (unknown) transition kernel $p$ (hence the use of optimism). Concretely, using Bernstein type confidence bonuses, and initializing $\bar V_{t, H+1}(\cdot) = 0$, the optimisitic $Q$-function is recursively evaluated as 
\begin{align}
\label{eq:Q-update}
 \begin{split}
 	&\bar  Q_{t, h}(s,a) \ceq \loss_t(s,a,h) \;+\; \hat p_{t,h} \bar V_{t, h+1}(s,a) \; + \;  \\ & \quad \frac{2 g_{t,h} \beta^p(n_t(s,a,h), 1/t^3)}{3 \maxp{1}{ n_t(s,a,h)}} \; +\;
\\ &\sqrt{\frac{2 \, \var_{\hat p_{t, h}}\bigl[\bar V_{t, h+1}\bigr](s,a) \, \cdot \beta^p(n_t(s,a,h),\,1/t^3)}{\maxp{1}{ n_t(s,a,h)}}}\:, 
 \end{split}
 \end{align}
with the associated value function $\bar V_{t, h+1}(s) \ceq \sum_{a}{\bar  Q_{t, h}(s,a)}{\rho_h(a | s)}$ and $g_{t,h} \ceq \max_{s} \bar V_{t, h+1}(s)$. We introduce $\hat p_{t, h} \bar V_{t, h+1}(s,a) \ceq  \bE_{s' \sim \hat p_{t, h}(\cdot  |  s,a)}\bigl[\bar V_{t, h+1}(s')\bigr],$ and $\var_{\hat p_{t, h}}\bigl[\bar  V_{t, h+1}\bigr](s,a) = \var_{s' \sim \hat p_{t, h}(\cdot  |  s,a)}(\bar V_{t, h+1}(s'))$. 
As proven in Appendix~\ref{appx:sup_player}, for some properly tuned threshold $\beta^p$, the $Q$-function computed in Eq.\eqref{eq:Q-update} is, with probability at least $1-1/t^3$, an optimistic estimate of $Q_{t,h}^{\bpol_t}(s,a)$, the $Q$-function of $\bpol_t$ in $(\loss_t, p)$. Next, one step of policy improvement via an online learner on the $\cA$-simplex.  We maintain $SH$ independent online learners 
$(\mathcal{L}_{s,h})_{s,h}$ on the simplex $\triangle_A$ (distributions over actions). 

The policy improvement step updates each online learner to improve the exploration policy for each $(s, h)$ via the estimated exploration $Q$-function. Notable choices include AdaHedge and other instances of mirror descent learners.  This is analyzed in Appendix~\ref{appx:sup_player}.

  To correct for potential empirical bias in initial estimates of some state-action pairs, the exploration policy $\bpol_t$ is mixed with a forced exploration policy $c_t$, which ensures that any reachable state will be explored sufficiently. Given some parameter $\gamma \in (0, 1)$, at episode $t$, with probability $t^{-\gamma}$, the behaviour policy $\bpol_t$ is replaced a forced exploration policy $c_t$ (explained below). This ensures broad coverage while keeping the forced exploration negligible: the total number of episodes allocated is sublinear in $t$, and for $\gamma$ close to $1$, only $O(\log t)$ episodes are allocated to forced exploration, letting the exploration mainly be led by \textsc{ComputeExplorationPolicy}. 

The forced exploration policy performs a single iteration update of UCBVI \citep{pmlr-v70-azar17a} with a reward function designed as an indicator of a randomly picked state-action pair. A triplet $(\tilde s^t, \tilde a^t, \tilde h^t)$ is selected randomly and a reward kernel is defined as $\kernforce_t: (s,a,h) \mapsto \indp{s=\tilde s^t, a=\tilde a^t, h=\tilde h^t}$ and $c_{t+1}$ is defined as the greedy policy wrt an optimistic estimate of the $Q$-function for the reward kernel $\kernforce_t$. The procedure is described and analyzed in Appendix~\ref{appx:sufficient_explore}. 

\noindent In Appendix~\ref{appx:extension}, we describe a simple modification to construct the challenger MDP when the goal is to identify an $\veps$-optimal policy (with $\veps > 0$).

\paragraph{Computational Cost.}
The computational cost of Algorithm~\ref{alg:repit} is dominated by the sampling and planning steps performed at each episode. Concretely, the algorithm generates $SAH$ Dirichlet samples of dimension $S$ and $SAH$ samples from a normal distribution.  
It then verifies whether the sampled MDP and the empirical MDP share the same optimal policy, which can be done via backward induction \citep{10.5555/528623} in time $O(S^2AH)$. 
Altogether, the per-episode computational complexity is $O(S^2AH)$, with a memory requirement of $O(S^2AH)$, consistent with standard model-based approaches. 
\begin{algorithm}[h!]
\DontPrintSemicolon
\LinesNumbered
\SetAlgoLined
\SetKwInOut{Input}{Require}
\SetKwInOut{Output}{Output}
\Input{Initial exploration $\bpol_1 = \{\bpol_1(\cdot | s, h)\}_{s, h}$; initial forced exploration $c_1$; mixing parameter $\gamma \in (0,1)$; clipping parameter $\alpha \in (0,\tfrac12)$; posterior inflation $(\eta_t)_{t\geq 1}$; initial MDP $\empmdp_{0} \in \Mdp$ (arbitrary)} 

\BlankLine
\For{episode $t = 1,2,\ldots$}{
  \tcp{Empirical best policy}
  Compute $\hat\pi_{\,t} \in \arg\max_{\pi} \hat V_{t, 0}^{\pi}$ \; 
  \BlankLine
  \tcp{Computing a challenger MDP}
  \For{$k = 1,2,\ldots$}{\label{algo:sample-challenger}
    Sample candidate challenger $\mdq_{t, k} \sim \nu^{\eta_t}_{t}$\;
    Compute the optimal $Q$-function $\;\wt Q_{t}^{k,\star}$ of $\mdq_{t, k}$ by backward induction\;
    \If{$\hat\pi_{t}$ is suboptimal for $\;\wt Q_{t}^{k,\star}$}{
      Set challenger $\mdq_{t} \gets \mdq_{t, k}$ and \textbf{break}\;
    }
  }
  \BlankLine
  
 Compute $c_{t}$ as in Algorithm~\ref{alg:appx_forcexp}\; 
    
  Draw $Z_t \sim \mathrm{Bernoulli}(t^{-\gamma})$ and define $\forall (s,h)\in \cS \times [H]$ 
  $$
     \bpolmix_t(\cdot | s, h) \;\leftarrow\;
     \begin{cases}
       c_t(\cdot | s, h), & \text{if } Z_t = 1,\\
       \bpol_t(\cdot | s, h), & \text{if } Z_t = 0,
     \end{cases}
  $$
  \BlankLine
  \tcp{Rollout and data collection}
  Set $s_1^t \gets s_{\mathrm{init}}$\;
  \For{$h = 1$ \KwTo $H$}{
    Sample action $a_h^t \sim \bpolmix_t(\cdot | s_h^t, h)$\;
    Observe reward $r_h^t$ and transition to $s_{h+1}^t$\;
  }
  \BlankLine
  \tcp{KL-based bonuses for the online learner}
  Construct bonus kernel $\loss_t(s,a, h)$ from $\empmdp_{t}$ and $\mdq_{t}$ as in Eq.\eqref{eq:main-loss-def}\; 
  \BlankLine
  \tcp{Behaviour policy improvement}
  Update main exploration policy $\bpol_{t+1} \gets \textsc{ComputeExplorationPolicy}\bigl(t, \bpol_t,\, \loss_t,\, \hat p_t,\, n_t \bigr)$\; 

}
\caption{\textsc{PIPS}: Policy Identification via Posterior Sampling in Tabular MDPs  \label{alg:repit}}
\end{algorithm}
\begin{algorithm}[h!]
\Fn{\textsc{ComputeExplorationPolicy}$(t, \bpol_t, \loss_t, \hat p_t, n_t)$}{
\DontPrintSemicolon
\LinesNumbered
\SetAlgoLined
\SetKwInOut{Input}{Require}
\SetKwInOut{Output}{Output}
\Input{Independent online learning instances $(\cL_{s, h})_{s, h}$}
\BlankLine
Set $\bar V_{t, H+1}(s) \leftarrow 0 \quad \ \forall \; s \in \cS$\;
\For{episode stage $h = H, H-1, \ldots, 1$}{
  \ForEach{$(s,a) \in \cS \times \cA$}{
    $\bar Q_{t, h}(s,a) \;\leftarrow\; \text{Eq.\eqref{eq:Q-update}} $ 
  }
  \tcp{Optimistic policy evaluation }
    $\bar V_{t, h}(s) \;\leftarrow\; \sum_{a \in \cA} \bpol_t(a  |  s,h)\, \bar Q_{t, h}(s,a) \quad \forall\; s \in \cS$
}
\BlankLine
  \tcp{Policy improvement via online Mirror descent learner}
  
  \ForEach{$(s,h) \in \cS \times [H]$}{
  Feed learner $\cL_{s,h}$ with gain $\bar {Q}_{t, h}(s, \cdot)$ \; 
    Update $\bpol_{t+1}( \cdot \mid s, h)$ from $\cL_{s, h}$ \; }

\BlankLine
\Return policy $\bpol_{t+1}$\;
\caption{\textsc{Exploration Policy Helper Function}
\label{alg:ol}} }
\end{algorithm}
\section{MAIN THEORETICAL RESULTS}
\label{sec:mainres}
In this section, we present the guarantees attained by\algoname{} both in terms of posterior contraction rate and sample complexity upper bound. 

Our first result characterizes the posterior contraction rate.  If we were to sample an MDP from the posterior the "Bayesian'' error would be $ \bP_{\mdq \sim \nu_t | \cH_{t-1}}(\pi^\star \notin \Pi^\star(\mdq)) = \int_{\alt(\mdp)} \diff  \nu_t((p, \mu))$ which by Theorem~\ref{thm:opt_post_contra} is asymptotically larger than $\myexp {- t \Gamma_\mdp}$ for any adaptive algorithm.  
The result below shows that\algoname{} reaches this optimal rate. 
\begin{restatable}{theorem}{postMain}
	Consider an MDP $\mdp \in \Mdp$ with a unique optimal policy $\pi^\star$. Let $\gamma \in (0, 1)$, $\alpha \in (0, 1/2)$, $\alpha < \gamma$, $\varsigma \in (0, 1)$, with $\varsigma> 2\alpha $. Run with parameterss $\gamma, \alpha$ and learning rate $\eta_t = O(t^{-\varsigma})$, letting $\nu_t$ denote the non-inflated posterior distribution, Algorithm~\ref{alg:repit} satisfies with probability one  
	$$  \limsup_{t \to \infty} 
    -\frac{1}{t} 
    \log \bP_{\wt\mdp \sim \nu_t | \cH_{t-1}}\!\bigl(\pi^\star \notin \Pi^\star(\mdq)\bigr) 
    \;=\; \Gamma_\mdp . $$
\end{restatable}
\noindent Combined with Theorem~\ref{thm:opt_post_contra}, this result ensures that when running Algorithm~\ref{alg:repit}, the posterior probability of mis-identifying the optimal policy decays exponentially fast with an optimal rate. 

The next result we prove is in terms of sample complexity. We show that when equipped with a valid stopping procedure, Algorithm~\ref{alg:repit} attains the optimal sample complexity prescribed by Theorem~\ref{thm:lbd-sc}.

\paragraph{Sample Complexity.}\algoname{} can be coupled with a posterior-sampling-based stopping rule similar to the stopping rule used in \citet{kone2024b} for bandits. This stopping rule is based on the number of \iid samples collected from $\nu^{\eta_t}_t$ before finding a challenger MDP, i.e., by clipping the number of iterations in line~\ref{algo:sample-challenger} of Algorithm~\ref{alg:repit} by a threshold $B(t, \delta)$ to be defined.

Recalling that $\mdq_{t, 1}, \mdq_{t, 2}, \ldots $ are sampled\iid from the posterior distribution $\nu_t^{\eta_t}$, the stopping time is formally defined as 
\begin{equation}
\label{eq:sto-time-ps}
\!\!\!\!\tau_\delta \ceq \!\inf \lb t \geq 1 \!:\!  \forall \,1\leq k\leq B(t, \delta), \!\hat \pi_{t} \!\in \Pi^\star(\mdq_{t, k}) \!\rb\!. \!\!	
\end{equation} 
The number of resampling is calibrated to ensure the $\delta$-correctness of this stopping rule.  
\begin{restatable}{theorem}{scMain} \label{thm:sc-main}
Let $\gamma \in (0, 1)$, $\alpha \in (0, 1/2)$, $\alpha < \gamma$, $\varsigma \in (0, 1)$, with $\varsigma> 2\alpha $. 	If $B(t,\delta)$ satisfies $\limsup_{\delta \to 0} \frac{\log B(t,\delta)}{\eta_t \log\frac1\delta} \leq 1$, then, \algoname{} coupled with the stopping time $\tau_\delta$ described in \eqref{eq:sto-time-ps} and run with parameterss $\gamma, \alpha$ and learning rate $\eta_t = O(t^{-\varsigma})$,
 satisfies 
	\begin{eqnarray*}
		&&\bE_\mdp\bigl[\tau_\delta\bigr] \leq \Gamma_\mdp^{-1}  \log\frac1\delta + o\bigl(\log\tfrac1\delta\bigr)\:, \; \text{and} \; \\&& \quad \bP_\mdp\!\biggl(\limsup_{\delta \to 0} \frac{\tau_\delta}{\log\tfrac1\delta} \leq \Gamma_\mdp^{-1}\biggr)  = 1\:. 
	\end{eqnarray*}
\end{restatable}

Theorem~\ref{thm:sc-main} provides a sufficient condition on $B(t,\delta)$ to guarantee optimal sample complexity, but it does not ensure $\delta$-correctness with this threshold. The lemma below provides the correctness guarantee when $\delta$ is small. 
\begin{lemma}
\label{lem:calibration}
Let $\beta$ be a threshold such that for any $\delta \in (0,1)$, the event
$$ \cE_\delta  \ceq  \lb \forall t \geq 1,  \sum_{x \in \cX} n_t(x)\KL(\empmdp_t\,\|\,\mdp)_{x} \leq\beta(t, \delta) \rb$$
holds with probability at least $1-\delta$, where $R_h(s,a) \ceq \cN(\mu(s,a, 1), 1)$, $\hat R_h^t(s,a) \ceq \cN(\muh_t(s,a,h), 1/n_t(s,a,h))$ and $\KL(\empmdp_t\,\|\,\mdp)_{s,a,h}\ceq \KL(\hat R_h^t(s,a) \otimes \hat p_t(\cdot| s, a, h)\,\|\, R_h(s,a)\otimes p(\cdot| s,a, h))$. $x$ denotes a generic $s,a,h$ and $\cX = \cS \times \cA \times [H]$. 
	For $B(t, \delta) = O(\exp(\eta_t \beta(t, \delta)) \log(t/\delta))$ the posterior stopping time defined in Eq.\eqref{eq:sto-time-ps}satisfies 
	$ \limsup_{\delta \to 0}  {\bP_\mdp(\tau_\delta < \infty, \hat \pi_{\tau_\delta} \neq \pi^\star)}/{\delta } \leq 1.$
\end{lemma} 
\citet{kaufmann_mixture_2021} prescribes correct thresholds $\beta(t,\delta) = \log(1/\delta) + O(\log(t))$. Thus $B(t,\delta)$ as defined in Lemma~\ref{lem:calibration} satisfies $\limsup_{\delta \to 0} \frac{\log B(t,\delta)}{\eta_t \log\frac1\delta} \leq 1,$ so that Theorem~\ref{thm:sc-main} applies and\algoname{} is asymptically optimal and correct with the value of $B(t,\delta)$ prescribed by Lemma~\ref{lem:calibration}. 
\section{EXPERIMENTS}
\label{sec:expe}

\begin{figure*}[tb]
  \centering
  \begin{minipage}[b]{\linewidth}
    \centering
    \includegraphics[width=0.323\linewidth]{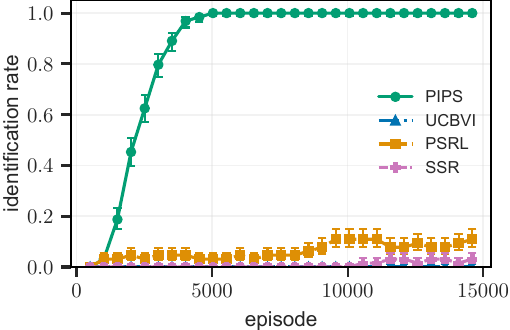}
      \hspace{0.025cm}
    \includegraphics[width=0.323\linewidth]{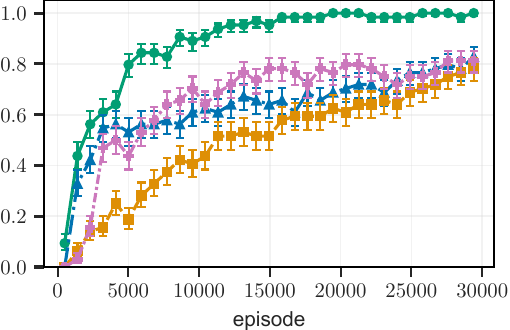}
    \hspace{0.025cm}
    \includegraphics[width=0.323\linewidth]{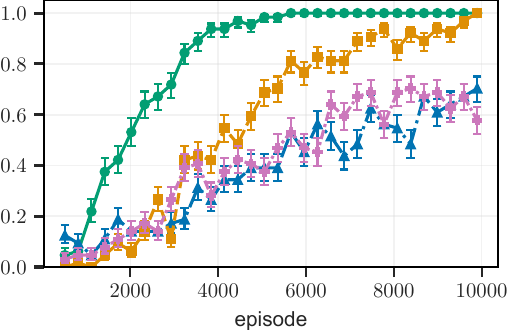}
   \caption{Empirical identification rate averaged on 64 runs with $95\%$ Wilson's confidence interval. From left to right: MOCA instance, RiverSwim, CombinationLock.}
    \label{fig:all_probid}
  \end{minipage}
\end{figure*}

In this section, we evaluate our algorithm against a range of competitors.
We examine two performance metrics : (i) the correct identification rate, defined as $1$ if the recommended policy is optimal and $0$ otherwise, and (ii) the performance factor
$ 1 - \tfrac{V_0^\star - V^{\hat \pi^t}}{V_0^\star}$. 
The average identification rate reflects the fraction of times the recommended policy is optimal, while the performance factor measures the value of the recommended policy relative to the optimal value. 

 To ensure fairness, all algorithms are run without a stopping rule. Each experiment is repeated 50 times, 
 and we report the average value of both metrics.
\paragraph{Baselines.}Algorithm~\ref{alg:repit}, is compared to several established baselines. We evaluate BPI-UCBVI \citep{pmlr-v139-menard21a}. Its sampling rule corresponds to the greedy policy with respect to an optimistic estimate of the optimal $Q$-function. Another baseline is PSRL \citep{osband2013moreefficientreinforcementlearning}, a randomized exploration algorithm that maintains a posterior distribution over MDPs and acts greedily wrt a sampled MDP, and SSR \citep{NEURIPS2022_298c3e32}, a variant of PSRL with distributions over $Q$-functions.

\subsection{Environments}
We evaluate the algorithms on challenging environments that require efficient exploration. 

\paragraph{MOCA}
The first environment, adapted from \citet{pmlr-v178-wagenmaker22a} and shown in Fig.~\ref{fig:mocamdp}, contains $(L+1)$ states and $(L+1)$ actions. Among them, $a^\star$ is the optimal action at each stage but transitions to the next states with small probability. In contrast, the other actions $a_1,\ldots,a_L$ are suboptimal, but taking $a_l$ in $s_{\mathrm{init}}$ deterministically leads to state $s_l$. This environment poses a subtle challenge: identifying the optimal action in states $s_1, \ldots, s_L$ may require first taking suboptimal actions in $s_{\text{init}}$. As a result, it is particularly difficult for regret-minimizing algorithms applied to BPI, since they tend to avoid actions with poor short-term performance. 

We also evaluate on the classic RiverSwim MDP (see e.g. \citet{STREHL20081309}), and the CombinationLock MDP (see e.g. \citet{pmlr-v162-zhang22aa}). 

\paragraph{RiverSwim.} In the RiverSwim environment, the optimal policy always selects the \texttt{RIGHT} action to reach the high-reward state at the end of the chain, despite the lower immediate rewards along the way. We use a chain of length $L=8$, set the horizon to $H=10$, and take $s_{\text{init}} = s_1$. The rewards are Gaussian with variance $1/1000$. The environment is described in Appendix~{\ref{appsubsec:add_exp}, Fig.\ref{fig:riverswim}.

\paragraph{Combination Lock.} We consider a unichain, episodic combination-lock MDP with $S=H$ states $s_1,\dots,s_H$ and an action set of size $A\geq 2$. Each nonterminal state $s_l$ ($l=1,\dots,H-1$) has a unique optimal action denoted $a_l^\star$. If the agent selects the optimal action $a_l^\star$ in state $s_l$, it transitions to $s_{l+1}$ with probability $1-\varepsilon$ and returns to the initial state $s_1$ with probability $\varepsilon$. Any suboptimal action $a\neq a_l^\star$ at $s_l$ sends the agent back to $s_1$ with probability $1$. The terminal state $s_H$ is absorbing, and the agent receives a reward of $1$ only upon reaching $s_H$ (i.e., after taking the optimal action at every step); all other transitions yield a reward of $0$. In our experiments, we set $H=10$, $A=3$, and $\varepsilon=0.01$.

\subsection{Summary} Fig.~\ref{fig:all_probid} plots the identification rate vs the number of episodes for the three MDPs. The most significant improvement of\algoname{} over its competitors is achieved in the MOCA MDP. This is expected as in this MDP, optimistic algorithms essentially explore by following the policy with the highest value, which, as described in Fig.~\ref{fig:mocamdp}, only reach some states with probability exponentially small in $S$. 

Fig.~\ref{fig:all_perf} shows the performance factor results. Additional experiments, presented in Appendix~\ref{appsubsec:add_exp}, Fig.~\ref{fig:appx-samples-mocaexp},\ref{fig:appx-samples-mocaexp-log}, \ref{fig:appx-samples-mocaexp-loglog}, compare the posterior contraction rates of\algoname{} and PSRL on the MOCA instance, highlighting the significant advantage of\algoname{} over its competitors. Further implementation details are provided in Appendix~\ref{appsubsec:impl_details}. 

\begin{figure*}[htb]
	    \centering
    \includegraphics[width=\linewidth]{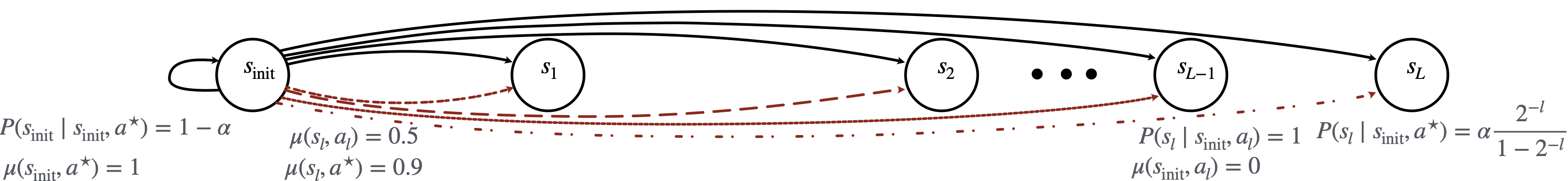}
    \caption{Custom MOCA environment where reaching the optimal policy requires exploratory steps involving sub-optimal actions $a_1,\ldots, a_L$ (red dotted lines).}
   \label{fig:mocamdp}
\end{figure*}

\begin{figure*}[hbt]
	\centering
        \includegraphics[width=\linewidth]{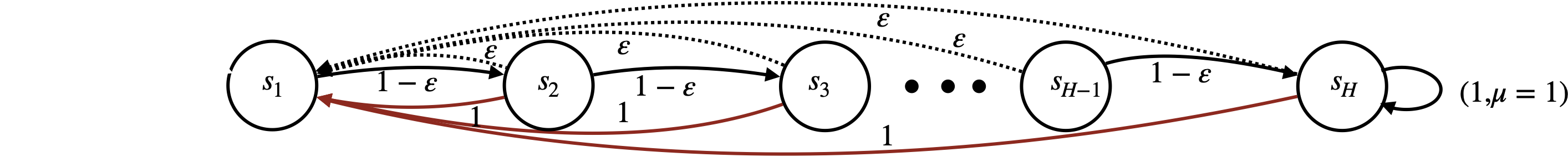}
    \caption{The $H$-state CombinationLock environment. Each nonterminal state offers $A$ actions. The unique optimal action at state $s_l$ advances the agent to $s_{l+1}$ with probability $1-\varepsilon$ (solid black arrow) and returns it to the initial state $s_{\mathrm{init}}$ with probability $\varepsilon$ (dashed arrow); any suboptimal action (red arrow) sends the agent back to $s_{\mathrm{init}}$ with probability $1$. The terminal state $s_H$ is absorbing and yields a reward $1$ (all other transitions give a reward $0$). Transition probabilities are annotated on the corresponding arrows.}
    \label{fig:comb_lock}
\end{figure*}

\begin{figure}[hbt]
  \centering
    \includegraphics[width=0.8\linewidth]{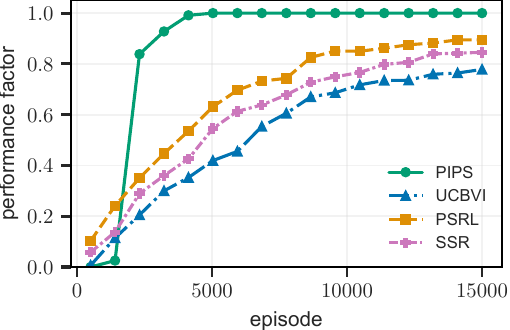}
      \vspace{0.05cm}
    \includegraphics[width=0.8\linewidth]{figs/res_comblock_8_3_var14_64_10000_eps_25_full_bernstein_probid.pdf}
   \caption{Performance factor averaged on 64 runs with $95\%$ Wilson's confidence interval. MOCA instance (top), and CombinationLock (bottom).
   }
    \label{fig:all_perf}
\end{figure}

\section{DISCUSSION AND OPEN PROBLEMS}
\label{sec:conclusion}

We have studied the problem of best policy identification in finite-horizon episodic tabular MDPs. In this work, we introduced a computationally efficient algorithm that leverages posterior sampling combined with an online learning algorithm to explore the MDP and iteratively solve the optimization problem characterizing the optimal lower bound for best policy identification. Our analysis establishes optimal asymptotic guarantees on posterior contraction and sample complexity when the algorithm is equipped with a $\delta$-PAC-calibrated threshold for the posterior stopping rule. 

Despite its computational efficiency, our algorithm still requires sampling $O(S^2AH)$ parameters per round to generate a challenger MDP, which can be a bottleneck in tabular MDPs with large state spaces. Furthermore, while its theoretical guarantees are tight in the asymptotic regime, 
they do not provide explicit dependence on the state, action, or horizon parameters, which may not be negligible in the non-asymptotic setting, particularly when $\delta = \Omega(1)$. Obtaining tight finite-sample guarantees in the moderate confidence regime remains an open challenge as current algorithms scale with diverse and often incomparable 
complexity terms, including substantial "second-order'' terms, such as those polynomial in $\log\bigl(1/\delta\bigr)$ \citep{pmlr-v178-wagenmaker22a, NEURIPS2022_27bf08fe, al-marjani23a, NEURIPS2024_28795419}.

This work can be extended in several promising directions. For instance, one could consider extensions to the linear function approximation setting \citep{pmlr-v125-jin20a}, or to the $(\veps, \delta)$-PAC framework, where the objective is to identify an $\veps$-optimal policy for some $\veps > 0$. We briefly discuss an extension of\algoname{} to the $(\veps, \delta)$-PAC setting in Appendix~\ref{appx:extension}, although the resulting guarantees are sub-optimal. Extending our main results to this framework remains an open question. 

\section*{Acknowledgements}
KJ was funded in part by NSF Awards 2141511, 2023239. 
This work was done while CK was visiting KJ at the University of Washington.
\printbibliography
\section*{Checklist}
\begin{enumerate}

  \item For all models and algorithms presented, check if you include:
  \begin{enumerate}
    \item A clear description of the mathematical setting, assumptions, algorithm, and/or model. [\textbf{Yes}] All claims are proven in the Appendix. 
    \item An analysis of the properties and complexity (time, space, sample size) of any algorithm. [\textbf{Yes}] The complexity is analyzed in a paragraph in the main, and all sample complexity claims are proven in the Appendix. 
    \item (Optional) Anonymized source code, with specification of all dependencies, including external libraries. [\textbf{Not Applicable}]
  \end{enumerate}

  \item For any theoretical claim, check if you include:
  \begin{enumerate}
    \item Statements of the full set of assumptions of all theoretical results. [\textbf{Yes}] Statements made in the main are proven in the Appendix
    \item Complete proofs of all theoretical results. [\textbf{Yes}] The results are proven in the Appendix
    \item Clear explanations of any assumptions. [\textbf{Yes}]  
  \end{enumerate}

  \item For all figures and tables that present empirical results, check if you include:
  \begin{enumerate}
    \item The code, data, and instructions needed to reproduce the main experimental results (either in the supplemental material or as a URL). \textbf{[Yes]} Section~\ref{appx:add_exp_impl_details} describes the experimental setup and other implementation details for reproducibility. 
    \item All the training details (e.g., data splits, hyperparameters, how they were chosen). [\textbf{Yes}] Described in Section~\ref{appx:add_exp_impl_details}
    \item A clear definition of the specific measure or statistics and error bars (e.g., with respect to the random seed after running experiments multiple times). [\textbf{Yes}] Described in Section~\ref{sec:expe}
    \item A description of the computing infrastructure used. (e.g., type of GPUs, internal cluster, or cloud provider). [\textbf{Yes}] Described in Section~\ref{appx:add_exp_impl_details}
  \end{enumerate}

  \item If you are using existing assets (e.g., code, data, models) or curating/releasing new assets, check if you include:
  \begin{enumerate}
    \item Citations of the creator If your work uses existing assets. [\textbf{Not Applicable}]
    \item The license information of the assets, if applicable. [\textbf{Not Applicable}]
    \item New assets either in the supplemental material or as a URL, if applicable. [\textbf{Not Applicable}]
    \item Information about consent from data providers/curators. [\textbf{Not Applicable}]
    \item Discussion of sensible content if applicable, e.g., personally identifiable information or offensive content. [\textbf{Not Applicable}]
  \end{enumerate}

  \item If you used crowdsourcing or conducted research with human subjects, check if you include:
  \begin{enumerate}
    \item The full text of instructions given to participants and screenshots. [\textbf{Not Applicable}]
    \item Descriptions of potential participant risks, with links to Institutional Review Board (IRB) approvals if applicable. [\textbf{Not Applicable}]
    \item The estimated hourly wage paid to participants and the total amount spent on participant compensation. [\textbf{Not Applicable}]
  \end{enumerate}

\end{enumerate}
\clearpage
\appendix
\thispagestyle{empty}
\onecolumn

\tableofcontents 
\clearpage

\numberwithin{equation}{section}
\section{EXTENDED NOTATION TABLE}
\label{sec:notation}

We recall the notation used in the main and introduce some notation used in the appendix. In what follows, $\mdp$ is always fixed, and its optimal policy is also fixed. In general, $\hat \cdot$ denotes empirical quantities, whereas $\wt \cdot$ denotes quantities and variables related to the posterior distribution. 

\label{sec:outline_notation} 
\begin{table}[H]
\caption{Additional notations used in the main and in the appendix}
\label{tab:notation_table_setting}
\begin{center}
\begin{tabular}{l | l}
  \toprule
 Symbol & Definition \\
 \toprule 
$\mdp$ & $(\cS,\cA,H,P , \mu,  s_\text{init})$ \\
$\empmdp_{t}$ & $(\cS,\cA,H, \hat p_{t},  \muh_{t}, s_\mathrm{init})$ \\
$\pi^\star$ & Unique (deterministic) optimal policy of $\mdp$ \\ 
$\triangle$ & $\cS$-Probability simplex \\ 
$\Mdp $ & Set of parameters of MDPs : $\{(0,1) \times \triangle\}^{SAH}$\\ 
$\Mdp^\veps$ & 
$\lb (r,q) \in \Mdp : \forall (s,a,h,s')\; , q(s'\mid s,a,h)\geq \veps \rb $\\
  $\cR$ & Expected reward kernel space: $(0,1)^{SAH}$ \\ 
  $\cP$ & Transition kernel space: $\triangle ^{SAH}$ \\ 
 $A^\pi$ &  $\lb (r,q) \in \Mdp: V_0^\pi(r,q) \geq V_0^{\pi^\star}(r,q)\rb$ 
\\ 
 $A^\pi(q)$ &  $\lb q: \exists r \in \cR \text{ such that } (r,q) \in A^\pi \rb$ 
 \\
 $\Pi_\mathrm{det}$ & Deterministic Markovian policies 
 \\
  $(\eta_t)_t$& Learning rate or posterior inflation rate \\  
  $\bpol_t$ & Behavioral or exploration policy at episode $t$ \\
  $c_t$ & Forced exploration or coverage policy at episode $t$ \\
  $n_t(s,a,t)$ & Number of pulls of $(s,a)$ up to the end of episode $t$ \\ 
  $n_t(s'\mid s,a,h)$ & Number of transitions $(s,a,s')$ at stage $h$ \\ 
    $\muh_t(s,a, h)$ & Empirical reward at $(s,a,h)$ based upon the first $t$ episodes  \\
  $\Pi^\star(\mdp)$ & Optimal start-state policies of $\mdp$ \\ 
  $\Pi_\mathrm{det}^\star(\mdp)$ & Optimal deterministic policies of $\mdp$ \\ 
 $\alt(\mdp)$ & Alternative models : $\big\{ \mdp' \in \Mdp  : \mdp \ll \mdp' \text{ and } \Pi^\star(\mdp) \cap \Pi^\star(\mdp') = \emptyset \big\}$ \\
  $\nu_{\eta_t}^t(\cdot   |   \alt(\empmdp_t))$ & Truncated and inflated posterior at the start episode $t$
  \\
  $\lb (s_h^t, a_h^t, r_h^t, s_{h+1}^t)\rb_{h=1}^{H}$ & Trajectory collected during episode $t$ \\
  $\hat p_t(\cdot   \mid    s,a, h) \ceq \hat p_t(\cdot \mid s,a,h)$ & Empirical transition probabilities at $(s,a,h)$ during episode $t$ 
  \\ 
  $W_h(s,a)$ & Maximum visitation probability of $(s,a,h)$: $\max_\rho w_h^\rho(s,a)$ \\
   \bottomrule  
\end{tabular}
\end{center}
\end{table}

\paragraph{Value gaps.}
\label{par:value_gaps}
We recall that a policy $\pi$ is \emph{globally optimal} if its support at state and stage is included in the argmax set of the optimal $Q$-function.  For a policy $\pi$, we define the \emph{gap} at $(s,a,h)$ as
\[
    \Delta_h^\pi(s,a) \;=\; V_h^\pi(s) - Q_h^\pi(s,a).
\]
Note that, if $\pi$ is globally optimal then the above gap is called the \emph{value gap} and quantifies the suboptimality of action $a$ in state $s$ at stage $h$ relative to the optimal action. We say that an MDP admits a \emph{strongly globally  policy} if, for every $(s,h)$, the argmax set of the optimal $Q$-function is a singleton : $|\argmax_{a \in \cA} Q_h^\star(s,a) | = 1$
i.e., there is a unique optimal action at each state and stage. If in addition, the start-state optimal policy is unique, then it must coincide with the greedy policy induced by $Q^\star$. In this case, we say the MDP admits a \emph{unique start-state optimal policy}, which is also strongly optimal. 

Finally, writing $a^\star \ceq \argmax_a Q_h^\star(s,a)$, we define at $(s,h)$,  
$ \Delta_h^\pi (s,a^\star) \,=\,\min_{a \neq a^\star} \Delta_h(s,a) $ and we define the \emph{minimal value gap} as 
\begin{equation}
	\Delta_{\min}^\pi \;\ceq\; \min_{s,a,h} \Delta_h^\pi (s,a)\:, 
	\end{equation}
which is postive when $\pi$ is the strongly globally optimal policy in $\mdp$. When the MDP is not clear from the context, we write $\Delta_{\min}^\pi(\mdp)$. 
 
\paragraph{Additional notation.} Given a (Markov) policy $\rho$, we denote by $w^\rho$ the visitation probabilities induced on the MDP, i.e., for any $s,a,h$, $w^\rho_{h}(s,a)$ is the probability that the state-action $(s,a)$ is visited at episode $h$ when $\rho$ is followed. 

The reward space is $\cR$ and the transitions space is $\cP$.  We have $\Mdp = \cR \otimes \cP$.

Since our results are targeting the first-order asymptotic bound, we let $T_0$  denote an arbitrary constant that may appear multiple times in our calculations.  

${O}(f(t))$ denotes an asymptotic upper bound when $t\to \infty$, hiding constants, and $g(t) = o(f(t))$ if and only if $g(t) / f(t) \to 0$ when $t\to \infty$. 

Finally, we define the following $\ell_1$-distance  on $\Mdp $, 
$$ \Vert \mdp - \mdp'\Vert_1 = \sum_{s,a,h} \bigl| \mu(s,a,h) - \mu'(s,a, h) \bigr| + \Vert p(\cdot | s,a, h)  - p'(\cdot |s,a, h) \Vert_1, \quad \forall\; \mdp, \mdp' \in \Mdp  \:.$$
\section{SAMPLE COMPLEXITY LOWER BOUND}
\label{appx:lower_bound}

We now state a lower bound on the sample complexity of any $\delta$-PAC algorithm for best policy identification, derived from the celebrated information-contraction principle.  

\smallskip
\noindent
Let
$$\alt(\mdp) \ceq \bigl\{  \mdp' \in \Mdp : \mdp \ll \mdp' \quad \text{and} \quad \Pi^\star(\mdp) \cap \Pi^\star(\mdp') = \emptyset \bigr\},$$
that is, the set of MDPs $\mdp'$, so that $\mdp$ is absolutely continuous with respect to $\mdp'$ but have no optimal policy in common with $\mdp$.  

At episode $t$, an adaptive algorithm selects an exploration policy $\rho_t$ based on past observations. Executing $\rho_t$ in the environment generates the trajectory 
$$
\{s_h^t, a_h^t, r_h^t, s_{h+1}^t\}_{h=1}^H, 
\quad s_1^t = s_\text{init}, 
\quad a_h^t \mid s_h^t  \sim \bpol_t(\cdot \mid s_h^t, h), 
$$
with transitions $s_{h+1}^t \mid s_h^t, a_h^t \sim p(\cdot \mid s_h^t,a_h^t, h)$ and rewards $r_h^t \mid s_h^t, a_h^t \sim R_h(s_h^t,a_h^t)$.  
Let $\cH_t$ denote the $\sigma$-algebra generated by all observations and exploration policies up to the end of episode $t$, including external randomness $(U_1,\dots,U_t)$ used for tie-breaking and randomization.  
\smallskip
\noindent
For two MDPs 
$$
\mdp = (\cS,\cA,H, \mdpp, \mdpr,s_\text{init}), 
\quad 
\mdq = (\cS,\cA,H, \tilde \mdpp, \tilde \mdpr ,s_\text{init}),
$$
we define the KL divergence between their kernels at $(s,a,h)$ as
$$
\KL(\mdp\,\|\,\mdq)_{s,a,h} \ceq
\KL\!\Bigl( R_h(s,a) \otimes p(\cdot \mid s,a, h) \,\Big\|\, \wt R_h(s,a) \otimes \wt p(\cdot \mid s,a, h) \Bigr).
$$

\smallskip
\noindent
An algorithm is said to be $\delta$-PAC if, for any $\mdp \in \Mdp$, its stopping time $\tau$ and recommendation $\hat\pi^\tau$ satisfy
$$
\bP_{\mdp}\bigl(\tau < \infty,\; \hat\pi_\tau \notin \Pi^\star(\mdp)\bigr) \leq \delta.
$$
\smallskip
\noindent
Using the data-processing inequality and change-of-distribution arguments (see \citet{garivier_optimal_2016}), we obtain the following fundamental bound:
\begin{lemma}
\label{lem:lower_bound}
Consider an MDP $\mdp$ with a unique optimal policy.  
Then, for any $\delta$-PAC algorithm with stopping time $\tau$, it holds that
$$
    \sum_{s,a,h} \bE_{\mdp}\lsb n_{\tau+1}(s,a,h)\rsb\,\KL(\mdp\,\|\,\mdq)_{s,a,h}
    \geq \log \frac{1}{2.4\delta},
    \qquad \forall\; \mdq \in \alt(\mdp).$$
\end{lemma}

Using this result, we prove the sample complexity lower bound. We recall that $\Omega_\mdp$ denotes the set of visitation probabilities induced by Markovian policies on $\mdp$ 

\begin{equation}
\Omega_\mdp \;\ceq\; 
    \bigl\{  \{w_h(s,a)\}_{s,a,h} : \exists \; \pi \in \Pi  \;\;\text{such that}\;\; 
    w_h(s,a) = \bP_\mdp^\pi(s_h = s, a_h = a), \;\forall (s,a,h) \bigr\}. 	
\end{equation}

\lbdSc* 
\begin{proof}
Fix $\mdp$ with a unique optimal policy and let $\tau$ be the stopping time of any $\delta$-PAC algorithm. By Lemma~\ref{lem:lower_bound}, for every $\mdq \in \alt(\mdp)$,
\begin{equation}\label{eq:lb-def-1}
\sum_{s,a,h} \bE_\mdp\!\lsb n_{\tau+1}(s,a, h)\rsb \, \KL(\mdp\,\|\,\mdq)_{s,a,h}
\;\geq\;\log \frac{1}{2.4\delta} \: .  
\end{equation}

We recall that $\rho^t$ is predictable (i.e., measurable wrt $\cH_{t-1}$). Therefore \begin{eqnarray*}
	 \bE_\mdp\!\left[n_{\tau+1}(s,a, h)\right]   
	 &=& \bE_\mdp\!\left[ \sum_{t\geq 1} \indp{ \tau \geq t } \cdot \indp{ s_h^t =s , a_h^t = a} \right] \\
	 &=& \bE_\mdp \lsb \bE \!\lsb \sum_{t\geq 1} \indp{ \tau \geq t } \cdot \indp {s_h^t =s , a_h^t = a }\middle | \cH_{t-1}\rsb \rsb
\end{eqnarray*}
which follows from the tower rule of expectations. Next observe that the event $\indp{  t \leq \tau }$ is the complementary of the event $\indp{ \tau < t}$ which is the same as $\indp{ \tau \leq t-1}$ so it is measurable wrt $\cH_{t-1}$. Thus 
\begin{eqnarray*}
	\bE\lsb \indp{ \tau \geq t} \indp{ s_h^t =s , a_h^t = a } \middle | \cH_{t-1}\rsb &=&  \indp{ \tau \geq t} \cdot \bE\lsb \indp{ s_h^t =s , a_h^t = a } \middle | \cH_{t-1}\rsb \\&=& \indp{ \tau \geq t } \cdot w^{\bpol_t}_h(s,a). 
\end{eqnarray*}

Combining the above displays, 
\begin{equation*}
	\bE_\mdp\!\lsb n_{\tau+1}(s,a,h)\rsb    = \bE_\mdp\!\lsb  \sum_{t=1}^\tau w^{\bpol_t}_h(s,a) \rsb,  
\end{equation*}
then dividing both sides by $\bE_\mdp[\tau]$ yields 
$$ w_h(s,a) \ceq \frac{\bE_\mdp\!\lsb  \sum_{t=1}^\tau w^{\bpol_t}_h(s,a) \rsb }{\bE_\mdp[\tau]}  =  \bE_\mdp\!\lsb  \frac{1}{\bE_\mdp[\tau]}\sum_{t=1}^\tau w^{\bpol_t}_h(s,a) \rsb. $$ 
Next, we remark that using Caratheodory's theorem, \citet{pmlr-v178-wagenmaker22a} justifies in the proof of their Lemma~F.4 that $\{w_h(s,a)\}_{s,a,h}$ as defined above belongs to $\Omega_\mdp$. 
Rearranging gives 
\begin{equation}\label{eq:lb-def-2}
\bE_\mdp[\tau]  \sum_{s,a,h} w_h(s,a)\, \KL(\mdp\|\mdq)_{s,a,h}
\;\geq\;\log \frac {1}{2.4\delta}, \forall\;\; \mdq \in \alt(\mdp), 
\end{equation}
the result follows by taking on the left-hand side the $\inf$ over $\alt(\mdp)$ followed by a $\sup$ over $\Omega_\mdp$. 
\end{proof}

\section{POSTERIOR STOPPING RULE}
\label{sec:posterior_stopping}
We discuss the proof of the stopping rule. We prove that under mild conditions, we can prove that a simple posterior stopping algorithm can be used to calibrate an efficient stopping rule. 

  \medskip
  
The stopping rule for this setting relies on the posterior distribution.  Letting the empirical MDP $\empmdp_t$,  
define the optimal policy in the empirical MDP as $\hat \pi_t$, and the value function for policy $\pi$ as $\hat V_{t, 0}^\pi$. 

Letting $B(t, \delta)$ be a function to be defined, and letting $\mdq_{t, 2}, \mdq_{t, 2}, \ldots, $ be \iid samples from the posteiror distribution over the MDPs, we denote by $\Pi^\star({\mdp})$ the set of optimal policies of $\mdp$. We introduce the following stopping time : 
\begin{equation}
\label{eq:sto-1}
	\tau \ceq \inf \lb t \geq 1:  \forall k\; \leq \, B(t, \delta),\, \hat \pi_{t} \in \Pi^\star({ \mdq_{t, k}}) \rb\:,
\end{equation}
and after stopping, we recommend the empirical optimal planning $\hat \pi^{t}$. We prove the following lemma. Let us introduce an event $\cE_\delta$ that holds with probability at least $(1-\delta)$ for any $\delta \in (0, 1)$. Let $\alt_\pi$ be the set of MDPs for which $\pi$ is not an optimal policy.

\begin{lemma}
\label{lem:stopping}
We have 
	\begin{equation*}
		\bP_\mdp\!(\tau<\infty, \hat \pi_\tau  \notin \Pi^\star(\mdp) ) \leq \delta/2 + \sum_{t=1}^\infty \bE_\mdp \lsb \indp{ \cE_{\delta/2}, \hat \pi_t  \notin \Pi^\star(\mdp) }  \exp\lp - \bP_{\mdq \sim \nu_t \mid \cH_{t-1}}\lp  \mdq \in \alt_{\hat \pi^t}\rp {B(t, \delta)} \rp \rsb \:
	\end{equation*}
\end{lemma}
\begin{proof}
	We should expect to show that this stopping rule is $\delta$-PAC for a specific choice of $B$ and possibly an inflation of the posterior. To start, we assume that each MDP admits a unique optimal policy and the posterior is defined over such a space, then observe that 
\begin{eqnarray*}
	\bP_\mdp(\tau<\infty, \hat \pi_\tau  \notin \Pi^\star(\mdp) ) &\leq&  \bP_\mdp\lp \tau<\infty, \hat \pi_\tau  \notin \Pi^\star(\mdp),  \lb \forall\; k \; \leq B(\tau, \delta),  \hat \pi_{\tau} \in \Pi^\star({\mdq^{\tau, k}}) \rb \rp  \\
	&\leq& \delta/2 + \sum_{t = 1}^\infty \bE_\mdp\lsb \indp{ \cE_{\delta/2}, \hat \pi_t  \notin \Pi^\star(\mdp) }  \bE_\mdp\lsb  \indp{   \forall\; k\; \leq B(t, \delta),  \hat \pi_{t} \in \Pi^\star({\wt \mdp_{t, k}}) } \middle | \cH_{t -1} \rsb \rsb. 
\end{eqnarray*}
Next, we have to control  
\begin{eqnarray*}
	L_{t, \delta }  \ceq \indp{ \cE_{\delta/2}, \hat\pi_t  \notin \Pi^\star_\mdp}  \bE_\mdp\lsb  \indp{   \forall\; k\; \leq B(t, \delta),  \hat\pi_{t} \in \Pi^\star({\mdq_{t, k}}) } \mid \cH_{t -1} \rsb\:, 
\end{eqnarray*}
for which a simple rewriting yields 
\begin{eqnarray*}
	L_{t, \delta} &=&  \indp{ \cE_{\delta/2}, \hat \pi_t  \notin \Pi^\star(\mdp) } \bP_{\mdq \sim \nu_t | \cH_{t-1}} \lp \hat \pi_{t} \in \Pi^\star({\mdq}) \rp^{B(t, \delta)} \\
	&=& \indp{ \cE_{\delta/2}, \hat \pi_t  \notin \Pi^\star(\mdp) }   \lp 1- \bP_{\mdq \sim \nu_t \mid \cH_{t-1}} \lp  \hat \pi_{t} \notin \Pi^\star({\mdq})\rp \rp^{B(t, \delta)} \\
	&\leq & \indp{ \cE_{\delta/2}, \hat \pi_t  \notin \Pi^\star(\mdp) }  \exp\lb - \bP_{\mdq \sim \nu_t \mid \cH_{t-1}}\lp  \hat \pi_{t} \notin \Pi^\star({\mdq})  \rp {B(t, \delta)} \rb \\
	&\leq& \indp{ \cE_{\delta/2}, \hat \pi_t  \notin \Pi^\star(\mdp) }  \exp\lb - \bP_{\mdq \sim \nu_t \mid \cH_{t-1}}\lp  \mdq \in \alt_{\hat \pi_t}\rp {B(t, \delta)} \rb \:.  
\end{eqnarray*}  
\end{proof}
  
Thus, calibrating the stopping rule properly requires tuning $B(t, \delta)$ such that the sum in the RHS of Lemma~\ref{lem:stopping} is smaller than $\delta/2$. Preceisly computing $\bP_{\mdq \sim \nu_t \mid \cH_{t-1}}\lp  \mdq \in \alt_{\hat \pi_t}\rp $ involves evaluating high-dimensional integrals over the complex domain $\alt_{\hat \pi_t}$. 

\section{OPTIMAL EXPLORATION VIA ONLINE LEARNING}
\label{appx:sup_player}

In this section, we analyze a no-regret RL algorithm for online learning in the setting where rewards change over time and their range is not fixed in advance. We show that such an algorithm generates a sequence of exploration policies that accumulates information optimally. Our main result is a generic regret bound for online RL with changing rewards coupled with a mirror descent update, which is of independent interest beyond the present application.

\subsection{Setup and Goal}
\label{subsec:ol-setup}

We analyze a generic version of Algorithm~\ref{alg:ol} where the policy 
improvement step uses a generic online learner on the action set $\mathcal{A}$.
The KL-based reward kernel fed to the online learner at episode $t$ is
\begin{equation}
\label{eq:loss-kernel-online}
\loss_t(s,a,h) \ceq
  \frac{\bigl(\mdqr_{t}(s,a,h)-\muh_{t}(s,a,h)\bigr)^2}{2}
  +\sum_{s'}\hat{p}_{t}(s'\mid s,a,h)
   \log\frac{\hat{p}_{t-1}(s'\mid s,a,h)}
     {\maxp{\mdqp_{t}(s'\mid s,a,h)}{\myexp{-t^\alpha}}},
\end{equation}
where the transition probability clipping ensures $\loss_t(s,a,h)\in(0,\frac{1}{2}+t^\alpha)$.

To ensure optimal data collection, the exploration algorithm must ultimately guarantee, with high probability, the sequence of exploration policies $(\bpol_t)_{t\geq 1}$ satisfies 
\begin{equation}
\label{eq:regret-goal}
\sup_{w\in\Omega_\mdp}\sum_{t=1}^T\sum_{s,a,h}
  w_h(s,a)\,\loss_t(s,a,h)
  - \sum_{t=1}^T\sum_{s,a,h} w_h^{\bpol_t}(s,a)\,\loss_t(s,a,h)
  \leq o(T).
\end{equation}
An online RL learner with changing rewards has been analyzed 
in~\citet{al-marjani23a}, but under two restrictive assumptions: the 
horizon $T$ is known in advance, and the reward range is fixed in $(0,1)$. 
In our setting, the reward range $(0,\frac{1}{2}+t^\alpha)$ grows with 
$t$, requiring a more careful analysis.

\subsection{Algorithm}
\label{subsec:ol-algorithm}

We present a general online RL algorithm, whose pseudo-code is given in 
Algorithm~\ref{alg:appx-ol}. It maintains $SH$ independent online learners 
$(\mathcal{L}_{s,h})_{s,h}$ on the simplex $\triangle_A$, one per 
state-stage pair, combined with an optimistic policy evaluation step to 
compute the loss fed to each learner. The key features are:
\begin{itemize}
  \item \textbf{Unknown dynamics.} The transition kernel $p$ is unknown 
    and fixed; the agent uses the empirical estimate $\hat{p}_t$ together 
    with Bernstein-type bonuses to ensure optimism.
  \item \textbf{Changing rewards.} The reward kernel $\loss_t$ is 
    deterministic and revealed to the agent at the end of each episode, 
    but changes with $t$.
  \item \textbf{Independent per-state policy learning.} The exploration policy 
    $\bpol_t(\cdot\mid s, h)$ is updated independently for each $(s,h)$ 
    via the corresponding learner $\mathcal{L}_{s,h}$, initialized at 
    the uniform distribution over $\mathcal{A}$.
\end{itemize}

\newcommand{\bonus}{\mathrm{bonus}}

\medskip\noindent
\textbf{Optimistic policy evaluation.}
At episode $t$, the agent collects the trajectory 
$\{(s_h^t,a_h^t,s_{h+1}^t)\}_{h=1}^H$ using $\bpol_t$, after which the 
reward kernel $\loss_t$ is revealed. The optimistic $Q$-function for 
$\bpol_t$ is computed by backward induction using the Bernstein-type bonus
\begin{equation}
\label{eq:bernstein-bonus}
\bonus_t(s,a,h) \ceq
  \sqrt{\frac{2\,\var_{\hat{p}_{t, h}}\!\left[\bar{V}_{t, h+1}^{\bpol_t}\right]\!(s,a)
    \cdot\beta^p(n_t(s,a, h),1/t^3)}{\maxp{1}{n_t(s,a, h)}}}
  +\frac{2\beta^p(n_t(s,a, h),1/t^3)\,g_h^t}
    {3\maxp{1}{n_t(s,a, h)}},
\end{equation}
where $g_{t, h} \ceq\max_s\bar{V}_{t, h+1}(s)$. The optimistic $Q$-function 
and its associated value function are defined recursively as
\begin{equation}
\label{eq:optimistic-Q}
\bar{Q}_{t, h}(s,a) \ceq \loss_t(s,a,h)
  +\hat{p}_{t, h}\bar{V}_{t,h+1}^{\bpol_t}(s,a)
  + \bonus_t(s,a, h),
\qquad
\bar{V}_{t, h}(s) \ceq \sum_{a\in\mathcal{A}}\bpol_t(a\mid s, h)\,
  \bar{Q}_{t, h}(s,a),
\end{equation}
with $\bar{V}_{t, H+1}\equiv 0$. In the policy improvement phase, 
each learner $\mathcal{L}_{s,h}$ is updated with $\bar{Q}_{t,h}(s,\cdot)$ 
and the next policy $\bpol_{t+1}$ is computed from the 
updated learner.

\medskip\noindent
\textbf{Choice of online learner.}
We use Bernstein-type bonuses \eqref{eq:bernstein-bonus} together with 
AdaHedge~\citep{rooij_ada_hedge} as the per-state online learner. AdaHedge 
is parameter-free, adapts its learning rate on the fly, and guarantees 
sublinear regret even for unbounded losses, which is essential here since 
the range of $\iota_t$ grows with $t$. Hedge with a scaled learning rate 
is also valid but yields less tight finite-time guarantees.
 \newcommand{\regret}{\text{Regret}}
 
 \medskip
 \hrulefill
\begin{algorithm}[H]
\Fn{\textsc{ComputeExplorationPolicy}$\bigl(t, \bpol_t,\, \loss_t,\, {\hat p}_t,\, n_t\bigr) $}{
\label{alg:appx-ol}
\DontPrintSemicolon
\SetInd{0.5em}{1em}   
\SetAlgoLined
\SetAlgoVlined
\SetKwInOut{Input}{Require}
\SetKwInOut{Output}{Output}
Set $\bar V_{t, H+1}(s) \equiv 0$ for all $s \in \cS$\;

\For{episode stage $h = H, H-1, \ldots, 1$}{
  \ForEach{$(s,a) \in \cS \times \cA$}{
    $\bar Q_{t, h}(s,a) \;\gets\;
      \loss_t(s,a, h)
      \;+\;  \hat p_{t, h}\bar V_{t, h+1}(s,a)      \;+\; \sqrt{ \dfrac{ 2\,\var_{\hat p_{t, h}}\!\bigl[\bar V_{t, h+1}\bigr](s,a)\;\cdot \beta^p(n_t(s,a, h),\,1/t^3) }{ \maxp{1}{ n_t(s,a, h)} } }
      \;+\; \dfrac{ 2\,\beta^p(n_t(s,a, h),\,1/t^{3}) }{ 3\,\maxp{1}{ n_t(s,a, h)} }$
  }
  \tcp{Optimistic policy evaluation}
    $\bar  V_{t, h}(s) \;\gets\; \sum_{a \in \cA} \bpol_t(a \mid s, h)\, \bar Q_{t, h}(s,a) \quad \forall s \in \cS$
}
\ForEach{$(s,h) \in \cS \times [H]$}{
  \tcp*{Policy improvement }
  Feed learner $\cL_{s,h}$ with gain $\bar {Q}_{t, h}(s, \cdot)$ \; 
    Update $\bpol_{t+1}( \cdot \mid s, h)$ from $\cL_{s, h}$ \; }
\Return policy $\bpol_{t+1}$\;
\caption{\textsc{Policy Update Helper Function}
\label{alg:appx-ol}}
}
\end{algorithm}
\medskip
\noindent

\subsection{Regret Upper Bound}
\label{subsec:sup_player}

We present the analysis below. We define the events : 
\begin{align}
	\cE_{\sect}^{p}(T) &\ceq \Bigl\{ \forall\; t\,\leq\, T\:, \forall \;(s,a,h)\:, n_t(s,a, h)\KL(\hat p_t(\cdot | s, a, h)\;\|\; p(\cdot| s,a, h))\leq  \beta^p(n_t(s,a, h), 1/T^2)\Bigr\} \:, \\[4pt]
	\cE_{\sect}^{\text{cnt}}(T) &\ceq \Bigl\{  \forall \; t\,\leq\, T\:, \forall\; (s,a,h)\:,  n_t(s,a, h) \geq \frac12 \bar  n_t(s,a,h) - \log(2SAHT^2)\Bigr\} \:,  \\
	\cE_{\sect}(T) &\ceq \cE_{\sect}^{p}(T) \cap \cE_{\sect}^{\text{cnt}}(T)\:.
\end{align}

At episode $t$, the challenger MDP $\mdq_t \ceq (\mdqr_t,\mdqp_t)$ is sampled from the inflated posterior, and $\empmdp_t = (\muh_t,\hat{p}_t)$ denotes the 
empirical MDP. The reward kernel $\loss_t$ is defined in \eqref{eq:loss-kernel-online}.

We prove the following result. 

\medskip
\begin{proposition}
\label{prop:regret}
	Let $(\cL_{s, h})_{s,h}$ be some online learners on $\triangle$ such that for any sequence of losses, the learner guarantees some regret $\regret_\cL(T)$. Then Algorithm~\ref{alg:appx-ol} guarantees the following regret bound on the event $\cE_{\sect}$, 
	$$ \regret(T) \leq O(T^{\alpha + \frac12 } \beta^p(T, 1/T^3)) +   H \regret_\cL(T)  \:.$$
	\end{proposition}

\begin{proof}
Below, we let $Q_{t, h}^\pi$ and  $V_{t, h}^\pi$ denote the state--action and state value functions in the MDP $\mdp_t^\mathrm{exp}\ceq (\loss_t, p)$. $\bar Q_{t, h}^\pi , \bar V_{t, h}^\pi$ denote the state--action and state value functions in the "optimistic MDP'' $\bar \mdp_t^\mathrm{exp} \ceq (\loss_t + \bonus_t, p)$. For any fixed comparator $\rho$,

 \begin{align}
 \label{eq:decomp}
		 \sum_{t=1}^T V_{t, 0}^\rho - V_{t,0}^{\bpol_t}   		 &= \sum_{t=1}^T  \Biggl[ V_{t, 0}^\rho -   \bar V_{t, 0}^{\rho}  +  \bar V_{t, 0}^{\rho} -  V_{t,0}^{\bpol_t} + \bar V_{t,0}^{\bpol_t} - \bar V_{t,0}^{\bpol_t} \Biggr].  
	\end{align}

On the event $\cE_{\sect}(T)$, for $t\in (\sqrt{T}, T)$, the optimism property $\bar{V}_{t,0}^{\rho}\geq V_{t, 0}^\rho$ holds (Lemma~\ref{lem:optimisitic_Q}), justifying that on $\cE_{\sect}(T)$ 

$$ \sum_{t=1}^T  \Biggl[ V_{t, 0}^\rho -   \bar V_{t, 0}^{\rho}  \Biggr]  \leq  O(T^{\alpha + \frac12}).$$

It remains to bound the remaining terms appearing in \myeqref{eq:decomp}. Let us define 
$$ (A) = \sum_{t=1}^T  \Biggl[\bar V_{t, 0}^{\rho} - \bar V_{t,0}^{\bpol_t}   \Biggr] \qquad (B) =   \sum_{t=1}^T  \Biggl[\bar V_{t,0}^{\bpol_t}  -  V_{t,0}^{\bpol_t} \Biggr]  $$

\medskip\noindent
\textbf{Bounding (A).}
By the performance difference lemma,
\begin{equation}
\label{eq:perf-diff-proof}
\bar{V}_{t,0}^{\rho} - \bar{V}_0^{t,\bpol_t}
  = \bE^\rho\!\left[
    \sum_{h=1}^H
    \bigl\langle\bar{Q}_{t,h}^{\bpol_t}(s_h,\cdot),\,
      \rho(\cdot\mid s_h, h)-\bpol_t(\cdot\mid s_h, h)\bigr\rangle
    \,\Big|\, s_1=s_{\mathrm{init}}
  \right].
\end{equation}
The inner product in \eqref{eq:perf-diff-proof} is exactly the loss fed to learner $\mathcal{L}_{s_h,h}$ at round $t$ against comparator $\rho(\cdot\mid s_h, h) = \rho_h(\cdot \mid s_h)$. By the regret guarantee of each learner,
\[
\sum_{t=1}^T\bigl\langle\bar{Q}_{t,h}^{\bpol_t}(s,\cdot),\,
  \rho(\cdot\mid s, h)-\bpol_t(\cdot\mid s, h)\bigr\rangle
\leq \mathrm{Regret}_\mathcal{L}(T),  \quad \forall (s, h) \in \cS \times \lbrack H \rbrack .
\]
Combining with \eqref{eq:perf-diff-proof} and summing over $h$,
\begin{equation}
\label{eq:bound-A}
(A) \leq H\,\mathrm{Regret}_\mathcal{L}(T).
\end{equation}

\medskip\noindent
\textbf{Bounding (B).}
Term $(B)$ controls the total size of the Bernstein bonuses. 
On $\cE_{\sect}(T)$, by induction,
\[
\bar V_{t,0}^{\bpol_t} - V_{t,0}^{\bpol_{t}}
  \leq 2\sum_{s,a,h} w_h^{\bpol_{t}}(s,a)\,\bonus_t(s,a, h),
\]
where, using $\var_{\hat{p}_{t, h}}(\bar{V}_{t, h+1})[s,a]\leq(g_h^t)^2$,
\begin{equation}
\label{eq:bonus-bound}
\bonus_t(s,a, h)
  \leq g_{t, h}\lsb \sqrt{\frac{2\beta^p(n_t(s,a, h),1/t^3)}{\maxp{1}{n_t(s,a, h)}}}
  + \frac{2\beta^p(n_t(s,a, h),1/t^3)}{3\maxp{1}{n_t(s,a,h)}} \rsb.
\end{equation}
We split the sum over $(s,a,h)$ into two parts using the set
\[
\mathcal{Z}_t \ceq \left\{(s,a,h):\,
  \bar{n}_t(s,a, h) \geq 4\log(2SAHT^2)\right\},
\qquad
\bar{n}_t(s,a, h) \ceq \sum_{k=1}^{t-1}\Bigl[
  (1-\veps_k)w_h^{\bpol_k}(s,a)+\veps_k w_h^{c_k}(s,a)
\Bigr].
\]
On $\cE_{\sect}(T)$, for $(s,a,h)\in\mathcal{Z}_t$: 
$n_t(s,a, h)\geq\frac{1}{4}\bar{n}_t(s,a, h)$.
Moreover, since $\varepsilon_t=t^{-\gamma}$, for $t\geq 2^{1/\gamma}$, $1-\varepsilon_{t+1}\geq\frac{1}{2}$.

\medskip\noindent
Using $\max_{t\leq T, h} \lbrack g_{t, h} \rbrack \leq T^\alpha$ and monotonicity of $\beta^p$ in its second argument,
\begin{align*}
\sum_{t=1}^T\sum_{(s,a,h)\in\cZ_t}
  w_h^{\bpol_{t}}(s,a)
  \frac{g_{t, h}\,\beta^p(n_t(s,a,h),1/t^3)}{\maxp{1}{n_t(s,a,h)}}
  &\leq T^\alpha \beta^p(T,1/T^3)
    \sum_{t=1}^T\sum_{(s,a,h)\in\cZ_t}
    \frac{w_h^{\bpol_{t}}(s,a)}{\maxp{1}{\frac{1}{4}\bar{n}_t(s,a, h)}}.
\end{align*}
Applying the elliptic potential (Lemma~\ref{lem:elliptic-potential}) 
and Cauchy--Schwarz,
\[
\sum_{t=1}^T\sum_{(s,a,h)\in\cZ_t}
  \frac{\frac{1}{4}(1-\veps_{t})w_h^{\bpol_{t}}}
       {\maxp{1}{\frac{1}{4}\bar{n}_t(s,a,h)}}
\leq 4\sum_{s,a,h}\log\!\lp 1+\tfrac{1}{4}\sum_{t=1}^{T+1}
  (1-\veps_{t})w_h^{\bpol_t}(s,a)\rp 
\]
and similarly, for any $\gamma \in (0,1)$ 
\begin{align}
	\label{eq:elliptic-sqrt}
	\begin{split}
\sum_{t=1}^T\sum_{(s,a,h)\in\cZ_t}
  \frac{w_h^{\bpol_{t}}(s,a)}{\sqrt{\maxp{1}{n_t(s,a, h)}}}
&\leq O\! \lp \sum_{t=1}^T \sum_{(s,a,h) \in \cZ_t}  \lsb {\frac{(1-\veps_{t})w_h^{\bpol_{t}}(s,a)}{\sqrt{\maxp{ 1}{ \sum_{k=1}^{t-1} (1-\veps_k) w_h^{\bpol_{k}}(s,a)} }}} \rsb \rp \\&\leq O\! \lp \sqrt{T+1} \rp.
	\end{split}
\end{align}

Combining these bounds,
\begin{equation}
\label{eq:term-I}
\text{Term}\ (B)\big|_{\mathcal{Z}_t}
  \leq O\!\lp \beta^p(T,1/T^3)\,T^{\alpha+ \frac12} \rp 
\end{equation}

\medskip\noindent
For terms outside $\cZ_t$, we note that for $(s,a,h)\notin\cZ_t$, $\bar{n}_t(s,a, h)<4\log(2SAHT^2)$,
so the total visits to such pairs is bounded. We have,
	
\begin{align*}
&\sum_{t=1}^T\sum_{(s,a,h)\notin\cZ_t}
  w_h^{\bpol_{t}}(s,a)\,\bonus_t(s,a, h)
  \leq \\ & \sum_{s,a,h}\max_{t\leq T}\bonus_t(s,a, h)
    \biggl[\sum_{t=1}^T  w_h^{\bpol_{t}}(s,a)\indp { \sum_{k=1}^{t-1} (1-\veps_k) w_h^{\bpol_k}(s,a) < 4 \log(SAHT^2) } \biggr]\leq \\
    &\max_{t\leq T, s, a, h}\bonus_t(s,a, h) \, O \!\lp  \sum_{t=1}^T  (1-\veps_t)w_h^{\bpol_{t}}(s,a)\indp { \sum_{k=1}^{t-1} (1-\veps_k) w_h^{\bpol_k}(s,a) < 4 \log(SAHT^2) } \rp\leq  \\
    & \max_{t\leq T, s, a, h}\bonus_t(s,a, h) \cdot O (\log T)
\end{align*}
where the latter inequality holds for $\gamma \in (0, 1)$, and we recall that $\veps_t \ceq t^{-\gamma}$.

\medskip
\noindent 
Using $$\max_{t\leq T, s,a,h}\bonus_t(s,a, h)\leq T^\alpha \lsb \sqrt{2\beta^p(T,1/T^3)}+
\frac{2\beta^p(T,1/T^3)}{3} \rsb,$$

we have 
\begin{equation}
\label{eq:term-outside-Z}
\sum_{t=1}^T\sum_{(s,a,h)\notin\cZ_t}
  w_h^{\bpol_{t}}(s,a)\,\bonus_t(s,a, h)
  \leq O (T^\alpha \beta^p(T,1/T^3) \log T).
\end{equation}

\medskip\noindent

Combining \eqref{eq:bound-A}, \eqref{eq:term-I}, and 
\eqref{eq:term-outside-Z}, and noting that $\beta^p$ is logarithmic in 
$T$ so all bonus terms are ${O}(\sqrt{T})$,
\[
\mathrm{Regret}(T)
  \leq H\,\mathrm{Regret}_\mathcal{L}(T) + O\!\lp \beta^p(T,1/T^3)\,T^{\alpha+ \frac12} \rp ,
\]
on the event $\cE_{\sect}(T)$.
\end{proof}

\bigskip

\begin{lemma}
\label{lem:optimisitic_Q}
On the event $\cE_{\sect}(T)$, for any $t \in (\sqrt{T} , T)$ and $s,a,h$, and any policy $\rho$, we have $\bar Q_{t,h}^ \rho(s,a) \geq Q_{t,h}^\rho(s,a)$. 
\end{lemma}

\begin{proof}
	We justify that for $t \in (\sqrt{T}, T)$, the bonus used in the optimistic $Q$-function is sufficient to ensure optimism on the event $\cE_{\sect}(T)$. 
	
	We proceed by induction; assuming $\bar  Q_{t, h+1}^{ \rho}(s,a) \geq Q_{t, h+1}^{\rho}(s,a)$ holds for all $s,a$ and some $h$ fixed. Note that this directly implies that $\bar V_{t, h+1}^{\rho}(s) \geq V_{t, h+1}^\rho(s)$ and we have 
	\begin{eqnarray*}
		\bar Q_{t, h}^{\rho}(s,a) &=& \loss_t(s,a,h) + \hat p_{t, h}\bar V_{t, h+1}^{\rho}(s,a) +  \sqrt{\frac{2 \var_{\hat p_{t, h} }\bigl[\overline V_{t, h+1}^{\rho}\big] (s,a) \cdot \beta^p( n_t(s,a, h), 1/t^3)}{\maxp{1}{ n_t(s,a, h)}}} + \frac{2\beta^p(n_t(s,a, h),1/t^3)g_{t, h}}{3\maxp{1}{ n_t(s,a, h)}}  \\
		& \overset{(a)}{\geq }& \loss_t(s,a,h) + \hat p_{t, h}\bar V_{t, h+1}(s,a) +  \sqrt{\frac{2 \var_{\hat p_{t, h}}\bigl[\overline V_{t, h+1}\bigr] (s,a) \cdot \beta^p(n_t(s,a, h), 1/T^2)}{\maxp{1}{ n_t(s,a, h)}}} + \frac{2\beta^p(n_t(s,a, h), 1/T^2)g_{t, h}}{3\maxp{1}{ n_t(s,a, h)}}   \\ 
		&\overset{(b)}{\geq}& \loss_t(s,a, h) +  p_h  V_{h+1}^\star(s,a)  \quad \text{(induction + monotonicity, see e.g. \citet{ pmlr-v139-menard21a})} 
	\end{eqnarray*}
	where the last inequality holds on the event $\cE_\sect(T)$ and we recall that for $t \in (\sqrt{T} , T)$, we have  $t^3 \geq T^2$, combining
	with the fact that $ \delta \mapsto \beta(x, \delta)$ is decreasing. This concludes the induction step, and as the induction hypothesis trivially holds for $H+1$, we have proved the claimed result. 
	
\end{proof}

We now state the main result of this section, which is the regret of the exploration algorithm.

 \begin{proposition}
\label{prop:guarantees-regmin}
	Running Algorithm~\ref{alg:ol} on an MDP $\mdp$ with the exploration function \textsc{ComputeExplorationPolicy} using AdaHedge as a subroutine generates a sequence of policies $(\bpol_t)_{t\geq 1}$ satisfying  on the event $\cE_{\sect}(T)$
	$$ \sup_{w \in \Omega_\mdp} \sum_{t\leq T} \sum_{s,a,h} w_{h}(s,a) \loss_t(s,a,h) - \sum_{t\leq T} \sum_{s,a,h} w_{h}^{\bpol_t}(s,a) \loss_t(s,a,h) \leq O(\beta^p(T, 1/T^3) T^{\alpha + \frac{1}{2}}) \:.$$
 \end{proposition}

\begin{proof} 
	This is a direct consequence of Proposition~\ref{prop:regret}. Indeed, using AdaHedge as subroutine ensures that after $T$ iterations, 
	we have (cf \citet{rooij_ada_hedge}, )
	$$ \regret_{\cL}(T) \leq \sigma_T \sqrt{T \log S } + 16  \sigma_T (2 + \log(S)/3), $$
	where $\sigma_T \ceq \max_{t\leq T, s, h, a} \bar Q_{t, h}^{\bpol_t} (s, a) $. 
	We note that , 
	$$\sigma_T \leq O(T^{\alpha} \beta^p(T, 1/T^3)), $$
	
	and the conclusion follows thanks to Proposition~\ref{prop:regret}. 
\end{proof}

\section{SUFFICIENT EXPLORATION AND MODEL ESTIMATION}
\label{appx:sufficient_explore}

In this section, we show that every reachable state--action pair is visited sufficiently often to guarantee accurate estimation of the unknown model parameters. Our approach achieves this by leveraging a regret minimization algorithm operating under \emph{changing reward functions}.

The use of regret minimization for exploration in MDPs has been widely studied. However, existing analyses typically rely on additional structural assumptions, such as the MDP being communicating or ergodic \citep{al-marjani23a, NEURIPS2022_27bf08fe}, the availability of lower bounds on visitation probabilities \citep{pmlr-v178-wagenmaker22a}, or the use of periodic restarts and re-initialization of the learning algorithm \citep{pmlr-v178-wagenmaker22a}. In contrast, our analysis does not require any of these assumptions.

\medskip

We define the following high-probability events:
\begin{align}
\cE_{\sect}^{p}(T)
&\ceq
\Bigl\{
\forall t \leq T,\;\forall (s,a,h):\;
n_t(s,a,h)\,\KL\!\big(\hat p_t(\cdot \mid s,a, h)\,\|\, p(\cdot \mid s,a,h)\big)
\leq \beta^p(n_t(s,a,h), 1/T^2)
\Bigr\}, \\
\cE_{\sect}^{\mu}(T)
&\ceq
\Bigl\{
\forall t \leq T,\;\forall (s,a,h):\;
n_t(s,a,h)\,\KL\!\big(\hat R_h^t(s,a)\,\|\, R_h(s,a)\big)
\leq \beta^\mu(n_t(s,a,h), 1/T^2)
\Bigr\}, \\
\cE_{\sect}^{\text{cnt}}(T)
&\ceq
\Bigl\{
\forall t \leq T,\;\forall (s,a,h):\;
n_t(s,a,h)
\geq \tfrac{1}{2}\,\bar n_t(s,a,h) - \log(2SAHT^2)
\Bigr\}.
\end{align}

Here,
\begin{equation}
	\label{eq:def-barn}
	\bar n_t(s,a,h)
\ceq 
\sum_{k=1}^{t-1} \biggl[(1-\veps_k)\, w_h^{\bpol_k}(s,a) + \veps_k\, w_h^{c_k}(s,a)\biggr].
\end{equation}

We recall that  $Z_t \sim \mathrm{Bernoulli}(t^{-\gamma})$ and   $$
     \bpolmix_t(\cdot | s, h) \;\leftarrow\;
     \begin{cases}
       c_t(\cdot | s, h), & \text{if } Z_t = 1,\\
       \bpol_t(\cdot | s, h), & \text{if } Z_t = 0,
     \end{cases}
     \quad \forall (s,h)\in \cS \times [H].
  $$
 We introduce the event 

\begin{equation}
\label{eq:sum-kappa}
	\cE_{\sect}^Z(T) \ceq \lb \forall (s,a,h), \, \forall t\leq T,\;   \sum_{k=1}^t  \veps_k \cdot \indp{\tilde s^k=s, \tilde a^k =a, \tilde h^k = h}  \geq   \frac{1}{2SAH} \sum_{k=1}^t \veps_k   - \log(2SAHT^2) \rb.
\end{equation} 

We further define
\begin{equation}
\label{eq:event-forced}
	\cE_{\sect}(T)
\ceq
\cE_{\sect}^{p}(T)
\cap
\cE_{\sect}^{\mu}(T)
\cap
\cE_{\sect}^{\text{cnt}}(T) \cap \cE_{\sect}^Z(T) 
\end{equation}

\newcommand{\test}{\sect}

\subsection{Online Regret of UCBVI with Changing Rewards}
We recall that a state--action pair $(s,a)$ is said to be reachable at stage $h$ if $\max_{\rho} w_h^\rho(s,a) > 0$, i.e., if there exists a policy that visits $(s,a)$ at stage $h$ with positive probability.

We now show that, on event $\cE_{\sect}(T)$, every reachable state--action pair is visited sufficiently often to ensure the consistency of the empirical estimates used by the algorithm.

To enforce exploration, we introduce a forced exploration mechanism based on a single step of UCBVI with a known reward function defined as the indicator of a randomly selected state--action--stage triplet. Formally, at episode $t$, a triplet $(\tilde s^t, \tilde a^t, \tilde h^t)$ is sampled uniformly at random, and we define the reward kernel $\kernforce_t$ as
\begin{equation}
\label{eq:reward_kern}
\kernforce_t(s,a,h) \;\ceq\; \indp{ s = \tilde s^t,\, a = \tilde a^t,\, h = \tilde h^t } \:.
\end{equation}

Let $ \mdp_t^\mathrm{f} \ceq (\cS, \cA, H, p, \kernforce_t, s_\mathrm{init})$ denote the corresponding MDP. For any policy $\pi$, the value at the initial state satisfies
\[
V_{t, 0}^{\pi} = w_{\tilde h^t}^\pi(\tilde s^t, \tilde a^t),
\]
that is, the visitation probability of $(\tilde s^t, \tilde a^t)$ at stage $\tilde h^t$ under policy $\pi$.

Consequently, an optimal policy for $\mdp_t^\mathrm{f}$ maximizes the visitation probability of the sampled triplet $(\tilde s^t, \tilde a^t, \tilde h^t)$.

Finally, we note that maximizing $w_h^\pi(s,a)$ is equivalent to maximizing $w_h^\pi(s)$, since the action chosen at stage $h$ does not affect the probability of reaching state $s$ at that stage. Therefore,
\[
W_h(s,a) \ceq \max_{\pi} w_h^\pi(s,a) = \max_{\pi} w_h^\pi(s).
\]

\subsubsection{Algorithm}

The forced exploration policy is computed via a single optimistic planning step using the empirical transition kernel $\hat p_t$. The goal is to construct an optimistic estimate of the optimal $Q$-function in the MDP $\mdp^\mathrm{f}_t \ceq (\kernforce_t, p)$. 

In this section, let $\bar Q_t$ and $\bar V_t$ denote the resulting optimistic state--action and value functions. For notational simplicity, we omit the explicit dependence on the reward kernel $\kernforce_t$.

We define the exploration bonus
\begin{equation}
\bonus_t(s,a,h) \;\ceq\;
\sqrt{
\frac{2 \, \var_{\hat p_{t,h}}\!\big[\bar V_{t,h+1}\big](s,a)\; \beta^p\!\big(n_t(s,a,h),\, 1/t^3\big)}
{\maxp{1} {n_t(s,a,h)}}
}
\;+\;
\frac{2\, \beta^p\!\big(n_t(s,a,h),\, 1/t^3\big)}
{3\, \maxp{1}{ n_t(s,a,h)}} \:.
\end{equation}

We then construct the optimistic $Q$-function as
\begin{equation}
\label{eq:q-opt-fe}
\bar  Q_{t,h}(s,a)
\;\triangleq\;
\kernforce_t(s,a,h)
\;+\;
\hat p_{t,h} \bar V_{t,h+1}(s,a)
\;+\;
\big( \bonus_t(s,a,h) \wedge 1 \big).
\end{equation}

The clipping by $1$ ensures consistency with the reward range.

The exploration policy $c_t$ is defined as the greedy policy with respect to $\bar Q_t$, namely
\[
c_{t}(s, h)
\;\ceq\;
\argmax_{a \in \cA} \bar Q_{t,h}(s,a),
\quad \forall (s,h) \in \cS \times \lbrack H \rbrack.
\]

In the sequel, we denote by $V_{t, \cdot}^{\rho}$ the value function in $\mdp_t^\mathrm{f}$, and, by $\bar V_{t, \cdot}^\rho$ the value function in the "optimistic MDP'' $\bar \mdp_t^\mathrm{f} \ceq (\kernforce_t + \bonus_t, \hat p_t)$. The associated $Q$-value functions are defined similarly. 

\bigskip 
\hrulefill
\begin{algorithm}[H]
\Fn{\textsc{ComputeForcedExplorationPolicy}$(t, \hat p_t, n_t)$}{
\caption{\textsc{Forced Exploration Helper Function}\label{alg:appx_forcexp}}
\DontPrintSemicolon
\SetInd{0.5em}{1em}   
\SetNlSty{textbf}{}{} 
\SetAlgoNlRelativeSize{-1} 
\LinesNumbered
\SetAlgoLined
\SetAlgoVlined

Sample $(\tilde s^t,\tilde a^t,\tilde h^t)$ uniformly at random from $\cS \times \cA \times [H]$\;

Set reward function $\kernforce_t(s,a,h) \gets \indp{ s=\tilde s^t,\, a=\tilde a^t,\, h=\tilde h^t }$\;

Initialize $\bar V_{t, H+1}(s) \equiv 0,\;\forall s \in \cS$\;

\For{$h = H, H-1, \ldots, 1$}{
  \For{$(s,a) \in \cS \times \cA$}{
    $\bar Q_{t, h}(s,a) \gets
      \kernforce_t(s,a,h)
      + \hat p_{t, h} \bar V_{t, h+1}(s,a)
      + \bonus_t(s,a,h)$\;
  }
  $\bar V_{t, h}(s) \gets \max_{a \in \cA}\bar  Q_{t, h}(s,a), \quad \forall s \in \cS$\; 
}

\tcp*{Greedy policy with respect to optimistic $Q$}
\For{$(s,h) \in \cS \times [H]$}{
  $ c_{t}(s, h) \gets \argmax_{a \in \cA} \bar  Q_{t, h}(s,a)$\;
}
\Return policy $c_{t}$\;
}\end{algorithm}

\bigskip
\subsubsection{Regret Upper Bound}
We define the well-explored set for a fixed $T$
\begin{equation}
\label{eq:eq-ref-zt}
\cZ_t
\ceq
\Bigl\{
(s,a,h):\;
\bar n_t(s,a,h) \geq 4 \log(2SAHT^2)
\Bigr\}.
\end{equation}

On the event $\cE_{\sect}^{\text{cnt}}(T)$, we have for all $(s,a,h) \in \cZ_t$:
\begin{equation}
\label{eq:op-lo-pm}
n_t(s,a,h)
\geq
\tfrac{1}{4}\,\bar n_t(s,a,h).
\end{equation}

We also introduce for any sequence $(\kappa_t)_{t\geq 1}$, the sum operator $\kappa_{1:T} \ceq  \sum_{t=1}^T \kappa_t$.

We prove the following result. 

\medskip

\begin{lemma}
\label{thm:forceexp_reg}
Let $(S_t)_{t \geq 1}$ be a sequence of binary random variables. On the event $\cE_{\sect}(T)$, the policies $(c_t)_{t\geq 1}$ produced by Algorithm~\ref{alg:appx_forcexp} satisfy for any $K \in (\sqrt T, T)$, 
\[
\sum_{t=1}^K S_t \varepsilon_t
\Bigl(
V_{t,0}^{\star}
-
V_{t,0}^{c_t}
\Bigr)
\;\leq\;
\sqrt{T^{1-\gamma}} + O(\log T) + O\!\lp 
\sqrt{\beta^p(T, 1/T^2)}
\sqrt{
1 + \sum_{t=1}^K \kappa_t \, w_h^{c_t}(s,a)
}\rp, 
\] where $\kappa_t = S_t \veps_t$.
\end{lemma}

\begin{proof} We decompose
\begin{equation}
\label{eq:eq-def-to-bound} 
\sum_{t=1}^K \kappa_t 
\bigl( 
V_{t,0}^{ \star}  
- V_{t,0}^{ c_t} 
\bigr)
\;\leq\;
\underbrace{\sum_{t=1}^{\sqrt{T}} \kappa_t}_{(I)}
+
\underbrace{\sum_{t=\sqrt{T}+1}^K 
\kappa_t 
\bigl( 
\bar V_{t, 0}  
- V_{t,0}^{ c_t} 
\bigr)}_{(II)}.
\end{equation}

We focus on controlling $(II)$.

For $t \geq \sqrt{T}$, we have $1/t^3 \leq 1/T^2$. Since $\delta \mapsto \beta^p(\cdot,\delta)$ is non-increasing, the exploration bonus used in Algorithm~\ref{alg:appx_forcexp} dominates the one associated with confidence level $1/T^2$. Therefore, when $\cE_{\sect}^p(T)$ holds, from classical UCBVI analysis technique (see e.g. \citet{pmlr-v70-azar17a}) 
\[
\bar Q_{t, h}(s,a) \;\geq\; Q_{t, h}^{\star}(s,a),
\quad \forall (s,a,h), \; \forall t \in (\sqrt{T}, T).
\]
By backward induction, this implies
\[
\bar V_{t, h}(s) \;\geq\; V_{t, h}^{\star}(s),
\quad \forall (s,h).
\]

Since $c_t$ is greedy with respect to $\bar Q_{t}$, standard decomposition with value difference lemma yields 
\[
\bar V_{t, 0} - V_{t, 0}^{c_t}
\;\leq\;
\sum_{s,a,h} 2\, w_h^{c_t}(s,a)\, \bigl( \bonus_t(s,a,h) \wedge 1 \bigr).
\]
We then obtain
\[
(II)
\;\leq\;
\sum_{t=1}^K \sum_{s,a,h}
2\, \kappa_t\, w_h^{c_t}(s,a)\, \bigl( \bonus_t(s,a, h) \wedge 1 \bigr).
\]

For $(s,a,h) \notin \cZ_t$, we have $\bar n_t(s,a,h) < 4 \log(SAHT^2)$, hence
\begin{align}
\sum_{t=1}^K \sum_{(s,a,h)\notin \cZ_t}
2 \kappa_t w_h^{c_t}(s,a)\bigl( \bonus_t(s,a, h) \wedge 1 \bigr)
&\leq
\sum_{t=1}^K \sum_{s,a,h}
2 \varepsilon_t w_h^{c_t}(s,a)
\indp{ \sum_{k=1}^{t-1} \varepsilon_k w_h^{c_k}(s,a) < 4 \log(SAHT^2) } \notag\\
&\leq 2 + 8 \log(SAHT^2).
\label{eq:sum-nozt}
\end{align}

For the well-explored triplets, let us define
\[
(III) \;\ceq\;
\sum_{t=1}^K \sum_{(s,a,h)\in \cZ_t}
2 \kappa_t w_h^{c_t}(s,a)\, \bigl( \bonus_t(s,a,h) \wedge 1 \bigr).
\]

\smallskip
\noindent

Using Eq.~\eqref{eq:op-lo-pm}
\begin{align}
\label{eq:term-xxcw-df}
\sum_{t=1}^K \sum_{(s,a,h)\in \cZ_t}
\kappa_t w_h^{c_t}(s,a)
\frac{\beta^p(n_t(s,a,h),1/T^2)}{\maxp{1}{n_t(s,a,h)}}
&\leq
O \!\lp  \beta^p(T,1/T^2)
\log\!\biggl(1 + \tfrac14 \sum_{t=1}^K \kappa_t w_h^{c_t}(s,a)\biggr)  \rp
\end{align}
which follows by elliptic potential (Lemma~\ref{lem:elliptic-potential}), the definition of $\cZ_t$, and monotonicity of $\beta^p$. 

\smallskip
\noindent

Similarly, since $\bar V_{t, h+1}(\cdot) \in (0,1)$, we have
$\var_{\hat p_{t,h}}\!\big[\bar V_{t, h+1}\big](s,a) \leq 1$. Hence,
\begin{align*}
&\sum_{t=1}^K \sum_{(s,a,h) \in \cZ_t}
\kappa_t \, w_h^{c_t}(s,a)
\sqrt{
\frac{
2\, \var_{\hat p_{t,h}}[\bar  V_{t, h+1}](s,a)\, \beta^p(n_t(s,a,h), 1/T^2)
}{
\maxp{ 1}{ n_t(s,a,h)}
}
} \leq \\
&\quad 
\sum_{t=1}^K \sum_{(s,a,h) \in \cZ_t}
\kappa_t \, w_h^{c_t}(s,a)
\sqrt{
\frac{
2\, \beta^p(T, 1/T^2)
}{
\maxp{ 1}{ \tfrac{1}{4}\,\bar n_t(s,a,h) }
}
} \leq \\
&\quad 
4 \sqrt{2\, \beta^p(T, 1/T^2)}
\sum_{t=1}^K \sum_{s,a,h}
\frac{
\kappa_t \, w_h^{c_t}(s,a)
}{
\sqrt{
\maxp{ 1}{\tfrac{1}{4} \sum_{k=1}^{t-1} \kappa_k w_h^{c_k}(s,a)}
}
},
\end{align*}
where the last inequality uses 

$$ \bar n_t(s,a,h) \geq \sum_{k=1}^{t-1} \kappa_k  w_h^{c_t}(s,a), $$
which follows from \myeqref{eq:def-barn}, where $\kappa_k = \veps_k S_k$ and $S_k$ is binary. 
 
Applying Lemma~16 of \citet{pmlr-v132-kaufmann21a}, we obtain
\[
\sum_{t=1}^K \sum_{(s,a,h) \in \cZ_t}
\kappa_t \, w_h^{c_t}(s,a)
\sqrt{
\frac{
2\, \var_{\hat p_{t,h}}[\bar V_{t, h+1}](s,a)\, \beta^p(n_t(s,a,h), 1/T^2)
}{
\maxp{ 1}{ n_t(s,a,h)}
}
} \leq O\!\lp 
\sqrt{\beta^p(T, 1/T^2)}
\sqrt{
1 + \sum_{t=1}^K \kappa_t \, w_h^{c_t}(s,a)
}\rp .
\]

Combining this bound with \myeqref{eq:term-xxcw-df} yields, on the event $\cE_{\sect}(T)$,
\begin{equation}
\label{eq:sum-op-ui}
(III) \;\leq\;
O\!\lp 
\sqrt{\beta^p(T, 1/T^2)}
\sqrt{
1 + \sum_{t=1}^K \kappa_t \, w_h^{c_t}(s,a)
}\rp 
\end{equation}

Finally, combining Eq.~\eqref{eq:eq-def-to-bound}, Eq.~\eqref{eq:sum-nozt}, and Eq.~\eqref{eq:sum-op-ui} concludes the proof.
\end{proof}

\subsection{Suffficient Exploration of Reachable States}
\label{subsec:explore_reachable} 
In this section, we prove sufficient exploration of reachable state--actions. We recall that 

$$ W_h(s,a) \ceq \sup_{\rho} w_h^\rho(s,a), $$
is the maximum visitation probability of  triplet $(s,a,h)$. 

We prove the following lemma. 

\medskip

\begin{lemma}
\label{lem:forced_exp}
Let $\gamma\in(0,1)$. There exists an explicit finite constant $k_0<\infty$, 
such that on the event $\cE_{\sect}(T)$, for all $T\geq k_0$, all 
$t\in(T^{2/3},T)$, and all $(s,a,h)$,
\[n_t(s,a,h) \geq \frac{W_h(s,a)}{4SAH} t^{-\gamma}.\]
The explicit expression of $k_0$ is given in the proof.
\end{lemma}

\begin{proof}
Let us fix a triplet $(s,a,h)$ and $t\in (T^{2/3}, T)$. We note that the claim is trivially satisfied for unreachable triplets. Hence, in the sequel, we assume $W_h(s,a,h)>0$. 

We have 
\begin{eqnarray*}
\bE \bigl[\indp{ s_{h}^t =s , a_{h}^t=a } \big |\cH_{t-1} \bigr] &=& (1-\veps_t) \;  \bE \bigl[  \indp{ s_{h} =s , a_{h}=a } \big | \cH_{t-1}, Z_t = 0 \bigr]   + \\ &&  \veps_t \; \bE \bigl[  \indp{ s_{ h}^t =s , a_{h}^t =a } \big |\cH_{t-1}, Z_t = 1 \bigr], \\
&\geq &   \veps_t  \,  w_h^{c_t}(s,a) \:. 
\end{eqnarray*}

Next, when $\cE_{\sect}^{\text{cnt}}(T)$ holds, \begin{eqnarray*}
 	n_t(s,a,h) &\geq& \frac12 \sum_{k=1}^{t-1} \bE\bigl[\indp {s_{h}^k = s , a_{h}^k = a } \big | \cH_{k-1} \bigr] - \log(6 SAHT^2)\:,\\ 
 	&\geq& \frac12 \sum_{k=1}^{t-1}  \veps_k   w_h^{c_k}(s,a) - \log(6 SAHT^2),
 	\end{eqnarray*} 
and the right-hand side is related to the dynamic regret of the UCBVI forced exploration routine as  
 \begin{eqnarray}
 \label{eq:-er-op-lm}
 	\sum_{k=1}^{t-1}  \veps_k   w_h^{c_i}(s,a)  &\geq& \sum_{k=1}^{t-1} \veps_k \indp{ \tilde s^k=s, \tilde a^k =a, \tilde h^k = h } V_{t,0}^{ c_k}. 
 \end{eqnarray}

Thanks to Lemma~\ref{thm:forceexp_reg}, on the event $\cE_{\sect}(T)$, with  $ S_k \ceq   \indp{\tilde s^k=s, \tilde a^k =a, \tilde h^k = h },$ and $\kappa_k  \ceq \veps_k S_k$
\begin{equation*}
	\sum_{k=1}^{t-1} \kappa_k V_{k,0}^{ c_k} \geq \sum_{k=1}^{t-1} \kappa_k V_{k, 0}^{\star} -  O(\sqrt{T^{1 - \gamma}}) - O(\log T) - O\!\lp 
\sqrt{ t^{1-\gamma} \cdot \beta^p(T, 1/T^2)} \rp. 
	\end{equation*} 
	
Rearranging the above displays yields 
\begin{align*}
	&n_t(s,a, h) \geq \\ & \frac12 {\sum_{k=1}^{t-1} \biggl[ \veps_k \indp{\tilde s^k=s, \tilde a^k =a, \tilde h^k = h } \biggr]}  \max_{\rho} w_{h}^\rho(s,a) -   O(\sqrt{T^{1 - \gamma}}) - O(\log T) - O\!\lp 
\sqrt{t^{1-\gamma} \cdot  \beta^p(T, 1/T^2)}
\rp.
\end{align*}

\medskip
 We remark that on the event $\cE^Z_{\sect}(T)$, 
$$\sum_{k=1}^{t-1} \veps_k \cdot \indp{\tilde s^k=s, \tilde a^k =a, \tilde h^k = h}   \geq \frac1{2SAH} \sum_{k=1}^{t-1} \veps_k - \log(2SAHT^2) $$

we have for any $ t \in (T^{2/3} , T)$, and any triplet $(s,a,h)$ 
\begin{eqnarray}
\label{eq:almost_end}
	n_{t}(s,a, h) &\geq& \frac{W_{h}(s,a)}{2SAH} \veps_{1\ldots t-1} 
	-   O\!\lp \sqrt{t^{\frac{3(1-\gamma)}{2}}}\rp - O(\log t) - O\!\lp 
\sqrt{t^{1-\gamma}\cdot \beta^p(t^{3/2}, 1/t^3)}
\rp.
\end{eqnarray}

Now, let us define 
\begin{align}
\begin{split}
	T_0({s,a,h}) \ceq \sup \lb t\geq 1:  n_{t}(s,a,h) \leq \frac{W_h(s,a)}{4SAH}t^{1-\gamma} \rb.
\end{split}
\end{align}

From \myeqref{eq:almost_end}, and, as $\veps_{1\ldots t} \sim \frac{t^{1-\gamma}}{1-\gamma}$, we know that on $\cE_{\sect}(T)$, $k_0(s,a,h)$ is upper bounded by a deterministic constant  $k_0(s,a,h)$ defined as 
$$ k_0(s,a,h)^{2/3} \ceq \sup\lb t \geq 1: \frac{\veps_{1\ldots t}}{2SAH} W_{h}(s,a) - \frac{t^{1-\gamma}}{4SAH} W_{h}(s,a)  \leq      O\!\lp \sqrt{t^{\frac{3(1-\gamma)}{2}}}\rp + O(\log t)  + O\!\lp 
\sqrt{t^{1-\gamma} \cdot \beta^p(t^{3/2}, 1/t^3)} \rp \rb.$$ 
 Indeed, the expressions in $O(\cdot)$ can be recovered from the proof to make the above well-defined, moreover, it is finite for any $\gamma \in (0, 1)$ and any reachable $(s,a,h)$. 

Next we define 
$$ k_0 \; \ceq\; \max_{(s,a,h) : W_h(s,a) >0} k_0({s,a,h}),$$
which concludes the proof. 
\end{proof}

\subsection{Guarantees on Model Estimation}
\label{subsec:model_est}
We prove that after a finite number of episodes the empirical optimal 
policy coincides with the true optimal policy on most episodes, when $\cE_{\sect}(T)$ holds.
The key tool is the following value difference lemma.

\begin{lemma}[Value Difference]
\label{lem:val-diff-lem}
For any policy $\pi\in\Pi$,
\[
V_0^\pi - \hat{V}_{t,0}^\pi
  = \bE_\mdp^\pi\!\left[\sum_{h=1}^H b_t(s_h,a_h, h)\right],
\]
where $b_t(s,a, h) \ceq \mu(s,a, h)-\muh_t(s,a,h)
+\bigl(p_h-\hat{p}_{t, h} \bigr)\hat{V}_{t, h+1}^{\pi}(s,a)$.
\end{lemma}

\begin{proof}
By Bellman's equations,
\begin{align*}
V_h^\pi(s)-\hat{V}_{t,h}^{\pi}(s)
  &= \bE_{a\sim\pi_h(\cdot\mid s)}\!\Big[
    \mu_h(s,a)-\muh_t(s,a,h)
    +p_h V_{h+1}^\pi(s,a)
    -\hat{p}_{t, h} \hat{V}_{t, h+1}^{\pi}(s,a)
  \Big] \\
  &= \bE_{a\sim\pi_h(\cdot\mid s)}\!\Big[
    b_t(s,a, h)
    + p_h\bigl(V_{h+1}^\pi-\hat{V}_{t, h+1}^{\pi}\bigr)(s,a)
  \Big].
\end{align*}
Applying this recursion from $h$ down to $H$ and using 
$s_1=s_{\mathrm{init}}$ gives the claimed result.
\end{proof}

\medskip\noindent
Combining Lemma~\ref{lem:val-diff-lem} with the sufficient exploration 
guarantee of Section~\ref{subsec:explore_reachable} and standard 
concentration bounds shows that the empirical value of any policy 
converges quickly to its true value. In particular, once the right-hand 
side of Lemma~\ref{lem:val-diff-lem} falls below
\begin{equation}
\label{eq:def-delta-min}
\Delta_{\min}(\Pi) \ceq V_0^\star
  - \max_{\pi\in\Pi\setminus\{\pi^\star\}} V_0^\pi,
\end{equation}
the empirical optimal policy $\hat\pi_t$ must equal $\pi^\star$. This is formalized in the next lemma.

\begin{lemma}
\label{lem:estimation}
There exists an explicit finite constant $k_0 < \infty$ such that on the event $\cE_{\sect}(T)$, for all $T > k_0$ and all $t\in(T^{2/3}, T)$,
\[
\hat\pi_t = \pi^\star.
\]
The constant $k_0$ is given explicitly in the proof.
\end{lemma}
\begin{proof}
Since $\pi^\star$ is the unique optimal start-state policy for $\mdp$, it suffices to show that $\hat{V}_{t,0}^{\pi^\star} > \hat{V}_{t,0}^{\pi}$ for all $\pi\neq\pi^\star$.

\medskip\noindent
By Lemma~\ref{lem:val-diff-lem}, for any $\pi\in\Pi$,
$V_0^\pi - \hat{V}_{t,0}^{\pi} = \bE_\mdp^\pi\!\left[\sum_{h=1}^H 
b_t(s_h,a_h, h)\right]$.

Since $\hat{V}_{t, h+1}^{\pi}(s)\leq H$ for all $(s,h)$, Pinsker's 
inequality gives, on $\cE_{\sect}(T)$,
\[
\lvert b_t(s,a, h) \rvert 
  \leq \sqrt{\frac{2\beta^\mu(t,1/t^2)}{n_t(s,a,h)}}
  + \sqrt{\frac{2H^2\beta^p(t,1/t^2)}{n_t(s,a,h)}}.
\]

By Lemma~\ref{lem:forced_exp}, on $\cE_{\sect}(T)$, there  exists another $k_0'<\infty$ such that 
for $T\geq k_0'$, and $t \in (T^{2/3}, T)$, and any 
reachable $(s,a,h)$, $$n_t(s,a,h)\geq \frac{W_h(s,a)}{4SAH}t^{1-\gamma}$$. Therefore, when $\cE_{\sect}(T)$ holds, for $T\geq k_0'$, and $t \in (T^{2/3}, T)$,
\begin{equation}
\label{eq:b-bound}
\lvert b_t(s,a, h) \rvert 
  \leq \sqrt{\frac{8SAH\,\beta^\mu(t,1/t^2)}{W_h(s,a)\,t^{1-\gamma}}}
  + \sqrt{\frac{8SAH^3\beta^p(t,1/t^2)}{W_h(s,a)\,t^{1-\gamma}}}.
\end{equation}

\medskip\noindent
Since $w_h^\pi(s,a)=0$ for unreachable triplets $(s,a,h)$,
\[
\biggl | V_0^\pi - \hat{V}_{t,0}^{\pi} \biggr| 
  = \biggl | \sum_{\substack{s,a,h:\\W_h(s,a)>0}} w_h^\pi(s,a)\,b_t(s,a, h) \biggr| 
  \leq \sum_{\substack{s,a,h:\\W_h(s,a)>0}} w_h^\pi(s,a)
  \left[
    \sqrt{\frac{8SAH\,\beta^\mu(t,1/t^2)}{W_h(s,a)\,t^{1-\gamma}}}
    +\sqrt{\frac{8SAH^3\beta^p(t, 1/t^2)}{W_h(s,a)\,t^{1-\gamma}}}
  \right].
\]
We define
\begin{equation}
\label{eq:def-T0}
T_0^{2/3} \ceq \max_{\substack{s,a,h:\\W_h(s,a)>0}}
  \sup\left\{t\geq 1 :
    \sqrt{\frac{8SAH\,\beta^\mu(t,1/t^2)}{W_h(s,a)\,t^{1-\gamma}}}
    +\sqrt{\frac{8SAH^3\beta^p(t,1/t^2)}{W_h(s,a)\,t^{1-\gamma}}}
    \geq \frac{\Delta_{\min}(\Pi)}{2}
  \right\},
\end{equation}
which is finite since $\beta^\mu,\beta^p$ are logarithmic in $t$, 
$\gamma\in(0,1)$, and $\Delta_{\min}(\Pi)>0$. We further let $k_0 = \maxp{k_0'}{ T_0}$. 

\medskip\noindent

Thus, when $\cE_{\sect}$ holds, for $T\geq k_0$, $t \in (T^{2/3}, T)$ and any $\pi\neq\pi^\star$,
\[
\hat{V}_{t,0}^{\pi^\star} - \hat{V}_0^{t,\pi}
  \geq \Delta_{\min}(\Pi)
  + \underbrace{\hat{V}_{t,0}^{\pi^\star} - V_0^{\pi^\star}}_{>\,-\Delta_{\min}(\Pi)/2}
  + \underbrace{V_0^{\pi} - \hat{V}_{t,0}^{\pi}}_{>\,-\Delta_{\min}(\Pi)/2}
  > 0,
\]
so that $\hat\pi_t = \pi^\star$. 
\end{proof}

\newpage 
\section{ANALYSING CONDITIONAL POSTERIOR SAMPLING}
\label{appx:learning_on_alt}
In this section, we analyze the guarantees on the inf player (distribution over $\alt(\empmdp_T)$) by interpreting it as an instance of continuous exponential weights (CEW). The goal of these results is to demonstrate that conditional posterior sampling, akin to continuous exponential weights, is a no-regret learner of the best challenger MDP. 

We recall the log-likelihood ratio between a model $(\mdqr, \mdqp)$ and the empirical MDP $\empmdp_{T}$ after $(T-1)$ episodes 
\begin{equation}
	\Upsilon_{T}(\mdqr,\mdqp) \ceq  \sum_{s,a,h}n_T(s,a,h)\lsb  \frac{\bigl(\mdqr(s,a,h) - \muh_T(s,a, h)\bigr)^2}{2} + \sum_{s'} \hat p_T(s'\mid s,a,h)\log \frac{\hat p_t(s'\mid s,a,h)}{\mdqp(s'\mid s,a,h)}\rsb \:. 
\end{equation}

The conclusion of this section is to prove the result below. 
\begin{theorem} 
\label{thm:inf-no-regret}
	Let $\eta_t \propto t^{-\varsigma}$ with $\varsigma >2\alpha$. There exists a constant $k_0 < \infty$ and a deterministic function $f(T)  = o(T)$ such for any $T\geq k_0$, the inequality 
\begin{align*}
	&\sum_{t=T^{2/3}}^{T-1} \bE_{(\mdqr, \mdqp )\sim \nu_t^{\eta_t}(\cdot \mid \alt(\empmdp_t))}\lsb  \ell^\alpha_t(\mdqr,\mdqp)\rsb \leq \\ & \inf_{(\mdqr, \mdqp) \in \alt(\empmdp_T)}\sum_{t=1}^{T-1}\sum_h \lsb \frac{(\muh_{T}(s_h^t,a_h^t, h ) - \mdqr(s_h^t,a_h^t, h))^2}{2} + \log \frac{1}{\mdqp(s_{h+1}^t, \mid s_h^t, a_h^t,  h)}\rsb      + f(T) \:, \text{ where }
\end{align*}
	
$$\ell^\alpha_{t}(\mdqr, \mdqp) \ceq \sum_{h}  \frac{(\muh_{t}\bigl (s_h^{t},a_h^{t}, h) - \mdqr(s_h^{t},a_h^{t}, h)\bigr)^2}{2}  +  \log \frac1{\maxp{\mdqp(s_{h+1}^{t}\mid s_h^{t},a_h^{t}, h)}{   \myexp{-t^\alpha} } } \:,	$$ holds with probability at least $1-O(1/T^2)$. Moreover, $f$ can be made explicit from the proof. \end{theorem}

\subsection{Setup}
Before introducing the proof of this result, let us introduce the function over reward parameters and the function over the transition kernels defined for some $\eta>0$ by
\begin{align}
\label{eq:def-f-g}
\begin{split}
	f_t^{\eta}(\mdqr) &\ceq  \expb{ - \sum_{s, a, h} \frac{\eta \cdot n_{t}(s,a, h)}2 (\muh_{t}(s,a,h ) - \mdqr(s,a, h))^2} \:, \\
	g_t^\eta (\mdqp) &\ceq  \expb{ \sum_{s,a,h} \sum_{s'} \eta \cdot n_t(s'|s,a,h)\log \mdqp(s'|s,a, h)} \:. 
	\end{split}
\end{align}

Since $f_{t}^\eta, g_{t}^\eta$ are random variables, all results claimed in the next lemma hold with probability $1$.  
\begin{lemma}
\label{lem:ratio-f-g}
Let $\eta>0$ and let $\mdp \ceq (\mdqr, \mdqp) \in \Mdp$ be any MDP with 
positive transitions. After observing the trajectory 
$\big((s_h^t, a_h^t, r_h^t, s_{h+1}^t)\big)_{h=1}^H$ at episode $t$, 
the following identities hold with probability $1$:
\begin{align}
\frac{g_{t+1}^\eta(\mdqp)}{g_t^\eta(\mdqp)} 
  &= \expb{ -\eta \sum_{h=1}^H \log \frac{1}{\Phi(s_{h+1}^t \mid s_h^t, a_h^t, h)}} , 
  \label{eq:ratio-g} \\[6pt]
\frac{f_{t+1}^\eta(\mdqr)}{f_t^\eta(\mdqr)} 
  &= \expb{ -\frac{\eta}{2} \sum_{h=1}^H 
     \bigl(\muh_t(s_h^t,a_h^t,h) - \mdqr(s_h^t,a_h^t,h)\bigr)^2 
     + \frac{\eta}{2}\,\sigma_{t+1}(\mdqr) },
  \label{eq:ratio-f}
\end{align}
where
\begin{equation}
\sigma_{t+1}(\mdqr) \ceq \sum_{h=1}^H n_{t+1}(s_h^t, a_h^t, h)
  \Bigl[
    \bigl(\muh_t(s_h^t,a_h^t,h)   - \mdqr(s_h^t,a_h^t,h)\bigr)^2 -
    \bigl(\muh_{t+1}(s_h^t,a_h^t,h) - \mdqr(s_h^t,a_h^t,h)\bigr)^2
  \Bigr].
  \label{eq:def-sigma}
\end{equation}
\end{lemma}

\begin{proof}
 Noting that any transition $(s'\mid s,a,h)$ not observed at the end of episode $t$ satisfies $n_t(s'\mid  s,a,h) = n_{t+1}(s' \mid  s, a, h)$, the definitions in \myeqref{eq:def-f-g} give 
	\begin{eqnarray*} 
\frac{g_{t+1}^\eta(\mdqp)}{g_t^\eta(\mdqp)} &=&  \expb{ \sum_{h} \eta \cdot n_{t+1}(s_{h+1}^{t}| s_h^{t},a_h^{t}, h) \log \mdqp(s_{h+1}^{t}|s_h^{t},a_h^{t}, h) - \sum_{h} \eta \cdot n_{t}(s_{h+1}^{t} \mid  s_h^{t},a_h^{t}, h) \log \mdqp(s_{h+1}^{t}|s_h^{t},a_h^{t},  h)}, 	\\
&\overset{(a)}{=}& \expb{ \sum_h  \eta \log \mdqp(s_{h+1}^{t}\mid s_h^{t},a_h^{t},  h) }, \\&=&  \expb{  - \eta\sum_h   \log \frac1{\mdqp(s_{h+1}^{t} \mid  s_h^{t},a_h^{t}, h)} },  
\end{eqnarray*} 
where $(a)$ follows since the observed transition $(s_h^t, a_h^t, s_{h+1}^t)$
contributes exactly $+1$ to the count.

\smallskip

The proof of the second statement of the lemma follows by noting that 
\begin{eqnarray*}
\frac{f_{t+1}^\eta (\mdqr)}{f_t^\eta(\mdqr)} &=&  \expb{ - \sum_{h} \frac{\eta \cdot n_{t+1}(s_h^{t}, a_h^{t}, h)}{2} (\muh_{t+1}(s_h^t,a_h^t, h ) - \mdqr(s_h^t,a_h^t, h))^2 + \right. \\ && \qquad \qquad \left.  \sum_{h} \frac{\eta \cdot n_t(s_h^{t}, a_h^{t}, h)}{2} (\muh_{t}(s_h^t,a_h^t, h) - \mdqr(s_h^t,a_h^t, h))^2}, 	\\
&\overset{(b)}{=}& \exp\lb  \sum_{h} \frac{\eta \cdot n_{t+1}(s_h^{t}, a_h^{t}, h)}{2}\Bigl[(\muh_{t}(s_h^t,a_h^t, h) - \mdqr(s_h^t,a_h^t, h))^2 -(\muh_{t+1}(s_h^t,a_h^t, h ) - \mdqr(s_h^t,a_h^t, h))^2 \Bigr] - \right . \\ && \qquad  \qquad \left. \sum_{h} \frac{\eta}{2} (\muh_{t}(s_h^t,a_h^t, h) - \mdqr(s_h^t,a_h^t, h))^2\rb, 
\end{eqnarray*}
where $(b)$ follows since $n_{t+1}(s_h^t, a_h^t, h) = n_{t}(s_h^t, a_h^t, h) +1$. Indeed, since $(s_h^t, a_h^t)$ is visited at stage $h$ during episode $t$, the count increases by exactly $1$. 

With the definition of $\sigma_{t+1}$, we have proven the claimed statement. 
\end{proof}

\subsection{Posterior Density Ratio and Stochastic Hedge Loss}
\medskip

 We introduce further notation related to the posterior density. As the transition and rewards are independent, it is a simple exercise to show that the posterior density is simply the product of $g_t^{\eta_t}$ and $f_t^{\eta_t}$.  

For $\eta >0$, we introduce $\lambda^\eta$ defined on $\Mdp$ as  
\begin{equation}
\label{eq:def-lbd}
	\lambda_t(\mdqr,\mdqp) \ceq - \frac1\eta \log \lb f_t^\eta(\mdqr ) g_t^\eta(\mdqp) \rb \:, 
\end{equation}
for an MDP $\mdq \ceq (\mdqr, \mdqp)$, and we define a density \wrt canonical measure of $\Mdp$ proportional to 
$$f_t^\eta(\mdqr) g_t^\eta(\mdqp) \cdot \indp{(\mdqr, \mdqp) \in \alt(\empmdp_t)} $$
whose normalizing constant is 
\begin{equation}
	\Lambda^\eta_t \ceq  \int f_t^\eta(\mdqr) g_t^\eta(\mdqp) \cdot \indp{(\mdqr,\mdqp) \in \alt(\empmdp_t)} \diff \mdqr  \diff \mdqp = \int \myexp{ - \eta \lambda_t(\mdqr, \mdqp) } \cdot \indp{(\mdqr,\mdqp) \in \alt(\empmdp_t)} \diff \mdqr \diff \mdqp \:,
\end{equation}
where we recall that $\alt(\empmdp_t) \C \Mdp$ and $\diff\mdqr \diff \mdqp$ denotes the standard (product) measure on this space.

\medskip
We denote by $\nu^\eta_{t}(\cdot)$ the distribution over $\Mdp$ with product density $f^\eta_t \cdot g^\eta_t$ and by $\nu_{t}^\eta(\cdot \mid \cX)$ its truncattion to the set $\cX \subseteq \Mdp$.  

\medskip
Given $\cH_{t-1}$, the probability of sampling an MDP $\mdq \ceq (\mdqr, \mdqp)$ from  $\nu^\eta_{t}(\cdot \mid \alt(\empmdp_t))$ is proportional to $$\myexp{ - \eta \lambda_t(\mdqr, \mdqp) } \cdot \indp{(\mdqr, \mdqp) \in \alt(\empmdp_t)}.$$ 
 
 \medskip
 
 \begin{lemma}
 \label{lem:ratio-normalizing-constants}
 Let $\eta>0$. If $\hat\pi_t = \hat \pi_{t+1}$ then, we have 
 		$$\log\frac{\Lambda_{t+1}^{\eta }}{\Lambda_t^{\eta }} \leq   \frac12 \log \bE_{(\mdqr, \mdqp)\sim \nu_{t}^\eta(\cdot \mid \alt(\empmdp_t))}\lsb \myexp{ - 2 \eta \cdot \ell^\alpha_{t}(\mdqr,\mdqp) } \rsb  +\frac12 \log\bE_{(\mdqr, \mdqp)\sim \nu_{t}^\eta(\cdot \mid \alt(\empmdp_t))} \lsb\myexp{ \eta \cdot \sigma_{t+1}(\mdqr) }  \rsb \:.$$
 \end{lemma}
 
 \begin{proof}
 First, we note that given and MDP $\mdp$, the set of MDPs $\mdp'$ such that $\mdp$ is not absolutely continuous w.r.t. $\mdp'$ has empty interior. Therefore $\alt(\empmdp_t) = \alt(\empmdp_{t+1})$ (up to a null set) whenever $\hat\pi_t = \hat \pi_{t+1}$. We have 

\begin{align*}
\log\frac{\Lambda_{t+1}^{\eta}}{\Lambda_t^{\eta}}
  &= \log\frac{
      \int f_{t+1}^\eta(\mdqr)\, g_{t+1}^\eta(\mdqp) 
      \cdot \indp{(\mdqr,\mdqp)\in\alt(\empmdp_t)}\,\diff\mdqr\,\diff\mdqp
    }{
      \int f_t^\eta(\mdqr)\, g_t^\eta(\mdqp) 
      \cdot \indp{(\mdqr,\mdqp)\in\alt(\empmdp_t)}\,\diff\mdqr\,\diff\mdqp
    } \\[4pt]
  &= \log \bE_{(\mdqr,\mdqp)\sim\nu_t^\eta(\cdot\mid\alt(\empmdp_t))}
     \left[\frac{f_{t+1}^\eta(\mdqr)}{f_t^\eta(\mdqr)}
           \cdot\frac{g_{t+1}^\eta(\mdqp)}{g_t^\eta(\mdqp)}\right] \\[4pt]
  &= \log \bE_{(\mdqr,\mdqp)\sim\nu_t^\eta(\cdot\mid\alt(\empmdp_t))}
     \left[\expb{
       -\frac{\eta}{2}\sum_h\left[
         \bigl(\muh_t(s_h^t,a_h^t,h)-\mdqr(s_h^t,a_h^t,h)\bigr)^2
         + 2\log\frac{1}{\mdqp(s_{h+1}^t\mid s_h^t,a_h^t,h)}
       \right]
       +\frac{\eta}{2}\,\sigma_{t+1}(\mdqr)
    }\right] \\[4pt]
  &\overset{(a)}{\leq}
     \frac{1}{2}\log\bE_{(\mdqr,\mdqp)\sim\nu_t^\eta(\cdot\mid\alt(\empmdp_t))}
     \left[\expb{
       -\eta\sum_h\left[
         \bigl(\muh_t(s_h^t,a_h^t,h)-\mdqr(s_h^t,a_h^t,h)\bigr)^2
         +2\log\frac{1}{\mdqp(s_{h+1}^t\mid s_h^t,a_h^t,h)}
       \right]
    }\right] \\
  &\quad
    +\frac{1}{2}\log\bE_{(\mdqr,\mdqp)\sim\nu_t^\eta(\cdot\mid\alt(\empmdp_t))}
     \!\lsb \myexp{\eta\cdot\sigma_{t+1}(\mdqr)}\rsb ,
\end{align*} 
where $(a)$ follows from the Cauchy--Schwarz inequality 
$\log\bE[XY]\leq\tfrac{1}{2}\log\bE[X^2]+\tfrac{1}{2}\log\mathbb{E}[Y^2]$.

\bigskip 

Intuitively, everything proceeds as if we were running continual exponential weights 
with learning rate $\eta$ and stochastic loss
\begin{equation}
\label{eq:def-loss}
l_{t}(\mdqr, \mdqp) \ceq \sum_{h=1}^H \lsb 
  \frac{\bigl(\muh_t(s_h^t,a_h^t,h) - \mdqr(s_h^t,a_h^t,h)\bigr)^2}{2}  
  + \log \frac{1}{\mdqp(s_{h+1}^t \mid s_h^t, a_h^t, h)}\rsb. 
\end{equation}

\medskip 
However, to control the magnitude of the stochastic loss, which might be unbounded due to low transition probabilities, we work with the clipped loss
\begin{equation}
\label{eq:def-trunc-loss}
\ell^{\alpha}_{t}(\mdqr, \mdqp) \ceq \sum_{h=1}^H \lsb 
  \frac{\bigl(\muh_t(s_h^t,a_h^t,h) - \mdqr(s_h^t,a_h^t,h)\bigr)^2}{2} 
  + \log \frac{1}{\maxp{\mdqp(s_{h+1}^t \mid s_h^t,a_h^t,h)}{
    \myexp{-t^\alpha}}}\rsb,
\end{equation}
which satisfies $\ell^\alpha_{t}(\mdqr,\mdqp) \leq l_{t}(\mdqr,\mdqp)$ 
by definition. Combining with the calculations above, we have with 
probability $1$,
\begin{align*}
\log\frac{\Lambda_{t+1}^{\eta}}{\Lambda_t^{\eta}}
  &\leq \frac{1}{2}\log\bE_{(\mdqr,\mdqp)\sim\nu_t^\eta(\cdot\mid\alt(\empmdp_t))}
    \!\lsb \myexp{-2\eta\,\ell^\alpha_{t}(\mdqr,\mdqp)}\rsb +  \frac{1}{2}\log\bE_{(\mdqr,\mdqp)\sim\nu_t^\eta(\cdot\mid\alt(\empmdp_t))}
    \!\lsb \myexp{\eta\,\sigma_{t+1}(\mdqr)} \rsb.
\end{align*}
\end{proof}

\smallskip\noindent
We further define
\begin{equation}
\label{eq:def-phi}
\varphi_{t+1}^\eta \ceq 
  \frac{1}{2}\log\bE_{(\mdqr,\mdqp)\sim\nu_t^\eta(\cdot\mid\alt(\empmdp_t))}
  \!\lsb \myexp{\eta\,\sigma_{t+1}(\mdqr)}\rsb ,
\end{equation}
which will be later upper bounded in Lemma~\ref{lem:phi-bound}.

\subsection{Hedge with Decreasing Learning Rate} 
We use a decreasing learning rate schedule $(\eta_t)_{t\geq 1}$ rather than a 
constant one, since a constant rate would require the doubling trick and 
discarding of past data. 
Throughout, we assume $(\eta_t)_{t\geq 1}$ is non-increasing and satisfies $\sum_{t=1}^T\eta_t = o(T)$. For generality, we assume $\eta_t = t^{-\varsigma}$. 

\begin{lemma}
\label{lem:hedge-regret}
Let $T_0 \leq T$ and assume the event $\cE_{\sect}^{T_0}(T) \ceq \{\forall\, T_0 \leq t \leq T,\ \hat \pi_t = \hat \pi_{T_0})\}$ holds. Define
\[
  \psi_{t+1}^\alpha \ceq \frac{\eta_t C_{t,\alpha}^2}{2} 
  + \eta_t^{-1}\varphi_{t+1}^{\eta_t},
  \qquad
  \varPsi_T^\alpha(T_0) \ceq \sum_{t=T_0}^{T-1} \psi_{t+1}^\alpha,
\]
where $C_{t,\alpha} \ceq H\!\left(\tfrac{1}{2}+t^\alpha\right)$ and 
$\varphi_{t+1}^{\eta_t}$ is defined in \eqref{eq:def-phi}. Then
\begin{align*}
\sum_{t=T_0}^{T-1}
  \bE_{(\mdqr,\mdqp)\sim\nu_t^{\eta_t}(\cdot\mid\alt(\empmdp_t))}
  \!\bigl[\ell_t^\alpha(\mdqr,\mdqp)\bigr]
  \;\leq\;
  &-\frac{1}{\eta_T}\log\Lambda_T^{\eta_T}
  + \varPsi_T^\alpha(T_0)
  + \frac{1}{\eta_{T_0}}\log\Lambda_{T_0}^{\eta_{T_0}}\; +  \\
  & \sum_{t=T_0}^{T-1}
    \left[
      \frac{1}{\eta_{t}}\log\Lambda_{t+1}^{\eta_{t}}
      - \frac{1}{\eta_{t+1}}\log\Lambda_{t+1}^{\eta_{t+1}}
    \right].
\end{align*}
\end{lemma}

\begin{proof}

Thanks to Lemma~\ref{lem:ratio-normalizing-constants} we have 
$$ \log\frac{\Lambda_{t+1}^{\eta}}{\Lambda_t^{\eta}} \leq \frac{1}{2}\log\bE_{(\mdqr,\mdqp)\sim\nu_t^\eta(\cdot\mid\alt(\empmdp_t))}
    \!\lsb \myexp{-2\eta\,\ell^\alpha_{t}(\mdqr,\mdqp)}\rsb  + \varphi_{t+1}^\eta$$
    where $\varphi_{t+1}^\eta$ is defined in \myeqref{eq:def-phi}. 
 \medskip
 Next, we note that $\ell_t^\alpha \in (0, (\frac12 + t^\alpha) \cdot H)$, and we define $C_{t, \alpha} \ceq H \, (\frac12 + t^\alpha)$. By Hoeffding's lemma we have for any $\eta$, 
 \begin{align}
 \label{eq:hoeff-lemma-applied}
 	\bE_{(\mdqr,\mdqp)\sim\nu_t^\eta(\cdot\mid\alt(\empmdp_t))}
    \!\lsb \myexp{-2\eta\,\ell^\alpha_{t}(\mdqr,\mdqp)}\rsb \leq \expb{-2\eta \bE_{(\mdqr,\mdqp)\sim\nu_t^\eta(\cdot\mid\alt(\empmdp_t))}
    \!\lsb \ell^\alpha_{t}(\mdqr,\mdqp)\rsb + \frac{\eta^2 C_{t, \alpha}^2}{2} }.   
 \end{align}   
    
Therefore, 
\begin{equation}
\label{eq:hedge-one-step}
\frac{1}{\eta_{t}}\log\frac{\Lambda_{t+1}^{\eta_{t}}}{\Lambda_t^{\eta_{t}}}
  \leq
  -\bE_{(\mdqr,\mdqp)\sim\nu_t^{\eta_{t}}(\cdot\mid\alt(\empmdp_t))}
   \!\bigl[\ell_{t}^\alpha(\mdqr,\mdqp)\bigr]
  + \underbrace{\frac12 \eta_{t}C_{t, \alpha}^2
    + \eta_{t}^{-1}\varphi^{\eta_{t}}_{t+1}}_{\displaystyle\psi^\alpha_{t+1}}. 
\end{equation}
Assuming $\cE_{\sect}^{T_0}(T)$ and  summing \eqref{eq:hedge-one-step} from $t = T_0$ to $T-1$ gives
\begin{equation}
\label{eq:sum-lhs}
\sum_{t=T_0}^{T-1}\frac{1}{\eta_t}\log\frac{\Lambda_{t+1}^{\eta_t}}{\Lambda_{t}^{\eta_t}}
  \leq
  -\sum_{t=T_0}^{T-1}
   \bE_{(\mdqr,\mdqp)\sim\nu_t^{\eta_t}(\cdot\mid\alt(\empmdp_{t}))}
   \!\bigl[\ell_t^\alpha(\mdqr,\mdqp)\bigr]
  + \underbrace{\sum_{t=T_0}^{T-1}\psi^\alpha_{t+1}}_{\varPsi^\alpha_T(T_0)}.
\end{equation}
The left-hand side is not directly telescoping. We decompose it as

\begin{align*}
\sum_{t=T_0}^{T-1}\frac{1}{\eta_t}\log\frac{\Lambda_{t+1}^{\eta_t}}{\Lambda_{t}^{\eta_t}}
  &= \sum_{t=T_0}^{T-1}
    \left[
      \frac{1}{\eta_t}\log\Lambda_{t+1}^{\eta_t}
      - \frac{1}{\eta_t}\log\Lambda_{t}^{\eta_t}
    \right] \\
  &= \sum_{t=T_0}^{T-1}
    \left[
      \frac{1}{\eta_{t+1}}\log\Lambda_{t+1}^{\eta_{t+1}}
      - \frac{1}{\eta_{t}}\log\Lambda_{t}^{\eta_{t}}
    \right]
    + \sum_{t=T_0}^{T-1}
    \left[
      \frac{1}{\eta_{t}}\log\Lambda_{t+1}^{\eta_{t}}
      - \frac{1}{\eta_{t+1}}\log\Lambda_{t+1}^{\eta_{t+1}}
    \right] \\
  &= \frac{1}{\eta_T}\log\Lambda_T^{\eta_T}
    - \frac{1}{\eta_{T_0}}\log\Lambda_{T_0}^{\eta_{T_0}}
    + \sum_{t=T_0}^{T-1}
    \left[
      \frac{1}{\eta_{t}}\log\Lambda_{t+1}^{\eta_{t}}
      - \frac{1}{\eta_{t+1}}\log\Lambda_{t+1}^{\eta_{t+1}}
    \right],
\end{align*}
where the last step uses that the first sum telescopes. Substituting back into \myeqref{eq:sum-lhs} and rearranging gives the claimed inequality.
\end{proof}

\subsection{Best Model in Hindsight}
The idea here is to show a lower bound on this quantity when there is enough mass around the best model. While such bounds are standard to derive for Hedge on convex sets (with convex loss), in our case, although the loss is convex, the set $\alt(\empmdp_t)$ is non-convex and a priori not a simple union of convex sets. 

\begin{lemma}
\label{lem:utl_lower_bound} Let $\gamma \ceq \frac{1}{6TSAH(H+1)}$. We have 
\begin{align*}
	 \frac{1}{\eta_{T}} \log\Lambda_{T}^{\eta_{T}}  \geq -\inf_{(\mdqr, \mdqp) \in \alt(\empmdp_T)}\sum_{t=1}^{T-1}\sum_h \lsb \frac{(\muh_{T}(s_h^t,a_h^t, h ) - \mdqr(s_h^t,a_h^t, h))^2}{2} + \log \frac{1}{\mdqp(s_{h+1}^t, \mid s_h^t, a_h^t,  h)}\rsb  \\  \frac{2(SAH+S^2AH) \log \gamma}{\eta} -  \gamma \frac {TH}2 \log (e S^2T) + \frac{\log  \vol(\alt(\empmdp_T) \cap \Mdp^{1/S\cdot \sqrt{T}})}{\eta}  -  O(1)  \:.
\end{align*}
\end{lemma}

\begin{proof}
	The proof of this result relies on a lower bound on the inflated posterior over  $\alt(\empmdp_T)$. Simplifying notation, let $\eta \ceq \eta_T$ and we have  
	\begin{eqnarray}
	\label{eq:appx-ed-ffg-k}
		\Lambda^\eta_T = \int \myexp{ - \eta \lambda_T(\mdqr, \mdqp) } \indp{\mdqr,\mdqp) \in \alt(\empmdp_T)} \diff\mdqr \diff\,\mdqp \:,
	\end{eqnarray}
	where $\lambda_T^\eta$ is a convex function as defined earlier. We prove that the posterior puts the majority of its mass around the mode. First, we need to prove that there is enough volume around the best model
	\begin{equation}
	\label{eq:appx-best-model}
		(\mdqr_T^\star, \mdqp_T^\star) \in \argmin_{(\mdqr, \mdqp) \in \closure({\alt(\empmdp_T)})}  \lambda_T(\mdqr, \mdqp),
	\end{equation}
	so that the integral in \myeqref{eq:appx-ed-ffg-k} will concentrate around the mode. We define $\mdq_T^\star = (\mdqr_T^\star, \mdqp_T^\star)$. 
	\medskip
	\noindent
	
	First, note that this argmin is well defined as $\closure({\alt(\empmdp_{T})})$ is a compact set. Because the instance in \eqref{eq:appx-best-model} might be close to the boundary of ${\alt(\empmdp_{T})}$, we build another instance around $\mdq_T^\star$ which is still contained in the alternative set ${\alt(\empmdp_{T})}$, and with a non-negligible enclosing volume around it. 
	
	\medskip 
	\noindent 
	Because $\mdq_T^\star$ belongs to the closure of the alternative set, there exists a deterministic policy $\tilde \pi \in \Pi_\mathrm{det} \backslash \{ \hat\pi_t \}$ such that $\tilde \pi$ is globally optimal in $\mdq_T^\star$. 
	
Next, we invoke Lemma~\ref{lem:build_easier} which shows that there exists some MDP 
 $\mdq' \ceq (\mdqr', \mdqp')$ such that $\tilde \pi$ is the unique global optimal and start-state policy, and its transition and reward kernels satisfy 
 \begin{equation}
 \label{eq:appx-lkolo-om}
 	\left\{ \begin{array}{ll}
 		\Delta(\mdq') &\geq \frac1T  \\ 
 		\mdqp'(\cdot \mid s,a, h) &\ceq  (1-\gamma_{s,a,h}) \mdqp_T^\star(\cdot \mid s,a, h) + \gamma_{s,a,h} \mdqp^\star_T (\cdot \mid s,\pi_h(s), h) \\
		\lvert \mdqr^\star_T (s,a, h) - \mdqr'(s,a, h)\rvert  & \leq d_{s,a, h} \ceq 2^{H-h+1} \veps 
 	\end{array}\right., 
 \end{equation}
 
 where $\gamma_{s,a,h} \in \left[  0, \frac{(1 + 2^{H-h}) \veps}{1 -  \lp 2^{H-h+1}  + 3 \rp	\veps} \right]$. 
 
In particular, $\mdq' \in \alt(\empmdp_{T})$. Next, we relate $\Lambda^\eta_T $ to the (near optimal) $\lambda^\eta_T(\mdqr', \mdqp') $. Thanks to Lemma~\ref{lem:convex_vol}, for $\gamma \ceq \frac{1}{6TSAH(H+1)}$, and calling $\mdq_T \ceq (\mdqr_T, \mdqp_T)$ the instance built by this lemma. We then have 

\begin{equation} 
	(1-\gamma)\mdq' + \gamma \alt(\empmdp_T) \subset  \alt(\empmdp_T).
\end{equation} 
Therefore, 
	\begin{align}
	\label{eq:mylasteq}
	\begin{split}
	\log\Lambda^\eta_{T} &=  \log \int \myexp{ - \eta \lambda_{T}(\mdqr, \mdqp ) }  \indp{(\mdqr, \mdqp ) \in \alt(\empmdp_T)} \diff \mdqr \diff \mdqp  \\
	&\geq  \log \int_{(1-\gamma) \mdq' + \gamma \alt(\empmdp_T)}  \myexp{ - \eta \lambda_{T}(\mdqr, \mdqp) } \diff \mdqr \diff \mdqp  \\
	&\geq \log \int_{  \alt(\empmdp_T) } \gamma^{S^2AH + SAH}\exp\lb - \eta  \lambda_{T}((1-\gamma) \cdot \bigl(\mdqr', \mdqp'\bigr) + \gamma\cdot \bigl(\mdqr, \mdqp ) \bigr) \rb \diff \mdqr \diff \mdqp  \\
	&\geq \log \lp \gamma^{S^2AH +SAH} \myexp{- \eta  (1-\gamma)\cdot \lambda_{T}(\mdqr', \mdqp') } \rp + \log \int_{  \alt(\empmdp_T)} \myexp{ - \eta  \gamma \cdot \lambda_{T}(\mdqr, \mdqp) }  \diff \mdqr \diff \mdqp \\
	&=  -\eta (1-\gamma)\lambda_{T}(\mdqr', \mdqp') +  (S^2AH + SAH) \log \gamma  + \log \int_{ \alt(\empmdp_T)} \myexp{ - \eta \gamma \lambda_{T}(\mdqr, \mdqp ) } \diff \mdqr  \diff \mdqp 	\end{split} 
\end{align}
which follows by convexity of $\lambda_{T}$. Combining the results above, in particular in \myeqref{eq:appx-lkolo-om}, 
	\begin{eqnarray*}
		\lambda_T(\mdqr', \mdqp') &\leq&    \sum_{t=1}^{T-1}\sum_h \lsb \frac{(\muh_{T}(s_h^t,a_h^t, h ) - \mdqr'(s_h^t,a_h^t, h))^2}{2} + \log \frac{1}{(1-\gamma_h)\mdqp_T^\star(s_{h+1}^t\mid s_h^t, a_h^t,  h)}\rsb  \\
		&\leq&   \sum_{t=1}^{T-1} \sum_h \lsb \frac{(\muh_{T}(s_h^t,a_h^t, h) - \mdqr_T^\star(s_h^t,a_h^t, h))^2}{2} + \log \frac{1}{(1-\gamma_h)\mdqp_T^\star(s_{h+1}^t \mid s_h^t, a_h^t, h)}\rsb   \\
		&& + TH  \lsb  4^H \veps^2  + 2^H \veps \rsb  \\
		&\leq&  \lambda_T(\mdqr_T^\star, \mdqp_T^\star) +  TH  \lsb  4^H \veps^2  + 2^H \veps + \log \frac{1}{1-\gamma_1}\rsb \quad \text{(monotonicity of $h \mapsto \gamma_h $)}\:. 
	\end{eqnarray*}

Therefore  \begin{equation}
	\label{eq:lo-pm-fd}
		\lambda_T(\mdqr', \mdqp') \leq \lambda_T(\mdqr_T^\star, \mdqp_T^\star) + TH  \lsb  4^H \veps^2  + 2^H \veps + \log \frac{1}{1-\gamma_1}\rsb
	\end{equation}

 It remains to control the residual integral $$ \log \int_{ \alt(\mdp)} \myexp{ - \eta \gamma \lambda_{T}(\mdqr, \mdqp) }  \diff \mdqr  \diff \mdqp.$$
We note that if this is too small the bound above will be meaningless. Also, observe that the integraand is unbounded, (due to the transition probabilities) hence to control this quantity we introduce $\Mdp^{1/S\cdot \sqrt{T}}$ to be the set ofs MDPs for which each transition is at least $1/S\cdot \sqrt{T}$. We have  
\begin{eqnarray*}
	 \int_{ \alt(\mdp)} \myexp{ - \eta \gamma  \lambda_{T}(\mdqr, \mdqp) } \diff \mdqr  \diff \mdqp  &\geq&  \int_{\alt(\mdp)\cap \Mdp^{1/\sqrt{T}} } \myexp{ - \eta\gamma \lambda_{T}(\mdqr, \mdqp) } \diff \mdqr  \diff \mdqp \\
	 &\geq& \int_{ \alt(\empmdp_T) \cap \Mdp^{1/S\cdot \sqrt{T}} } \exp\lb  - \eta \gamma \sum_{s,a,h}n_t(s,a, h) \bigl(\frac12  + \frac12 \log (S^2T) \bigr)\rb  \diff \mdqr  \diff \mdqp \\
	 &=&  \myexp{ - \eta \gamma \frac {TH}2 \log (e S^2T) } \, \vol(\alt(\empmdp_T) \cap \Mdp^{1/S\cdot \sqrt{T}})
\end{eqnarray*}

Combining the above displays yields, 
\begin{equation}
	\log \int_{ \alt(\mdp) } \myexp{- \eta\gamma \lambda_{T}(\mdqr, \mdqp)} \diff \mdqr \diff \mdqp \geq - \eta \gamma \frac {TH}2 \log (e S^2T) + \log  \vol(\alt(\empmdp_T)\cap \Mdp^{1/S\cdot \sqrt{T}}).  
\end{equation}

Plugging-in this back into \myeqref{eq:mylasteq} yields  
$$ \frac{1}{\eta}\log\Lambda_{T} \geq  - (1-\gamma) \lambda_{T}(\mdqr', \mdqp') +  \frac{2(SAH+S^2AH) \log \gamma }{\eta} - \gamma \frac{\log (e S^2T)}{2}  + \frac{\log  \vol(\alt(\empmdp_T) \cap \Mdp^{1/S\cdot \sqrt{T}})}{\eta} $$
and combining with \myeqref{eq:lo-pm-fd} yields  
\begin{align}
\begin{split}
	&\frac{1}{\eta}\log\Lambda^\eta_{T} \geq \\ &  - \lambda _T(\mdqr_T^\star, \mdqp_T^\star) +  \frac{2(SAH+S^2AH) \log \gamma}{\eta}  -\gamma \frac {TH}2 \log (e S^2T) + \frac{\log  \vol(\alt(\empmdp_T) \cap \Mdp^{1/S \cdot \sqrt{T}})}{\eta}  -  \\ &TH  \lsb  4^H \veps^2  + 2^H \veps + \log \frac{1}{1-\gamma_1}\rsb. 
	\end{split}
\end{align}

Plugging $\veps = 1/T$  in the above, using the the inequality $\log(1+x) \leq x$ and recalling that $\gamma_1 \in \left[  0, \frac{(1 + 2^{H-1})}{T -  \lp 2^{H}  + 3 \rp	} \right] $ yields 
$$ H  \lsb  \frac{4^H}{T}   + 2^H  + T\frac{\gamma_1}{1-\gamma_1}\rsb = O(1).$$

Rewriting the terms above, we have 
$$ \frac{1}{\eta}\log\Lambda^\eta_{T} \geq  - \lambda_T(\mdqr_T^\star, \mdqp_T^\star) -  \frac{2(SAH+S^2AH) \log \gamma}{\eta} -  \gamma \frac {TH}2 \log (e S^2T) + \frac{\log  \vol(\alt(\empmdp_T) \cap \Mdp^{1/S\cdot \sqrt{T}})}{\eta}  -  O(1) $$
which achieves to prove the claimed result. 
\end{proof}

\subsection{Technical and Concentration Lemmas}

We show a series of lemmas to complete the analysis of the conditional posterior resampling. 

\begin{lemma}
\label{lem:ZT-bound}
We have 
$$ \sum_{t=T_0}^{T-1}\left[
  \frac{1}{\eta_t}\log\Lambda_{t+1}^{\eta_t}
  - \frac{1}{\eta_{t+1}}\log\Lambda_{t+1}^{\eta_{t+1}}
\right] \leq \sum_{t=T_0}^{T-1}\left[\frac{\log\vol(\alt(\empmdp_{t+1}))}{\eta_t} - \frac{\log\vol(\alt(\empmdp_{t+1}))}{\eta_{t+1}}\right]. $$ 
\end{lemma}

\begin{proof}
 Let 
\[
\rho_t(\mdqr,\mdqp) \ceq 
  \frac{1}{\vol(\alt(\empmdp_t))}
  \indp{(\mdqr,\mdqp)\in\alt(\empmdp_t)}
\]
be the uniform measure on $\alt(\empmdp_t)$. By definition of $\Lambda_t^\eta$, \[
\frac{1}{\eta}\log\Lambda_t^\eta 
  = \frac{1}{\eta}\log\!\int \rho_t(\mdqr,\mdqp)
    \myexp{-\eta\,\lambda_t(\mdqr,\mdqp)}\diff \mdqr\,\diff \mdqp
  + \frac{\log\vol(\alt(\empmdp_t))}{\eta}.
\]
Introducing $f_t(\eta) \ceq \frac{1}{\eta}\log\!\int\rho_t(\mdqr,\mdqp)
\myexp{-\eta\,\lambda_t(\mdqr,\mdqp)}\,\diff \mdqr\,\diff \mdqp$, we decompose
\begin{align*}
&\sum_{t=T_0}^{T-1}\left[
  \frac{1}{\eta_t}\log\Lambda_{t+1}^{\eta_t}
  - \frac{1}{\eta_{t+1}}\log\Lambda_{t+1}^{\eta_{t+1}}
\right] =  \\ & \sum_{t=T_0}^{T-1}\left[f_{t+1}(\eta_t) - f_{t+1}(\eta_{t+1})\right]
  + 
    \sum_{t=T_0}^{T-1}\left[\frac{\log\vol(\alt(\empmdp_{t+1}))}{\eta_t} - \frac{\log\vol(\alt(\empmdp_{t+1}))}{\eta_{t+1}}\right].
\end{align*}
By Lemma~\ref{lem:pos-derv}, $f_{t+1}$ is non-decreasing. Since 
$(\eta_t)_t$ is non-increasing, $\eta_t \geq \eta_{t+1}$ for all $t$, 
so $f_{t+1}(\eta_t) - f_{t+1}(\eta_{t+1}) \geq 0$, 
which concludes the proof.
\end{proof} 

\begin{lemma}
\label{lem:phi-bound}
Define  $C_{t,\alpha} \ceq H\!\left(\tfrac{1}{2}+t^\alpha\right)$ and assume $\eta_t \propto t^{-\varsigma}$, with $\varsigma > 2\alpha$. 
There exists deterministic $f_\varPsi$ such that the event 
$$ \cE_{\sect}^{\varPsi}(T) \ceq \lb \varPsi_T^\alpha(T_0) \ceq \sum_{t=T_0}^{T-1} \lsb \frac{\eta_t C_{t,\alpha}^2}{2} 
+ \eta_t^{-1}\varphi_{t+1}^{\eta_t} \rsb \leq f_\varPsi(T) \rb $$
holds with probability $1- O(1/T^2)$. Moreover, $f_\varPhi$ is given in the proof and satisfies $f_\varPsi = o(T)$
\end{lemma}

\begin{proof}
We bound each term in $\varPsi_T^\alpha(T_0)$ separately.
For the first term, we observe that 
\[
\sum_{t=T_0}^{T-1} \eta_t C_{t,\alpha}^2 
  \leq O\! \lp   \sum_{t=1}^{T-1} t^{2\alpha} t^{-\varsigma} \rp = o(T),
\]
when $\varsigma > 2\alpha $. 

\medskip\noindent
To control the second term in $\varPsi$, recalling the definition of $\varphi_{t+1}^{\eta_t}$ from \eqref{eq:def-phi} and $\sigma_{t+1}(\mdqr)$ from \eqref{eq:def-sigma}, we have  
\begin{align*}
\sigma_{t+1}(\mdqr) 
  &\leq \sum_h 2\,n_{t+1}(s_h^t,a_h^t,h)
    \bigl(\muh_t(s_h^t,a_h^t,h)-\mdqr(s_h^t,a_h^t,h)\bigr)
    \bigl(\muh_t(s_h^t,a_h^t,h)-\muh_{t+1}(s_h^t,a_h^t,h)\bigr) \\
  &= \underbrace{\sum_h 2\,n_{t+1}(s_h^t,a_h^t,h)\,
      \muh_t(s_h^t,a_h^t,h)
      \bigl(\muh_t(s_h^t,a_h^t,h)-\muh_{t+1}(s_h^t,a_h^t,h)\bigr)
    }_{\mathscr{J}_{t+1}} \\
  &\quad
   -\underbrace{\sum_h 2\,n_{t+1}(s_h^t,a_h^t,h)\,
      \mdqr(s_h^t,a_h^t,h)
      \bigl(\muh_t(s_h^t,a_h^t,h)-\muh_{t+1}(s_h^t,a_h^t,h)\bigr)
    }_{\kappa_{t+1}(\mdqr)}.
\end{align*}
The term $\mathscr{J}_{t+1}$ is independent of $(\mdqr,\mdqp)$ and is 
controlled by Azuma--Hoeffding (see Lemma~\ref{lem:J-bound} below).
For $\kappa_{t+1}(\mdqr)$, applying Hoeffding's lemma with the bound 
$|\kappa_{t+1}(\mdqr)| \leq \sqrt{\phi_{t+1}}$, where
\[
\phi_{t+1} \ceq \left(
  \sum_h 2\,n_{t+1}(s_h^t,a_h^t,h)
  \bigl|\muh_t(s_h^t,a_h^t,h)-\muh_{t+1}(s_h^t,a_h^t,h)\bigr|
\right)^{\!2},
\]
gives
\[
\eta_t^{-1}\log\bE_{(\mdqr,\mdqp)\sim\nu_t^{\eta_t}
  (\cdot\mid\alt(\empmdp_t))}
  \!\left[e^{-\eta_t\kappa_{t+1}(\mdqr)}\right]
\leq
-\bE_{(\mdqr,\mdqp)\sim\nu_t^{\eta_t}
  (\cdot\mid\alt(\empmdp_t))}
  \!\left[\kappa_{t+1}(\mdqr)\right]
+ \frac12 \eta_t\phi_{t+1}.
\]
Summing over $t$ and applying concentration to both $\mathscr{J}_{t+1}$ 
and $\phi_{t+1}$ (via Lemmas~\ref{lem:J-bound}, ~\ref{lem:phi-conc}, and ~\ref{lem:D-bound}) 
shows that $\sum_{t=K}^{T-1}\eta_t^{-1}\varphi_{t+1}^{\eta_t} = o(T)$ 
on an event with probability larger than $1-O(1/T^2)$. The conclusion then follows. 
\end{proof}

 \begin{lemma}
\label{lem:J-bound}

There exists a deterministic function $f_{\mathscr{J}}$ such that the event 
\[
\cE_{\sect}^{\mathscr{J}}(T) \lb \mathscr{J}_T \ceq \sum_{t=1}^{T}\sum_{h=1}^H 
  2\,n_{t+1}(s_h^t,a_h^t,h)\,
  \muh_t(s_h^t,a_h^t,h)
  \bigl(\muh_t(s_h^t,a_h^t,h)-\muh_{t+1}(s_h^t,a_h^t,h)\bigr)
  \leq f_{\mathscr{J}}(T) \rb, 
\]
\end{lemma}

holds with probability at least $1- O(1/T^2)$. Moreover $f_{\mathscr{J}}$ is given in the proof and satisfies $f_{\mathscr{J}} = o(T)$. 

\begin{proof}
Define $S_t(s,a,h) \ceq \sum_{k \in \lbrack t-1\rbrack} \indp{s_h^k=s,\,a_h^k=a}\,R_h^k$,
where $R_h^k \mid s_h^k,a_h^k \sim \mu(s_h^k,a_h^k,h)+\zeta_h^k$ and
$\zeta_h^k\sim\mathcal{N}(0,1)$ is independent noise. A direct computation gives
\begin{align*}
n_{t+1}(s_h^t,a_h^t,h)
\bigl(\muh_t(s_h^t,a_h^t,h)-\muh_{t+1}(s_h^t,a_h^t,h)\bigr)
&= S_t(s_h^t,a_h^t,h) - S_{t+1}(s_h^t,a_h^t,h) + \muh_t(s_h^t,a_h^t,h) \\
&= \muh_t(s_h^t,a_h^t,h) - \mu(s_h^t,a_h^t,h) + \zeta_h^{t},
\end{align*}
so that
\begin{equation}
\label{eq:J-decomp}
\mathscr{J}_T = 
  \underbrace{2\sum_{t=1}^T\sum_{h=1}^H 
    \muh_t(s_h^t,a_h^t,h)
    \bigl(\muh_t(s_h^t,a_h^t,h)-\mu(s_h^t,a_h^t,h)\bigr)
  }_{(I)}
  +\underbrace{2\sum_{t=1}^T\sum_{h=1}^H
    \muh_t(s_h^t,a_h^t,h)\,\zeta_h^{t}
  }_{(II)}.
\end{equation}

\medskip\noindent
\textbf{Term (II).} Since $\muh_t$ is 
$\mathcal{H}_{t-1}$-measurable, $(s_h^t,a_h^t,h)$ is independent of $\zeta_h^{t}$ which is itself independent of $\mathcal{H}_{t-1}$, and centered. Therefore, the partial sums form a martingale 
with subgaussian increments. By Azuma--Hoeffding's inequality, the event
\[
\left\{
  (II) \leq 2\sqrt{2\log(T^2)
  \sum_{t=1}^T\sum_{h=1}^H\muh_t(s_h^t,a_h^t,h)^2}
\right\}
\]
holds with probability at least $1-1/T^2$.

\medskip\noindent
\textbf{Term (I).} By Cauchy--Schwarz,
\begin{align*}
(I) 
  &\leq 2\sum_{h=1}^H
    \sqrt{\sum_{t=1}^T 
      \frac{\muh_t(s_h^t,a_h^t,h)^2}{\maxp{1}{n_t(s_h^t,a_h^t,h)}}
    }
    \cdot
    \sqrt{\sum_{t=1}^T
      \maxp{1}{n_t(s_h^t,a_h^t,h)}
      \bigl(\muh_t(s_h^t,a_h^t,h)-\mu(s_h^t,a_h^t,h)\bigr)^2
    } \\
  &\leq 2\sum_{h=1}^H
    \sqrt{\sum_{t=1}^T 
      \frac{\muh_t(s_h^t,a_h^t,h)^2}{\maxp{1}{n_t(s_h^t,a_h^t,h)}}
    }
    \cdot\sqrt{T\,\beta^\mu(T,1/T^2)},
\end{align*}
where the second inequality uses the self-normalized concentration event
$$ \cE_{\sect}^{\mu}(T) \ceq \lb \forall t \leq T, \: \sum_{s,a, h} n_t(s,a, h) \KL(\hat R_h^t(s,a) \;\|\; R_h(s,a)) \leq  \beta^{\mu}(T, 1/T^2) \rb, $$
which holds with probability $1-1/T^2$. Since rewards are bounded in $(0,1)$, the elliptic potential yields 
\[
\sum_{t=1}^T 
  \frac{\muh_t(s_h^t,a_h^t,h)^2}{\maxp{1} {n_t(s_h^t,a_h^t,h)}}
  \leq 4\log(T+1),
\]
giving $(I) \leq 4H\sqrt{2T\,\beta^\mu(T,1/T^2)\log(T+1)}$.

\medskip\noindent Combining the bounds on (I) and (II), and using that
$\sum_{t=1}^T\sum_h\muh_t(s_h^t,a_h^t,h)^2 \leq HT$
together with $\beta^\mu(T,1/T^2)= O(\log T)$, we obtain
$\mathscr{J}_T \leq O\!\left(\sqrt{T\log T}\right) = o(T)$
on an event holding with probability at least $1-1/T^2$.
\end{proof}

 \begin{lemma}
\label{lem:phi-conc}
There exists a deterministic sequence $\phi(T) = o(T)$ such that the event
\[
\mathcal{E}_{\sect}^{\phi}(T) \ceq \left\{
\sum_{t=1}^{T}\sum_{h=1}^H \eta_t
  \Bigl[2\,n_{t+1}(s_h^t,a_h^t,h)
  \bigl|\muh_t(s_h^t,a_h^t,h)-\muh_{t+1}(s_h^t,a_h^t,h)\bigr|
  \Bigr]^2 \leq \phi(T)
\right\}
\] holds with probability at least $1-O(1/T^2)$, where $\phi(T)$ is given explicitly in the proof.
\end{lemma} 

\begin{proof}
Using the identity established in the proof of Lemma~\ref{lem:J-bound},
\[
n_{t+1}(s_h^t,a_h^t,h)
\bigl(\muh_t(s_h^t,a_h^t,h)-\muh_{t+1}(s_h^t,a_h^t,h)\bigr)
= \muh_t(s_h^t,a_h^t,h) - \mu(s_h^t,a_h^t,h) + \zeta_h^{t},
\]
and the inequality $(x+y)^2 \leq 2x^2+2y^2$, we obtain
\begin{equation}
\label{eq:phi-sq-bound}
\Bigl[2\,n_{t+1}(s_h^t,a_h^t,h)
  \bigl|\muh_t(s_h^t,a_h^t,h)-\muh_{t+1}(s_h^t,a_h^t,h)\bigr|
\Bigr]^2
\leq
8\bigl(\muh_t(s_h^t,a_h^t,h)-\mu(s_h^t,a_h^t,h)\bigr)^2
+ 8\bigl(\zeta_h^{t}\bigr)^2.
\end{equation}

\medskip\noindent

Using the self-normalized concentration event
$$ \cE_{\sect}^{\mu}(T) \ceq \lb \forall t \leq T, \: \sum_{s,a, h} n_t(s,a, h) \KL(\hat R_h^t(s,a) \;\|\; R_h(s,a)) \leq  \beta^{\mu}(T, 1/T^2) \rb, $$
it follows 
\begin{align*}
\sum_{t=1}^T\bigl(\muh_t(s_h^t,a_h^t,h)-\mu(s_h^t,a_h^t,h)\bigr)^2
  &= \sum_{t=1}^T
    \frac{\maxp{1}{n_t(s_h^t,a_h^t,h)}
      \bigl(\muh_t(s_h^t, a_h^t, h)-\mu(s_h^t, a_h^t, h)\bigr)^2}
    {\max\{1,n_t(s_h^t,a_h^t,h)\}} \\
  &\leq \sum_{t=1}^T
    \frac{2\beta^\mu(T,1/T^2)}{\maxp{1}{n_t(s_h^t,a_h^t,h)}},
\end{align*}
which thus holds with probability at least $1-1/T^2$.

For the $\chi^2$ terms, by a union bound over $SAHT$ events, the following 
\[
\mathcal{E}_{\sect}^\zeta(T) \ceq \left\{\forall\,t\leq T,\;\forall\,(s,a,h), \;
  |\zeta_h^{t}| \leq \sqrt{2\log(2SAHT^3)}\right\}
\]
holds with probability at least $1-1/T^2$.

\medskip\noindent
On $\mathcal{E}_{\sect}^\mu(T)\cap\mathcal{E}_{\sect}^\zeta(T)$, 
which holds with probability at least $1-2/T^2$, substituting both bounds 
into \eqref{eq:phi-sq-bound}, summing over $t$ and $h$, and applying the  
elliptic potential bound
$\sum_{t=1}^T \frac{1}{\maxp{1}{n_t(s_h^t,a_h^t,h)}} \leq 4\log(T+1)$
gives
\begin{align*}
\sum_{t=1}^{T}\sum_{h=1}^H \eta_t
  \Bigl[2\,n_{t+1}&(s_h^t,a_h^t,h)
  \bigl|\muh_t-\muh_{t+1}\bigr|((s_h^t,a_h^t,h))\Bigr]^2 \\
  &\leq 16H\log(2SAHT^3)\sum_{t=1}^T\eta_t
  + 64H\,\beta^\mu(T,1/T^2)\log(T+1),
\end{align*}
which is $o(T)$ since $\sum_{t=1}^T\eta_t=o(T)$ and 
$\beta^\mu(T,1/T^2)=O(\log T)$.
\end{proof}

 \begin{lemma}
\label{lem:D-bound}
There exists a deterministic sequence $d(T) = o(T)$ such that the event
\[
\mathcal{E}^D_{\sect}(T) \ceq \left\{
\sum_{t=K}^{T}
  -\bE_{(\mdqr,\mdqp)\sim\nu_t^{\eta_t}(\cdot\mid\alt(\empmdp_t))}
  \!\bigl[\kappa_{t+1}(\mdqr)\bigr]
  \leq d(T)
\right\}
\]
holds with probability at least $1-1/T^2$, where $d(T)$ is given 
explicitly in the proof.
\end{lemma}
\begin{proof}
Recalling the definition of $\kappa_{t+1}(\mdqr)$ from the proof of 
Lemma~\ref{lem:phi-bound} and introducing
\[
u_t(s,a,h) \ceq 
  \bE_{(\mdqr,\mdqp)\sim\nu_t^{\eta_t}(\cdot\mid\alt(\empmdp_t))}
  \!\bigl[\mdqr(s,a,h)\bigr],
\]
we have
\begin{align*}
D_T(T_0) 
  &= \sum_{t=T_0}^{T-1}\sum_{h=1}^H 2\,n_{t+1}(s_h^t,a_h^t,h)\,
    u_t(s_h^t,a_h^t,h)
    \bigl(\muh_{t+1}(s_h^t,a_h^t,h)-\muh_t(s_h^t,a_h^t,h)\bigr) \\
  &= \sum_{t=T_0}^{T-1}\sum_{h=1}^H 2
    \bigl[\mu(s_h^t,a_h^t,h)-\muh_t(s_h^t,a_h^t,h)-\zeta_h^{t}\bigr]
    u_t(s_h^t,a_h^t,h),
\end{align*}
where the second equality uses the identity from Lemma~\ref{lem:J-bound}.
We bound the two resulting terms separately.

\medskip\noindent
Since $u_t(s,a,h)\leq 1$, Cauchy--Schwarz and the 
elliptic potential (Lemma~\ref{lem:elliptic-potential}) give, on 
$$\cE_{\sect}^{\mu}(T) \ceq \lb \forall t \leq T, \: \sum_{s,a, h} n_t(s,a, h) \KL(\hat R_h^t(s,a) \;\|\; R_h(s,a)) \leq  \beta^{\mu}(T, 1/T^2) \rb,$$
the following 
\begin{align*}
\sum_{t=T_0}^{T-1}\sum_{h=1}^H
  \bigl(\mu-\muh_t\bigr)(s_h^t,a_h^t,h)\cdot u_t(s_h^t,a_h^t,h)
  &\leq \sqrt{\sum_{t,h}
    \frac{u_t(s_h^t,a_h^t,h)^2}{\maxp{1}{n_t(s_h^t,a_h^t,h)}}}
  \cdot\sqrt{\sum_{t,h}
    \maxp{1}{n_t(s_h^t, a_h^t, h)}\bigl(\mu-\muh_t \bigr)^2 (s_h^t, a_h^t, h)} \\
  &\leq \sqrt{4H\log(T+1)}\cdot\sqrt{2T\,\beta^\mu(T,1/T^2)}.
\end{align*}

\medskip\noindent
The sequence $(-\zeta_h^{t}\,u_t(s_h^t,a_h^t,h))_{t,h}$ is a martingale difference sequence. By Azuma--Hoeffding's inequality, the event
\[
\left\{
  \sum_{t=T_0}^{T-1}\sum_{h=1}^H -\zeta_h^{t+1}\,u_t(s_h^t,a_h^t,h)
  \leq \sqrt{2\log(T^2)\sum_{t,h} u_t(s_h^t,a_h^t,h)^2}
\right\}
\]
holds with probability at least $1-1/T^2$.

\medskip\noindent
Combining both bounds and using $u_t\leq 1$ gives,
on an event of probability at least $1-2/T^2$,
\[
D_T(T_0) \leq 2\sqrt{8HT\,\beta^\mu(T,1/T^2)\log(T+1)}
  + 2\sqrt{2HT\log(T^2)}
  = o(T),
\]
which defines the event $\cE_{\sect}^D(T)$.
\end{proof}

\subsection{Proof of Theorem~\ref{thm:inf-no-regret}}
We can now prove  Theorem~\ref{thm:inf-no-regret} with the previous intermediate results. 

Let us define 
\begin{equation}
	\cE_{\sect} (T) = \cE_{\ref*{appx:sufficient_explore}} (T) \cap \cE^\mu_{\sect} (T) \cap \cE_{\sect}^\phi(T) \cap \cE_{\sect}^D(T) \cap \cE_{\sect}^{\mathscr{J}}(T) \cap \cE_{\sect}^{\varPsi}(T),
\end{equation}
where $\cE_{\ref*{appx:sufficient_explore}} (T)$ is defined in Appendix~\ref{appx:sufficient_explore}, and the other events are defined above. By Lemma~\ref{lem:estimation}, there exists $k_0<\infty$ such that for $T\geq k_0$, when $\cE_{\ref*{appx:sufficient_explore}} (T)$ holds, 
$$ \hat\pi_t = \pi^\star, \quad \forall t \in (T^2/3, T).$$

Setting $T_0 \ceq T^{2/3}$ in Lemma~\ref{lem:hedge-regret}, $\cE_{\ref*{appx:sufficient_explore}} (T) \impl \cE_{\sect}^{T^{2/3}}(T)$, so that 
\begin{align*}
\sum_{t=T^{2/3}}^{T-1}
  \bE_{(\mdqr,\mdqp)\sim\nu_t^{\eta_t}(\cdot\mid\alt(\empmdp_t))}
  \!\bigl[\ell_t^\alpha(\mdqr,\mdqp)\bigr]
  \;\leq\;
  &-\frac{1}{\eta_T}\log\Lambda_T^{\eta_T}
  + \varPsi_T^\alpha(T^{2/3})
  + \frac{1}{\eta_{T^{2/3}}}\log\Lambda_{T^{2/3}}^{\eta_{T^{2/3}}}\; +  \\
  & \sum_{t=T^{2/3}}^{T-1}
    \left[
      \frac{1}{\eta_{t}}\log\Lambda_{t+1}^{\eta_{t}}
      - \frac{1}{\eta_{t+1}}\log\Lambda_{t+1}^{\eta_{t+1}}
    \right].
\end{align*} 

Thanks to Lemma~\ref{lem:phi-bound}, Lemma~\ref{lem:ZT-bound}, and Lemma~\ref{lem:utl_lower_bound}, and noting that $\cE_{\sect}(T)$ holds with probability at least $1-O(1/T^2)$ concludes the proof.

 \section{SADDLE-POINT CONVERGENCE}
\label{appx:saddle_point_convergence}

In this section, we prove the saddle-point convergence result that 
characterizes the optimality of the sampling rule in terms of posterior 
contraction. We work on the intersection of the following concentration 
events:
\begin{align}
\label{eq:conc-events}
\mathcal{E}_{\sect}^p(T) &\ceq \left\{\forall\,t\leq T:\;
  \sum_{s,a,h} n_t(s,a,h) \KL\!\left(\hat{p}_t(\cdot\mid s,a,h)\,\|\,p(\cdot\mid s,a,h)\right)
  \leq \beta^p(T,1/T^2)\right\}, \nonumber\\
  \mathcal{E}_{\sect}^\mu(T) &\ceq \left\{\forall\,t\leq T:\; \sum_{s,a,h} n_t(s,a,h)
  \KL\!\left(\hat{R}_h^t(s,a)\,\|\,R_h(s,a)\right)
  \leq \beta^\mu(T,1/T^2)\right\}, \nonumber\\ \cE_{\sect}^{\mu+p}(T) &\ceq \left\{\forall\,t\leq T:\;
  \sum_{s,a,h} n_t(s,a,h) \KL\!\left(\hat{R}_h^t(s,a)\otimes\hat{p}_h^t(\cdot\mid s,a)
  \,\|\,R_h(s,a)\otimes p(\cdot\mid s,a, h)\right) \leq \beta^{\mu+p}(T,1/T^2)\right\}.
\end{align}

Explicit thresholds $\beta^p,\beta^\mu,\beta^{\mu+p}$ ensuring high-probability 
coverage follow from self-normalized concentration (see e.g.~\citet{kaufmann_mixture_2021}). Each event holds with probability 
at least $1-1/T^2$, so their intersection holds with probability at least $1-3/T^2$.

\medskip\noindent
We recall the log-likelihood ratio between a model $(\mdqr,\mdqp)$ and the 
empirical MDP $\empmdp_t$:
\begin{equation}
\label{eq:def-glr}
\Upsilon_t(\mdqr,\mdqp) \ceq \sum_{s,a,h} n_t(s,a,h)\left[
  \frac{\bigl(\mdqr(s,a,h)-\muh_t(s,a,h)\bigr)^2}{2}
  +\sum_{s'}\hat{p}_t(s'\mid s,a,h)
   \log\frac{\hat{p}_t(s'\mid s,a,h)}{\mdqp(s'\mid s,a,h)}
\right].
\end{equation}

The $\glr$ (Generalized Likelihood Ratio) is defined as 

\begin{equation}
	\glr(T) \ceq \inf_{(\mdqr, \mdqp)\in \alt(\empmdp_T)} {\Upsilon}_T(\mdqr, \mdqp)
\end{equation}

The main result of this section is the following.

\begin{proposition}
\label{prop:saddle-point}
Let $\gamma\in(0,1)$, $\alpha\in(0,\frac{1}{2})$, and $\varsigma > 2\alpha$ and $\eta_t = O(t^{-\varsigma})$. There exists $k_0<\infty$ such that for any $T > k_0$,
\[
\glr(T) \geq T \cdot \Gamma_\mdp - O(\log T)- O(T^{1+\alpha - \gamma}) - o(T) - O\! \lp T^{ \alpha+\frac{\gamma+1}2} \sqrt {\log T}\rp\:,
\]
holds with probability at least $1- {O}(1/T^2)$.
\end{proposition}

\begin{proof}
 We note that 
\begin{eqnarray*}
	\glr(T) &\ceq& \inf_{(\mdqr, \mdqp)\in \alt(\empmdp_T)} {\Upsilon}_T(\mdqr, \mdqp)\:, \\
	&=& \inf_{(\mdqr, \mdqp)\in \alt(\empmdp_T)} \sum_{s,a,h} n_T(s,a, h)\lsb  \frac{(\mdqr_h(s,a) - \muh_T(s,a, h))^2}{2} + \sum_{s'} \hat p_T(s'\mid s,a,h)\log \frac{1}{\mdqp(s'\mid s,a,h)}\rsb \\ && +  \sum_{s,a,h}\sum_{s'}n_T(s,a)\hat p_T(s'\mid s,a,h)\log {\hat p_T(s'\mid s,a,h)}\:.
\end{eqnarray*}

\medskip 
We recall that in the equation above, for unobserved transitions, $0\cdot\log 0 = 0$. Next, since $n_T(s,a, h)\hat p_T(s'\mid s,a,h) = n_T(s'\mid s,a,h)$ (and this is always true, even at initialization), we have 
$$\sum_{s,a,h} \sum_{s'} n_T(s,a, h) \hat p_T(s' \mid s,a,h)\log \frac{1}{\mdqp(s'\mid s,a, h)}   = \sum_{t \in \lbrack T-1\rbrack} \sum_h \log \frac{1}{\mdqp(s_{h+1}^t \mid s_h^t,a_h^t, h)}.$$ 

\medskip\noindent Similarly we have 
$$ \sum_{s,a,h} n_T(s,a,h)  \frac{\bigl(\mdqr(s,a, h) - \muh_T(s,a,h)\bigr)^2}{2} = \sum_{t \in \lbrack T-1\rbrack} \sum_h \frac{(\mdqr(s_h^t,a_h^t, h) - \muh_T(s_h^t,a_h^t, h))^2}{2}.$$

Combining these yields
\begin{align*}
	\Upsilon_T(\mdqr,\mdqp) = \sum_{t \in \lbrack T-1\rbrack} \lsb \sum_{h} \frac{(\mdqr(s_h^t,a_h^t, h) - \muh_T(s_h^t,a_h^t, h))^2}{2} + \log \frac{1}{\mdqp(s_{h+1}^t \mid s_h^t,a_h^t,h)}\rsb + \sum_{t \in \lbrack T-1\rbrack}\sum_h \log \hat p_T(s_{h+1}^t\mid s_h^t, a_h^t,h) \:.
\end{align*}

Moreover, the above is well-defined since $\forall t<T,  p_T(s_{h+1}^t\mid s_h^t, a_h^t,h) >0$.

\medskip

Next, we invoke Theorem~\ref{thm:inf-no-regret} using the regret properties of the conditional posterior sampling algorithm, we have with probability larger than $1-O(1/T^2)$, 
\begin{align}
\begin{split}
&\inf_{(\mdqr, \mdqp)\in \alt(\empmdp_T)} \sum_{t\in \lbrack T-1 \rbrack } \sum_h \lsb \frac{\bigl(\mdqr(s_h^t,a_h^t, h) - \muh_T(s_h^t,a_h^t, h)\bigr)^2}{2} + \log \frac{1}{\mdqp(s_{h+1}^t\mid s_h^t,a_h^t, h)} \rsb \geq \\  & \sum_{t=T^{2/3} }^{T-1} \bE_{(\mdqr ,\mdqp )\sim \nu_t^{\eta_t}(\cdot \mid \alt(\empmdp_t))}\bigl[  \ell^\alpha_t(\mdqr ,\mdqp)\bigr] \; - \; o(T)\:,
\end{split}
	\end{align}
	which yields 
	\begin{align}
	\label{eq:glr-lb1}	
		&\glr(T)\geq \sum_{t = 1}^{T-1} \bE_{(\mdqr ,\mdqp)\sim \nu_t^{\eta_t}(\cdot \mid \alt(\empmdp_t))}\lsb  \ell^\alpha_t(\mdqr ,\mdqp)\rsb + \sum_{t\in \lbrack T-1 \rbrack} \sum_h \log {\hat p_T(s_{h+1}^t\mid s_h^t,a_h^t,h)} - o(T), 
		\end{align} 
where, since $\ell^\alpha_t \in (0, H(\frac12 + t^{\alpha}))$ and $\alpha < \frac12$, the first $T^{2/3}$ terms of the sum contribute for $O(T^{\frac{2+2\alpha}{3}}) = o(T)$. 
		
		\medskip\noindent
By Lemma~\ref{lem:expected_to_random_loss}, 
with probability at least $1-O(1/T^2)$, we have
\begin{equation}
\label{eq:random-to-expected}
\sum_{t \in \lbrack T-1\rbrack}\ell_t^\alpha(\mdqr_t,\mdqp_t)
\leq \sum_{t \in \lbrack T-1\rbrack}
  \bE_{(\mdqr,\mdqp)\sim\nu_t^{\eta_t}(\cdot\mid\alt(\empmdp_t))}
  \!\bigl[\ell_t^\alpha(\mdqr,\mdqp)\bigr]
  + \sqrt{2\log(T^2)\sum_{t=1}^T C_{t,\alpha}^2}, 
\end{equation}
with $C_{t, \alpha} = H(\frac12 + t^{\alpha})$. We recall  that $\mdq_t \ceq (\mdqr_t, \mdqp_t)$ is the challenger MDP sampled at time $t$ in the pseudocode of Algorithm~\ref{alg:repit}. 

Substituting \eqref{eq:random-to-expected} into \eqref{eq:glr-lb1} and recalling the definition of $\ell_t^\alpha$ from 
\eqref{eq:def-trunc-loss}, we obtain
\begin{align}
\glr(T)
  &\geq \sum_{t \in \lbrack T-1 \rbrack}\sum_{h=1}^H\left[
    \frac{\bigl(\muh_t(s_h^t,a_h^t,h)-\mdqr_t(s_h^t,a_h^t,h)\bigr)^2}{2}
    +\log\frac{1}{\maxp{
      \mdqp_t(s_{h+1}^t\mid s_h^t,a_h^t,h)}{\myexp{-t^\alpha}
    }}
  \right] \nonumber\\
  &\quad+\sum_{t \in \lbrack T-1 \rbrack}\sum_{h=1}^H
    \log\hat{p}_T(s_{h+1}^t\mid s_h^t,a_h^t,h)
  - o(T) \nonumber \\
  &= \sum_{t \in \lbrack T-1 \rbrack} \sum_{h=1}^H\left[
    \frac{\bigl(\muh_t(s_h^t,a_h^t,h)-\mdqr_t(s_h^t,a_h^t,h)\bigr)^2}{2}
    +\log\frac{ \maxp{\hat{p}_t(s_{h+1}^t\mid s_h^t,a_h^t,h)}{\veps}} 
      {\maxp{
        \mdqp_t(s_{h+1}^t\mid s_h^t,a_h^t,h)}{\myexp{-t^\alpha}
      }}
  \right] \nonumber\\
  &\quad+\sum_{t \in \lbrack T-1 \rbrack}\sum_{h=1}^H
    \log\frac{\hat{p}_T(s_{h+1}^t\mid s_h^t,a_h^t,h)}
             {\maxp{\hat{p}_t(s_{h+1}^t\mid s_h^t,a_h^t,h)}{\veps}}
  - o(T), \label{eq:glr-lb3}
\end{align}
where the equality rearranges the logarithms by adding and subtracting 
$\maxp{\hat{p}_t(s_{h+1}^t\mid s_h^t,a_h^t,h)}{\veps}$.

\medskip

Thanks to Lemma~\ref{lem:pulls-to-w}, 
with probability larger than $1-O(1/T^2)$.

\begin{align}
\label{eq:pulls-to-w}
\begin{split}
  &\sum_{t=1}^{T-1}\sum_{s,a,h} w_h^{\bpolmix_t}(s,a)
    \left[
      \frac{\bigl(\muh_t(s,a,h)-\mdqr_t(s,a,h)\bigr)^2}{2}
      +\sum_{s'} p(s'\mid s,a, h)
       \log\frac{\maxp{\hat{p}_t(s'\mid s,a,h)}{\veps}}
         {\maxp{\mdqp_t(s'\mid s,a,h)}{\myexp{-t^\alpha}}}
    \right] \\
  &-\sum_{t=1}^{T-1}\sum_{h=1}^H\left[
      \frac{\bigl(\muh_t(s_h^t,a_h^t,h)-\mdqr_t(s_h^t,a_h^t,h)\bigr)^2}{2}
      +\log\frac{\maxp{\hat{p}_t(s_{h+1}^t\mid s_h^t,a_h^t,h)}{\veps}}
        {\maxp{\mdqp_t(s_{h+1}^t\mid s_h^t,a_h^t,h)}{\myexp{-t^\alpha}}}
    \right]
  \leq \\ & \sqrt{2\sum_{t=1}^{T-1} 
    H^2(t^\alpha + \tfrac12)^2 \log(T^2)}
    \end{split}
\end{align}

Combining \eqref{eq:glr-lb3} with \eqref{eq:pulls-to-w} gives

\begin{align}
\label{eq:glr-lb4}
\begin{split}
&\glr(T) \geq \sum_{t=1}^{T-1}\sum_{s,a,h} w_h^{\bpolmix_t}(s,a)\left[
    \frac{\bigl(\muh_t(s,a,h)-\mdqr_t(s,a,h)\bigr)^2}{2}
    +\sum_{s'} p(s'\mid s,a,h)
     \log\frac{\maxp{\hat{p}_t(s'\mid s,a,h)}{\veps}}
       {\maxp{\mdqp_t(s'\mid s,a,h)}{\myexp{-t^\alpha}}}
  \right] \\
  &\quad +\sum_{t=1}^{T-1}\sum_{h=1}^H
    \log\frac{\hat{p}_T(s_{h+1}^t\mid s_h^t,a_h^t,h)}
             {\maxp{\hat{p}_t(s_{h+1}^t\mid s_h^t,a_h^t,h)}{\veps}} - o(T). 
  \end{split}
\end{align}

\medskip
Next, we relate the true transition $p(s'\mid s,a,h)$ to the empirical 
$\hat{p}_t(s'\mid  s,a,h)$ in the logarithm above. 

We then invoke Lemma~\ref{lem:true_trans_to_emp_trans}, to show that with probability at least $1-O(1/T^2)$, 
\begin{align}
  \label{eq:true-to-emp-trans}
\begin{split}
\sum_{t=1}^{T-1}\sum_{s,a,h} w_h^{\bpolmix_t}(s,a)
  \sum_{s'} p(s'\mid  s,a,h)
  \log\frac{\maxp{\hat{p}_t(s'\mid s,a,h)}{\veps}}
    {\maxp{\mdqp_t(s'\mid s,a,h)}{\myexp{-t^\alpha}}}
  &\geq\; \\
  \sum_{t=1}^{T-1}\sum_{s,a,h} w_h^{\bpolmix_t}(s,a)
  \sum_{s'}\hat{p}_t(s'\mid s,a,h)
  \log\frac{\maxp{\hat{p}_t(s'\mid s,a,h)}{\veps}}
    {\maxp{\mdqp_t(s'\mid s,a,h)}{\myexp{-t^\alpha}}}&
  - O\! \lp T^{\alpha + \frac{\gamma +1}{2}} \sqrt {\log T}  \rp - o(T)
  \end{split}
\end{align}
Substituting \eqref{eq:true-to-emp-trans} into \eqref{eq:glr-lb4} and noting that $$ \hat{p}_t(s'\mid s,a,h) \log\maxp{\hat{p}_t(s'\mid s,a,h)}{\veps} \geq \hat{p}_t(s'\mid s,a,h) \log\hat{p}_t(s'\mid s,a,h)$$
gives 
\begin{align}
\glr(T)
  &\geq \sum_{t=1}^T\sum_{s,a,h} w_h^{\tilde\rho^t}(s,a)\left[
    \frac{\bigl(\muh_t(s,a,h)-\mdqr_t(s,a,h)\bigr)^2}{2}
    +\sum_{s'}\hat{p}_t(s'\mid s,a,h)
     \log\frac{\maxp{\hat{p}_t(s' \mid s,a,h)}{\veps}}
       {\maxp{\mdqp_t(s' \mid s,a,h)}{\myexp{-t^\alpha}}}
  \right] \nonumber\\
  &\quad+\sum_{t=1}^T\sum_{h=1}^H
    \log\frac{\hat{p}_T(s_{h+1}^t\mid s_h^t,a_h^t,h)}
             {\maxp{\hat{p}_t(s_{h+1}^t\mid s_h^t,a_h^t,h)}{\veps}}
  - O\! \lp T^{\alpha + \frac{\gamma +1}{2}} \sqrt {\log T} \rp - o(T)
. \label{eq:glr-lb5}
\end{align}

\medskip

Recalling the definition of the reward kernel
\begin{equation}
\label{eq:def-loss-kernel}
\loss_t(s,a,h) \ceq
  \frac{\bigl(\muh_t(s,a,h)-\mdqr_t(s,a,h)\bigr)^2}{2}
  +\sum_{s'}\hat{p}_t(s,a,s',h)
   \log\frac{\hat{p}_t(s' \mid s,a,h)}
     {\maxp{\mdqp_t(s'\mid s,a,h}{\myexp{-t^\alpha}}},
\end{equation}
and substituting into \eqref{eq:glr-lb5}, we obtain, with probability $1- O(1/T^2)$ 
\begin{align}
\glr(T)
  &\geq \sum_{t=1}^T\sum_{s,a,h} w_h^{\bpolmix_t}(s,a)\,\loss_t(s,a,h)
  +\sum_{t=1}^T\sum_{h=1}^H
    \log\frac{\hat{p}_T(s_{h+1}^t\mid s_h^t,a_h^t,h)}
             {\maxp{\hat{p}_t(s_{h+1}^t\mid s_h^t,a_h^t,h)}{\veps}}
  - {O}\!\left(T^{\alpha+\frac{\gamma+1}{2}} \sqrt {\log T} \right)
  - o(T). \label{eq:glr-lb6}
\end{align}
Decomposing the mixed policy $\bpolmix_t = (1-\varepsilon_t)\bpol_t
+\varepsilon_t c_t$,
\begin{align}
\sum_{t=1}^{T-1}\sum_{s,a,h} w_h^{\bpolmix_t}(s,a)\,\loss_t(s,a,h)
  &= \sum_{t=1}^{T-1}\sum_{s,a,h} w_h^{\bpol_t}(s,a)\,\loss_t(s,a,h)
  + \sum_{t=1}^{T-1}\sum_{s,a,h}
    \veps_t\bigl[w_h^{c_t}(s,a)-w_h^{\bpol_t}(s,a)\bigr]
    \loss_t(s,a,h).
  \label{eq:mixture-decomp}
\end{align}

By Lemma~\ref{lem:borne_max_loss}, the correction term satisfies
\begin{equation}
\label{eq:mixture-correction}
\left|\sum_{t=1}^{T-1}\sum_{s,a,h}
  \varepsilon_t\bigl[w_h^{c_t}(s,a)-w_h^{\bpol_t}(s,a)\bigr]
  \loss_t(s,a,h)\right|
\leq \frac{2}{1+\alpha-\gamma}\,T^{1+\alpha-\gamma}.
\end{equation}
Substituting \eqref{eq:mixture-decomp} and \eqref{eq:mixture-correction} 
into \eqref{eq:glr-lb6},
\begin{equation}
\label{eq:glr-lb7}
\glr(T) \geq \sum_{t=1}^T\sum_{s,a,h} w_h^{\bpol_t}(s,a)\,\loss_t(s,a,h)
  +\sum_{t=1}^{T-1}\sum_{h=1}^H
    \log\frac{\hat{p}_T(s_{h+1}^t\mid s_h^t,a_h^t,h)}
             {\hat{p}_t(s_{h+1}^t\mid s_h^t,a_h^t,h)}
  - {O}\!\left(T^{1+\alpha-\gamma}\right)
  - {O}\!\left(T^{\alpha+\frac{\gamma+ 1}{2}} \sqrt {\log T} \right)
  - o(T).
\end{equation}
By Lemma~\ref{lem:control_ratio_prob}, with probability larger than $1-O(1/T^2)$, 
\begin{equation} 
\label{eq:ratio-prob-bound}
\left|\sum_{t=1}^{T-1}\sum_{h=1}^H
  \log\frac{\hat{p}_T(s_{h+1}^t\mid s_h^t,a_h^t,h)}
           {\maxp{\hat{p}_t(s_{h+1}^t\mid s_h^t,a_h^t,h)}{\veps}}
\right|
\leq {O}\!\left(T^{\alpha+\frac{\gamma+ 1}{2}} \sqrt {\log T} / \veps\right)
\end{equation}
Combining \eqref{eq:glr-lb7} and \eqref{eq:ratio-prob-bound} gives 
\begin{equation}
\label{eq:glr-lb8}
\glr(T) \geq \sum_{t=1}^{T-1}\sum_{s,a,h} w_h^{\bpol_t}(s,a)\,\loss_t(s,a,h)
  - {O}\!\left(T^{1+\alpha-\gamma}\right)
  - {O}\!\left(T^{\alpha+\frac{\gamma+ 1}{2}} \sqrt {\log T} \right)
  - o(T).
\end{equation}

\medskip

Using the no-regret property of the online RL learner (Proposition~\ref{prop:guarantees-regmin}), we have for $\alpha \in (0,\frac12)$ 
\begin{eqnarray*}
	\glr(T) &\geq&  \sup_{w \in \Omega_\mdp}  \sum_{t\in \lbrack T-1 \rbrack}\sum_{s,a, h} w_h(s,a) \loss_t(s,a, h) - O\!(T^{1+\alpha - \gamma})  - O\!\lp T^{ \alpha + \frac{\gamma +1}2 } \sqrt {\log T} \rp  - o(T).
\end{eqnarray*}
\medskip
\noindent
Next, we invoke Lemma~\ref{lem:emp-to-real} to relate the empirical quantities in  $\loss_t$ to their actual values. Indeed Lemma~\ref{lem:emp-to-real} shows that on an event that holds with probability $1-O(1/T^2)$ 
\begin{align}
\begin{split}
& \sum_{t=1}^{T-1}\sum_{s,a,h} w_h(s,a)\,\loss_t(s,a,h)
  \geq \\ & \sum_{t=1}^{T-1}\sum_{s,a,h} w_h(s,a)
  \left[
    \frac{\bigl(\mdqr_t(s,a,h)-\mu(s,a,h)\bigr)^2}{2}
    +\sum_{s'} p(s'\mid s,a,h)
     \log\frac{p(s'\mid s,a,h)}
       {\maxp{\mdqp_t(s'\mid s,a,h)}{\myexp{-t^\alpha}}}
  \right]
  + o(T)
  \end{split}	
\end{align}

Therefore, 
\begin{eqnarray*}
	\glr(T) &\geq& \sup_{w \in \Omega_\mdp}  \sum_{t=1}^{T-1} \sum_{s,a,h} w_h(s,a) \lsb  \frac{\bigl(\mdqr_t(s, a, h) - \mu(s,a, h)\bigr)^2}{2}  +  \sum_{s'} p(s'\mid s,a,h) \log \frac{p(s' \mid s,a,h)}{\maxp{\mdqp_t(s'\mid s,a, h)}{\myexp{-t^{\alpha}}}} \rsb \\&&- O(T^{1+\alpha - \gamma}) - o(T) - O\!\lp T^{\alpha+\frac{\gamma + 1}2} \sqrt {\log T} \rp
\end{eqnarray*}

\medskip
The final step to conclude the proof invokes Lemma~\ref{lem:make_inf_appear} to prove that there exists $k_0<\infty$ such that for $T\geq k_0$, we have for any $w\in \Omega_\mdp$
\begin{align*}
\sum_{s,a, h} w_h(s,a) \lsb \frac{\bigl( \mdqr_t(s, a, h) - \mu(s,a, h) \bigr)^2}{2}   +  \sum_{s'} p(s'\mid s,a,h) \log \frac{p(s'\mid s,a,h)}{\maxp{\mdqp_t(s'\mid s,a, h)}{ \myexp{-t^{\alpha}}}}\rsb  \geq \\
\inf_{(\mdqr , \mdqp) \in \alt(\mdp)}\sum_{s,a, h} w_h(s,a) \lsb \frac{\bigl(\mdqr(s, a, h) - \mu(s,a, h) \bigr)^2}{2}  +  \sum_{s'} p(s'\mid s,a,h) \log \frac{p(s'\mid s,a,h)}{\mdqp(s'\mid s,a, h) }\rsb - O(1/t).
\end{align*}

We then have 
\begin{eqnarray*}
	\glr(T) &\geq&  \sup_{w \in \Omega_\mdp}  \sum_{t=1}^{T-1} \inf_{(\mdqr , \mdqp) \in \alt(\mdp)}\sum_{s,a, h} w_h(s,a) \lsb \frac{(\mdqr(s, a, h) - \mu(s,a, h))^2}{2}  +  \sum_{s'} p(s'\mid s,a,h) \log \frac{p(s' \mid s,a,h)}{\mdqp_h(s,a, s') }\rsb \\&& - O(\log T)- O(T^{1+\alpha - \gamma}) - o(T) - O\!\lp T^{ \alpha+\frac{\gamma +1}2} \sqrt {\log T}\rp \\
	&=& T \cdot \Gamma_\mdp - O(\log T)- O(T^{1+\alpha - \gamma}) - o(T) - O\! \lp T^{ \alpha+\frac{\gamma+1}2} \sqrt {\log T}\rp\:,
\end{eqnarray*}
which concludes the proof.

\end{proof}

\subsection{Concentration Lemmas}
We now prove the concentrations lemmas used in this section. 

\begin{lemma}
\label{lem:true_trans_to_emp_trans}
There exists $k_0<\infty$ such that for  $T\geq k_0$, with probability $1 - O(1/T^2)$, the following holds: 
\begin{align*}
&\sum_{t=1}^T\sum_{s,a,h} w_h^{\bpolmix_t}(s,a)\sum_{s'}
  p(s'\mid s,a, h)\log\frac{\hat{p}_t(s'\mid s,a,h)}
    {\maxp{\mdqp_t(s'\mid s,a,h)}{\myexp{-t^\alpha}}}
  \geq \\ &
\sum_{t=1}^T\sum_{s,a,h} w_h^{\bpolmix_t}(s,a)\sum_{s'}
  \hat{p}_t(s'\mid s,a,h)\log\frac{\hat{p}_t(s'\mid s,a,h)}
    {\maxp{\mdqp_t(s'\mid s,a,h)}{\myexp{-t^\alpha}}}
  - O\! \lp T^{\alpha + \frac{\gamma +1}{2}} \sqrt {\log T} \rp - o(T).
\end{align*}
\end{lemma}

\begin{proof}
For fixed $(t,s,a,s',h)$, letting 
$L_t(s,a,s') \ceq \log\frac{\hat{p}_t(s,a,s',h)}
{\maxp{\mdqp_t(s'\mid s,a,h)}{\myexp{-t^\alpha}}}$,
\[
\bigl|p_h(s,a,s') - \hat{p}_t(s,a,s',h)\bigr|\cdot|L_t(s,a,s')|
\leq \|\hat{p}_t(\cdot\mid s,a,h)-p_h(\cdot\mid s,a,h)\|_1\cdot t^\alpha,
\]
since $|L_t|\leq t^\alpha$. Thanks to Lemma~\ref{lem:forced_exp}, there exists $k_0<\infty$, such that for $T\geq k_0$ 
for all $t\in (T^{2/3}, T)$ and all reachable $(s,a,h)$, if $\cE_{\sect}^p(T)$ holds, then 
\[
\|\hat{p}_t(\cdot\mid s,a,h)-p(\cdot\mid s,a,h)\|_1
  \leq O \! \lp  \sqrt{t^{\gamma-1} \beta^p(t^2, 1/t^3) }\rp ,
\]
so the per-$(t,s,a,h)$ error is at most $O \! \lp  \sqrt{t^{2\alpha + \gamma-1} \beta^p(t^2, 1/t^3) }\rp$.
The first $T^{2/3}$ terms contributes for $O(T^{\frac{2+2\alpha}{3}}) = o(T)$ (as $\alpha \in (0, \frac12)$). Summing over $t$, $s$, $a$, $h$ and weighting by $w_h^{\bpolmix_t}$, the error term is at most 
\[
O\! \lp T^{\alpha + \frac{\gamma +1}{2}} \sqrt {\beta^p(T^2, 1/T^3)} \rp. 
\]
 The  conclusion follows as $\beta^p(T^2, 1/T^3) = O(\log T)$. 
\end{proof}

\begin{lemma}
\label{lem:borne_max_loss}
We have 
\[
\left|\sum_{t=1}^T\sum_{s,a,h}
  \varepsilon_t\biggl[w_h^{c_t}(s,a)-w_h^{\bpol_t}(s,a)\biggr]
  \loss_t(s,a,h)
\right|
\leq O(T^{1-\gamma + \alpha}).
\]
\end{lemma}

\begin{proof}
Since $|\loss_t(s,a,h)|\leq \frac{1}{2}+t^\alpha$ for all $(s,a,h)$
and $\veps_t = t^{-\gamma}$, the triangle inequality gives
\begin{align*}
\left|\sum_{t=1}^T\sum_{s,a,h}
  \varepsilon_t\bigl[w_h^{c_t}(s,a)-w_h^{\bpol_t}(s,a)\bigr]
  \loss_t(s,a,h)\right|
  &\leq \sum_{t=1}^T\varepsilon_t
    \left|\sum_{s,a,h}
      \bigl[w_h^{c_t}(s,a)-w_h^{\bpol_t}(s,a)\bigr]
      \loss_t(s,a,h)\right| \\
  &\leq 2\sum_{t=1}^T t^{-\gamma}\!\left(t^\alpha+\frac{1}{2}\right) \\
  &\leq O(T^{1-\gamma + \alpha}),
\end{align*}
where the second inequality uses that 
$\sum_{s,a,h}|w_h^{c_t}(s,a)-w_h^{\bpol_t}(s,a)|\leq 2$, 
and the last step bounds the sums by integrals.
\end{proof}

\begin{lemma}
\label{lem:emp-to-real}
There  exists $k_0 < \infty$ such that for $T\geq k_0$, for any valid state-action allocation 
$w\in\Omega_\mdp$, the following ,
\begin{align}
\begin{split}
& \sum_{t=1}^T\sum_{s,a,h} w_h(s,a)\,\loss_t(s,a,h)
  \geq \\ & \sum_{t=1}^T\sum_{s,a,h} w_h(s,a)
  \left[
    \frac{\bigl(\mdqr_t(s,a,h)-\mu(s,a,h)\bigr)^2}{2}
    +\sum_{s'} p(s'\mid s,a,h)
     \log\frac{p(s'\mid s,a,h)}
       {\maxp{\mdqp_t(s'\mid s,a,h)}{\myexp{-t^\alpha}}}
  \right]
  + O \!\lp T^{\alpha + \tfrac{1+ \gamma }{2}} \sqrt{\log T} \rp
  \end{split}	
\end{align}
holds with probability larger than $1- O(1/T^2)$. 
\end{lemma}

\begin{proof}
Since $w_h(s,a)=0$ for any unreachable $(s,a,h)$, we may restrict 
attention to reachable triplets throughout. For each $(s,a,s',h)$, define
\[
C_t(s,a,s',h) \ceq
  \hat{p}_t(s'\mid s,a,h)\log\frac{\hat{p}_t(s'\mid s,a,h)}
    {\maxp{\mdqp_t(s'\mid s,a,h)}{\myexp{-t^\alpha}}}
  - p(s'\mid s,a,h)\log\frac{p(s'\mid s,a,h)}
    {\maxp{\mdqp_t(s'\mid s,a,h)} {\myexp{-t^\alpha}}}.
\]
If $p(s'\mid s,a,h)=0$ then $C_t(s,a,s',h)\geq 0$ and the term only 
improves the bound. For $p(s'\mid s,a,h)>0$, the function 
$x\mapsto x\log\frac{x}{\varepsilon}$ is convex, so
\begin{align*}
C_t(s,a,s',h)
  &\geq \left(1+\log\frac{p_h(s,a,s')}
      {\maxp{\mdqp_t(s'\mid s,a,h)}{\myexp{-t^\alpha}}}\right)
    \bigl(\hat{p}_t(s'\mid s,a,h)-p(s'\mid s,a,h)\bigr) \\
  &\geq -\bigl(1+t^\alpha+|\log p(s'\mid s,a,h)|\bigr)
    \bigl|\hat{p}_t(s'\mid s,a,h)-p(s'\mid s,a,h)\bigr|.
\end{align*}

 Thanks to Lemma~\ref{lem:forced_exp}, there exists $k_0<\infty$, such that for $T\geq k_0$ 
for all $t\in (T^{2/3}, T)$ and all reachable $(s,a,h)$, if $\cE_{\sect}^p(T)$ holds, then 
\[
\|\hat{p}_t(\cdot\mid s,a,h)-p(\cdot\mid s,a,h)\|_1
  \leq O \! \lp  \sqrt{t^{\gamma-1} \beta^p(t^2, 1/t^3) }\rp ,
\]
so
\[
C_t(s,a,s',h) 
  \geq -\bigl(1+t^\alpha+|\log p_h(s,a,s')|\bigr)
    O\!\left(\sqrt{ t^{\gamma-1} \beta^p(t^2, 1/t^3) }\right).
\]

Applying a similar resaoning to the rewards (bounded in $(0,1)$) yield 
$$ \bigl(\mdqr_t(s,a,h)- \muh_t(s,a,h)\bigr)^2 \geq  \bigl(\mdqr_t(s,a,h) -  \mu(s,a,h)\bigr)^2 +  \bigl( \muh_t(s,a,h) - \mu(s,a,h)\bigr)^2 -  2 \bigl | \muh_t(s,a,h) - \mu(s,a,h) \bigr|, $$
which, again, due to Lemma~\ref{lem:forced_exp} yields 

$$ \bigl | \muh_t(s,a,h) - \mu(s,a,h) \bigr| \leq O\!\left(\sqrt{ t^{\gamma-1} \beta^p(t^2, 1/t^3) }\right).$$
 
Summing over $t$, $s$, $a$, $s'$, $h$ and weighting by $w_h(s,a)$,
\begin{align*}
	& \sum_{t=1}^T\sum_{s,a,h} w_h(s,a)\,\loss_t(s,a,h)
  \geq \\  &\sum_{t=1}^T\sum_{s,a,h} w_h(s,a)
  \left[
    \frac{\bigl(\mdqr_t(s,a,h)-\mu(s,a,h)\bigr)^2}{2}
    +\sum_{s'} p(s'\mid s,a,h)
     \log\frac{p(s'\mid s,a,h)}
       {\maxp{\mdqp_t(s'\mid s,a,h)}{\myexp{-t^\alpha}}}
  \right] -  \sum_{t \leq T} t^{\alpha} O\!\left(\sqrt{ t^{\gamma-1} \beta^p(t^2, 1/t^3) }\right). 
\end{align*}

Further noting that 
$$ \sum_{t \leq T} t^{\alpha} O\!\left(\sqrt{t^{\gamma-1} \beta^p(t^2, 1/t^3) }\right) \leq O \!\lp T^{\alpha + \tfrac{1+ \gamma }{2}} \sqrt{\log T} \rp$$ 
concludes the proof. 
\end{proof}

\begin{lemma}
\label{lem:pulls-to-w}
Fix $\veps \in (0,1)$. The event
\begin{align}
\mathcal{E}_w(T) \ceq \Biggl\{
  &\sum_{t=1}^T\sum_{s,a,h} w_h^{\bpolmix_t}(s,a)
    \left[
      \frac{\bigl(\muh_t(s,a,h)-\mdqr_t(s,a,h)\bigr)^2}{2}
      +\sum_{s'} p(s'\mid s,a, h)
       \log\frac{\maxp{\hat{p}_t(s'\mid s,a,h)}{\veps}}
         {\maxp{\mdqp_t(s'\mid s,a,h)}{\myexp{-t^\alpha}}}
    \right] \nonumber\\
  &-\sum_{t=1}^T\sum_{h=1}^H\left[
      \frac{\bigl(\muh_t(s_h^t,a_h^t,h)-\mdqr_t(s_h^t,a_h^t,h)\bigr)^2}{2}
      +\log\frac{\maxp{\hat{p}_t(s_{h+1}^t\mid s_h^t,a_h^t,h)}{\veps}}
        {\maxp{\mdqp_t(s_{h+1}^t\mid s_h^t,a_h^t,h)}{\myexp{-t^\alpha}}}
    \right]
  \leq \\ & \sqrt{2\sum_{t=1}^T 
    H^2(t^\alpha + \tfrac12)^2 \log(T^2)}
\Biggr\} 
\end{align}
holds with probability at least $1-1/T^2$.
\end{lemma}

\begin{proof}
Let us define, for each $t$ and $(s,a,s',h)$,
\[
\Psi_t(s,a,s',h) \ceq
  \frac{\bigl(\muh_t(s,a,h)-\mdqr_t(s,a,h)\bigr)^2}{2}
  +\log\frac{\maxp{\hat{p}_t(s'\mid s,a,h)}{\veps}}{\maxp{\mdqp_t(s'\mid s,a,h)}{\myexp{-t^\alpha}}},
\]
and note that $\lvert \Psi_t(s,a,s',h) \rvert \leq t^\alpha +\tfrac12 $. Since $\muh_t$ 
is $\cH_{t-1}$-measurable and $(\mdqr_t,\mdqp_t)$ depends only on $\cH_{t-1}$ plus additional independent randomization, we 
augment the filtration to $\wt{\cH}_{t-1} \ceq 
\cH_{t-1}\cup\{\mdqr_t,\mdqp_t\}$. Conditioned on $\widetilde{\mathcal{H}}_{t-1}$, the only source of randomness is the 
trajectory at episode $t$, generated under $\bpolmix_t$, so
\[
\bE\!\left[
  \sum_{h=1}^H\Psi_t(s_h^t,a_h^t,s_{h+1}^t,h)
  \,\middle|\,\widetilde{\mathcal{H}}_{t-1}
\right]
= \sum_{h=1}^H\sum_{s,a,s'} w_h^{\bpolmix_t}(s,a)\,
  p(s'\mid  s,a,h)\,\Psi_t(s,a,s',h).
\]
Therefore the process
\[
m_t \ceq \sum_{k=1}^t\left[
  \sum_{h=1}^H\Psi_k(s_h^k,a_h^k,s_{h+1}^k,h)
  -\sum_{h=1}^H\sum_{s,a,s'} w_h^{\bpolmix_k}(s,a)\,
    p(s'\mid s,a,h)\,\Psi_k(s,a,s',h)
\right]
\]
is a martingale with, increments bounded in absolute value by $H (t^\alpha + \tfrac12)$.
By Azuma--Hoeffding's maximal inequality, 
\[
\Pr\!\left( m_{T-1} > \sqrt{2\sum_{t=1}^T 
    H^2(t^\alpha + \tfrac12)^2 \log(T^2)}
\right) \leq \frac{1}{T^2}, 
\]
which gives 
the claimed event $\mathcal{E}_w(T)$.
\end{proof}

\begin{lemma}
\label{lem:control_ratio_prob}
Let $\varepsilon > 0$. With probability at least $1-O(1/T^2)$, 
\[
\left|\sum_{t=1}^{T-1}\sum_{h=1}^H
  \log\frac{\hat{p}_T(s_{h+1}^t\mid s_h^t,a_h^t,h)}
    {\maxp{\hat{p}_t(s_{h+1}^t\mid s_h^t,a_h^t,h)}{\veps }}
\right|
\leq O \!\lp T^{\alpha + \tfrac{1+ \gamma }{2}} \sqrt{\log T} / \veps \rp.
\]
\end{lemma}

\begin{proof}
Since $\hat{p}_T(s_{h+1}^t\mid s_h^t,a_h^t,h)\geq 1/T > 0$ (the 
transition is observed at episode $t\leq T$) and 
$\maxp{\hat{p}_t}{\varepsilon}\geq\varepsilon > 0$, every logarithm is 
finite. By the triangle inequality,
\[
\left|\sum_{t=1}^{T-1}\sum_{h=1}^H
  \log\frac{\hat{p}_T(s_{h+1}^t\mid s_h^t,a_h^t,h)}
    {\maxp{\hat{p}_t(s_{h+1}^t\mid s_h^t,a_h^t,h)}{\varepsilon}}
\right|
\leq \sum_{t=1}^{T-1}\sum_{h=1}^H
  \left|\log\frac{\hat{p}_T(s_{h+1}^t\mid s_h^t,a_h^t,h)}
    {\maxp{\hat{p}_t(s_{h+1}^t\mid s_h^t,a_h^t,h)}{\varepsilon}}\right|.
\]
For  $T_0 = o(T)$, the contribution from $t < T_0$ is ${o}(T)$ since $\maxp{\hat{p}_t}{\varepsilon}\geq\varepsilon$, 
giving $|\log(\hat{p}_T/\varepsilon)|\leq\log(1/\varepsilon)$ per term, times $T_0$.

Thanks to Lemma~\ref{lem:forced_exp}, there exists $k_0<\infty$, such that for $T\geq k_0$ 
for all $t\in (T^{2/3}, T)$ and all  $(s,a,h)$, if $\cE_{\sect}^p(T)$ holds, then 
\[
\|\hat{p}_t(\cdot\mid s,a,h)-p(\cdot\mid s,a,h)\|_1
  \leq O \! \lp  \sqrt{t^{\gamma-1} \beta^p(t^2, 1/t^3) }\rp\:.
\]

We assume $t\geq T^{2/3}$. We split into two cases.

\medskip\noindent
In the first case, assume $\hat{p}_t(s_{h+1}^t\mid s_h^t,a_h^t,h)\geq\varepsilon$.
Then $\maxp{\hat{p}_t}{\varepsilon}=\hat{p}_t$, and since $\gamma<\frac{1}{2}$
both $\hat{p}_T$ and $\hat{p}_t$ lie in 
$[(1-t^{\gamma-1/2})p_h,\,(1+t^{\gamma-1/2})p_h]$, so
\[
\left|\log\frac{\hat{p}_T}{\hat{p}_t}\right|
\leq \log\frac{1+ O \! \lp  \sqrt{t^{\gamma-1} \beta^p(t^2, 1/t^3) }\rp}{1- O \! \lp  \sqrt{t^{\gamma-1} \beta^p(t^2, 1/t^3) }\rp}
\leq O \! \lp  \sqrt{t^{\gamma-1} \beta^p(t^2, 1/t^3) }\rp
\]

\medskip\noindent
In the second case, assuming $\hat{p}_t(s_{h+1}^t\mid s_h^t,a_h^t,h)<\varepsilon$, 
then $\maxp{\hat{p}_t}{\varepsilon}=\varepsilon$. Since $\hat{p}_t < \varepsilon$,
the concentration bound gives
\[
p(s_{h+1}^t\mid s_h^t,a_h^t,h) 
  \leq \hat{p}_t(s_{h+1}^t\mid s_h^t,a_h^t,h) + O \! \lp  \sqrt{t^{\gamma-1} \beta^p(t^2, 1/t^3) }\rp 
  < \varepsilon + O \! \lp  \sqrt{t^{\gamma-1} \beta^p(t^2, 1/t^3) }\rp .
\]
Applying the concentration bound again to $\hat{p}_T$, yields 
\[
\hat{p}_T(s_{h+1}^t\mid s_h^t,a_h^t,h)
  \leq p(s_{h+1}^t\mid s_h^t,a_h^t,h) + O \! \lp  \sqrt{T^{\gamma-1} \beta^p(T^2, 1/T^3) }\rp
  \leq \varepsilon + O \! \lp  \sqrt{t^{\gamma-1} \beta^p(t^2, 1/t^3) }\rp.
\]
Therefore, using $\log(1+x)\leq x$,
\begin{align*}
&\log\frac{\hat{p}_T(s_{h+1}^t\mid s_h^t,a_h^t,h)}{\varepsilon}
  \leq \\& \log\frac{\varepsilon+O \! \lp  \sqrt{t^{\gamma-1} \beta^p(t^2, 1/t^3) }\rp}{\varepsilon}
  = \log\!\left(1+\frac{O \! \lp  \sqrt{t^{\gamma-1} \beta^p(t^2, 1/t^3) }\rp}{\varepsilon}\right)
  \leq \\ & \frac{O \! \lp  \sqrt{t^{\gamma-1} \beta^p(t^2, 1/t^3) }\rp}{\veps}
  = O \! \lp  \sqrt{t^{\gamma-1} \beta^p(t^2, 1/t^3) } / \veps \rp,
\end{align*}
where the last step uses that $\veps$ is a fixed positive constant.

\medskip\noindent
In both cases the contribution per $(t,h)$ is  $O \! \lp  \sqrt{t^{\gamma-1} \beta^p(t^2, 1/t^3) }\rp$. Summing gives 
\[
\sum_{t=T_0}^{T-1}\sum_{h=1}^H
  \left|\log\frac{\hat{p}_T}{\maxp{\hat{p}_t}{\varepsilon}}\right|
\leq O \!\lp T^{\alpha + \tfrac{1+ \gamma }{2}} \sqrt{\log T} / \veps \rp
\]
\end{proof}

\begin{lemma}
\label{lem:expected_to_random_loss}
The event
\begin{equation}
\label{eq:conv-event}
	\mathcal{E}_{\sect}^{\mathrm{conv}}(T) \ceq \left\{
  \sum_{t\in \lbrack T-1 \rbrack}
  \left[
    \ell_t^\alpha(\mdqr_t,\mdqp_t)
    - \bE_{(\mdqr,\mdqp)\sim\nu_t^{\eta_t}(\cdot\mid\alt(\empmdp_t))}
    \!\bigl[\ell_t^\alpha(\mdqr,\mdqp)\bigr]
  \right]
  \leq \sqrt{2\log(T^2)\sum_{t=1}^T C_{t,\alpha}^2}
\right\}
\end{equation}
holds with probability at least $1-1/T^2$, where 
$C_{t,\alpha} \ceq H\!\left(\tfrac{1}{2}+t^\alpha\right)$.
In particular, the right-hand side is $O\!\left(\sqrt{T^{1+2\alpha}\log T}\right)$,
which is $o(T)$ whenever $\alpha < \frac{1}{2}$.
\end{lemma}

\begin{proof}
Recalling the definition of $\ell_t^\alpha$ from \eqref{eq:def-trunc-loss},
its increments satisfy $\ell_t^\alpha(\mdqr,\mdqp)\in(0, C_{t,\alpha})$
for all $(\mdqr,\mdqp)$. Conditioned on $\mathcal{H}_{t-1}$, the only 
source of randomness in $\ell_t^\alpha(\mdqr_t,\mdqp_t)$ is the sample 
$(\mdqr_t,\mdqp_t)\sim\nu_t^{\eta_t}(\cdot\mid\alt(\empmdp_t))$, so
\[
\bE\!\left[\ell_t^\alpha(\mdqr_t,\mdqp_t)\mid\mathcal{H}_{t-1}\right]
= \bE_{(\mdqr,\mdqp)\sim\nu_t^{\eta_t}(\cdot\mid\alt(\empmdp_t))}
  \!\bigl[\ell_t^\alpha(\mdqr,\mdqp)\bigr].
\]
Therefore the process
\[
m_t \ceq \sum_{k=1}^{t}\left[
  \ell_k^\alpha(\mdqr_k,\mdqp_k)
  - \bE_{(\mdqr,\mdqp)\sim\nu_k^{\eta_k}(\cdot\mid\alt(\empmdp_k))}
  \!\bigl[\ell_k^\alpha(\mdqr,\mdqp)\bigr]
\right]
\]
is a martingale with increments bounded in $(-C_{t,\alpha}, C_{t,\alpha})$.
By Azuma--Hoeffding's maximal inequality,
\[
\Pr\!\left(m_{T-1} > \sqrt{2\log(T^2)\sum_{t=1}^T C_{t,\alpha}^2}\right)
\leq \frac{1}{T^2},
\]
which gives the claimed event. 
\end{proof}

The result below proves that the clipping of the transition probabilities on the sampled instance has a negligible impact on the loss of the regret minimizer. 

\begin{lemma} There exists a constant $k_0<\infty$ such that for all $t\geq k_0$ 
\label{lem:make_inf_appear}
\begin{align*}
\sum_{s,a, h} w_h(s,a) \lsb \frac{\bigl( \mdqr_t(s, a, h) - \mu(s,a, h)\bigr)^2}{2}   +  \sum_{s'} p(s'\mid s,a,h) \log \frac{p(s'\mid s,a,h)}{\maxp{\mdqp_t(s'\mid s,a, h)}{ \myexp{-t^{\alpha}}}}\rsb  \geq \\
\inf_{\mdq \ceq (\mdqr , \mdqp) \in \alt(\mdp)}\sum_{s,a, h} w_h(s,a) \lsb \frac{(\mdqr(s, a, h) - \mu(s,a, h))^2}{2}  +  \sum_{s'} p(s'\mid s,a,h) \log \frac{p(s'\mid s,a,h)}{\mdqp_h(s'\mid s,a, h) }\rsb - O(1/t) , 
\end{align*}
holds with probability $1$. 
\end{lemma}
\begin{proof}
	To ease notation, we introduce 
	$$ d^\alpha_{w}((\mdqr_t, \mdqp_t), (\mdpr, \mdpp)) \, \ceq \,\sum_{s,a,h} w_h(s,a) \lsb  \frac{(\mdqr_t(s, a, h) - \mu(s,a, h))^2}{2}   +  \sum_{s'} p(s'\mid s,a,h) \log \frac{p(s'\mid s,a,h)}{\maxp{\mdqp_t(s'\mid s,a, h)} {\myexp{-t^\alpha}}} \rsb \:.$$
	Since $\mdq_t \ceq (\mdqr_t, \mdqp_t) \in \alt(\mdp)$ there exists a deterministic policy $\tilde \pi \neq \pi^\star$ such that $\tilde \pi$ is a global optimal policy in $\mdq_t$.

We invoke Lemma~\ref{lem:build_easier}, which provides an MDP $\mdp' \ceq (\tilde p, \tilde \mu)$ for which $\tilde \pi$ is the only global optimal policy and there is a unique optimal action per stage and state. Assuming $\frac{1}{t}\leq \frac13 \frac{1}{1+2^{H-1}}$ i.e., $t \geq 3(1+ 2^{H-1})$ then we have by the construction in Lemma~\ref{lem:build_easier} with $\veps  = 1/t$, \begin{eqnarray*}
		\left \{  \begin{array}{ll}
		\tilde p_h(\cdot \mid s,a) &\ceq  (1-\gamma_{s,a,h}) \mdqp_t(\cdot \mid s,a, h) + \gamma_{s,a,h} \mdqp_t (\cdot \mid s,\pi_h(s), h) \\
		\lvert \tilde \mu(s,a, h) - \mdqr_t(s,a, h)\rvert  & \leq d_{s,a, h} \ceq 2^{H-h+1} / t
		\end{array} \right \}
	\end{eqnarray*}
	where  $\gamma_{s,a,h} \in \lsb 0, \frac{(1 + 2^{H-h}) \veps}{1 -  \lp 2^{H-h+1}  + 3 \rp	\veps} \rsb$. Therefore, we have 
	
\begin{eqnarray*}
	d^\alpha_{w}((\mut, \tilde p), (\mdpr, p)) &\leq& \sum_{s,a,h} w_h(s,a) \lsb  \frac{(\mut(s, a, h) - \mu(s,a, h))^2}{2} +  \sum_{s'} p(s'\mid s,a,h) \log \frac{p(s'\mid s,a,h)}{\maxp{(1-\gamma_h)\mdqp_t(s'\mid s,a, h)} {\myexp{-t^{\alpha}}}} \rsb \:\\
	&\leq&  d^\alpha_{w}((\mdqr_t, \mdqp_t), (\mdpr, \mdpp)) + \sum_{s,a,h} w_h(s,a)  \lsb \frac{4^{H-h+1}}{2t^2 } + \frac{2^{H-h+1}}{t}\log \frac{1}{1- \gamma_h } \rsb \\
	&\leq& d^\alpha_{w}((\mdqr_t, \mdqp_t), (\mdpr, \mdpp)) + \frac{4^H}{2t^2} + \frac{2^H}{t} + \frac{\gamma_1}{1- \gamma_1} \quad \text{(by monotonicity of $h \mapsto \gamma_h$ and $\log (1+x) \leq x$)} \\
	&\leq& d^\alpha_{w}((\mdqr_t, \mdqp_t), (\mdpr, \mdpp))  + O(1/t). 
\end{eqnarray*}

\medskip\noindent

Now we invoke Lemma~\ref{lem:incr-trans} to build an instance $M'_\gamma \ceq (\mut, \\tilde p^\gamma)$ close to $(\mut, \tilde p)$, but for which $\tilde \pi$ is still the unique optimal policy and such that each transition probability is larger than some $\gamma$. We invoke Lemma~\ref{lem:incr-trans} with $\gamma = 1/t$. This instance is obtained by a mixture of $(\mut, \tilde p)$ with $(\mut, u)$ with weights $1-\gamma, \gamma$, where $u$ is the uniform transition kernel. 

We then have thanks to $\log(1+x) \leq x$, 
\begin{eqnarray*}
	d^\alpha_{w}((\mut, \tilde p^\gamma), (\mdpr, \mdpp)) &\leq& \sum_{s,a,h} w_h(s,a) \lsb  \frac{(\mut(s, a, h) - \mu(s,a, h))^2}{2} +  \sum_{s'} p(s'\mid s,a,h) \log \frac{p(s'\mid s,a,h)}{\maxp{(1-\gamma) \tilde  p(s'\mid s,a, h) }{ \myexp{-t^{\alpha}}}} \rsb \: \\
	&\leq& d^\alpha_{w}((\tilde \mdpr, \tilde \mdpp), (\mdpr, \mdpp)) + \frac{\gamma}{1-\gamma}  \\
	&=& d^\alpha_{w}((\tilde \mdpr, \tilde \mdpp ), (\mdpr, \mdpp)) + \frac{1/t}{1-1/t} \\
	&=& d^\alpha_{w}((\mut, \tilde p), (\mdpr, \mdpp)) + O(1/t)\:.
\end{eqnarray*}
Thus, combining the two displays above, we have 
\begin{eqnarray*}
	d^\alpha_{w}((\mdqr_t, \mdqp_t), (\mu, p)) \geq  d^\alpha_{w}((\mut, \tilde  p^\gamma), (\mu, p)) - O(1/t). 
\end{eqnarray*}
However, the transition probabilities in $\tilde p^\gamma$ are larger than $\gamma/S$ since $\tilde p^\gamma$ is the mixture with the uniform transition kernel with weight $\gamma$. Thus, assuming $t$ is such that  $1/(t S)> \myexp{-t^\alpha}$, we have for all $(s,a,h,s')$
$$ \tilde p^\gamma(s' \mid s,a, h) \geq  \frac{1}{t S} \geq \myexp{-t^\alpha}$$
Therefore, 
\begin{eqnarray*}
	d^\alpha_{w}((\tilde \mu, \tilde p^\gamma), (\mu, p)) &=& \sum_{s,a,h} w_h(s,a) \lsb  \frac{(\mdqr_h(s, a) - \mu(s,a, h))^2}{2} +  \sum_{s'} p(s'\mid s,a,h) \log \frac{p(s'\mid s,a,h)}{\tilde p^\gamma(s' \mid s,a, h)} \rsb\:. 
\end{eqnarray*}

Thus, since Lemma~\ref{lem:incr-trans} ensures that $\tilde \pi \neq \pi^\star$ is the unique  optimal policy for each stage and state, we have 
$ (\tilde\mu, \tilde p^\gamma) \in \alt(\mdp)$. Putting the above results together, we have 
\begin{eqnarray*}
d^\alpha_{w}((\mdqr_t, \mdqp_t), (\mdpr, \mdpp)) &\geq&	\sum_{s,a,h} w_h(s,a) \lsb  \frac{(\mdqr(s, a, h) - \mu(s,a, h))^2}{2} +  \sum_{s'} p(s'\mid s,a,h) \log \frac{p(s'\mid s,a,h)}{\tilde p^\gamma(s'\mid s,a, h)} \rsb - O(1/t) \\
&\geq& \inf_{(\mdqr, \mdqp) \in \alt(\mdp)}\sum_{s,a, h} w_h(s,a) \lsb \frac{(\mdqr(s, a, h) - \mu(s,a, h))^2}{2}  +  \sum_{s'} p(s'\mid s,a,h) \log \frac{p(s'\mid s,a,h)}{\mdqp(s'\mid s,a, h) }\rsb - O(1/t). 
\end{eqnarray*}
 \end{proof}

\section{POSTERIOR CONVERGENCE}
 \label{appx:posterior_convergence}

 In this section, we prove the results related to the posterior contraction rate. First, we prove the lower bound on the posterior anti-concentration rate. 
 More precisely, we prove the result below. 

 \postLbdMain* 

Let $\lp \bpol_t\rp_{t\geq 1}$ be the sequence of exploration policies used by an adaptive algorithm, i.e.,  at episode $t$, the learner selects a policy $\bpol_t$ depending on past observations summarized in $\cH_{t-1}$.  

We that recall  
\begin{equation}
\label{eq:appx-def-glr}
	\Upsilon_t(\mdqr, \mdqp)  \ceq   \sum_{s,a,h} n_t(s,a,h) \lsb \frac{\bigl(\muh_t(s,a, h) - \mdqr(s,a, h)\bigr )^2}{2} + \sum_{s'} \hat p_t(s'\mid s,a,h) \log \frac{\hat p_t(s'\mid s,a,h)}{\mdqp(s'\mid s,a,h)}\rsb \:,
\end{equation}
 is the log-likelihood ratio between the empirical MDP $\empmdp_t \ceq (\hat\mu_t, \hat p_t)$ and the MDP  $\mdq \ceq (\mdqr, \mdqp) \in \Mdp$.
 
 \subsection{Asymptotic Limit of Likelihood Ratio}
 
 Before proving the result above, we show the following intermediate proposition. 
 \begin{proposition}
\label{prop:upbd-glr}
For any adaptive algorithm for best policy identification, with probability 
one,
\[
\limsup_{t\to\infty}\;\frac{1}{t} \lsb 
  \inf_{(\mdqr,\mdqp)\in\alt(\mdp)}\Upsilon_t(\mdqr,\mdqp)\rsb 
  \leq \Gamma_\mdp.
\]
\end{proposition}
 \begin{proof}

 First, let us justify that the empirical frequency vector $(\tfrac{n_t(s, a, h)}{t-1})_{s, a,h}$ is an almost valid state--action allocation for $\mdp$. 
 
 For this, let us define the $t$-indexed stochastic process 
 \begin{eqnarray*}
 	m_t({s,a, h}) &\ceq&  \sum_{k=1}^t \lsb  \indp{ s_h^k =s, a_h^k=a } - w_h^{\bpol_k}(s,a ) \rsb, \\
 	&=& n_{t+1}(s,a, h ) - \sum_{k\leq t} w_h^{\bpol_k}(s,a )\:, 
 \end{eqnarray*} where $w^{\bpol_t}$ is the state--action visitation measure induced by $\bpol_t$ on $\mdp$. Note that $(m_t({s,a,h}))_t$ is a martingale and its increments are bounded in $(-1,1)$. Thanks to Azuma-Hoeffding, the event 
 
\begin{equation}
\label{eq:good-evt-op-xc}
\cE_{\sect}^{s,a,h}(T) \ceq \lb  \forall \; t \; \leq  T, \Bigl | n_{t+1}(s,a,h) - \sum_{k\leq t} w_h^{\bpol_k}(s,a ) \Bigr | \leq \sqrt{2T \log(2T^2 SAH)}\rb 	
\end{equation}
holds with probability at least $1-1/(SAH T^2)$ so that 
\begin{equation}
\cE_{\sect}(T) \ceq \bigcap_{s,a,h} \cE_{\sect}^{ s,a,h}(T),  
\end{equation}
holds with probability at least $1-1/T^2$. Therefore, thanks to Borell-Cantelli's lemma, there exits $\wt {T}$ such that with probability 1, for $T\geq \wt {T}, \cE_{\sect}( T)$ holds which implies that for any $s,a,h$ and $T\geq \wt T$, 
\begin{equation}
	\Bigl | \tfrac{n_{T+1}(s,a,h)}{T} -  \frac1T \sum_{t\leq T} w_h^{\bpol_t}(s,a)  \Bigr | \leq \sqrt{\frac{2 \log(2T^2 SAH)}{T}}\:, 
\end{equation}
and further remark that, as the space of valid-action allocation is convex (see e.g.,\citet{wagenmaker22b}), $\frac1T \sum_{t\in \lbrack T \rbrack} w_h^{\bpol_t}$ is a valid state-action allocation for $\mdp$.  

\bigskip 

Since transition probabilities may be arbitrarily small, the KL divergence in $\Upsilon_t$ is potentially unbounded. To handle this, we restrict to alternatives with bounded transitions. Concretely, observe that the model obtained by setting all transitions to uniform and all rewards to zero on states visited by the optimal policy of $\mdp$ lies in $\alt(\mdp)$, as  $\mdp$ is trivially absolutely continuous with respect to it. Therefore, for any $\varepsilon > \log S$, the restricted alternative set
\[
\alt(\mdp,\varepsilon) \ceq \left\{
  (\mdqr,\mdqp)\in\alt(\mdp) : \forall\,(s,a,h,s'),\;
  \mdqp(s'\mid s,a,h)\geq \myexp{-\varepsilon}
\right\}
\]
is non-empty. 
\medskip\noindent
Since $\alt(\mdp,\varepsilon)\subset\alt(\mdp)$,
\begin{align}
\label{eq:glr-restricted}
\begin{split}
 &\inf_{(\mdqr,\mdqp)\in\alt(\mdp)}\Upsilon_t(\mdqr,\mdqp)
  \leq \\ & \inf_{(\mdqr,\mdqp)\in\alt(\mdp,T^{1/4})}
    \sum_{s,a,h} n_T(s,a,h)\left[
      \frac{\bigl(\muh_T(s,a,h)-\mdqr(s,a,h)\bigr)^2}{2}
      +\sum_{s'}\hat{p}_T(s'\mid s,a,h)
       \log\frac{\hat{p}_T(s'\mid s,a,h)}{\mdqp(s'\mid s,a,h)}
    \right].
    \end{split}
\end{align}

\medskip 
\noindent
By definition of $\alt(\mdp,T^{1/4})$, every $(\mdqr,\mdqp)$ in this set satisfies
\begin{equation}
\label{eq:bounded-log}
\forall\,(s,a,h,s'):\quad
\log\frac{1}{\mdqp(s'\mid s,a,h)} \leq T^{1/4},
\end{equation}
which allows concentration arguments to be applied in the next step.

\medskip\noindent

We define the event $\cE_{\sect}^{\mu+p}(T) \ceq \cE_{\sect}^\mu(T)\cap\cE_{\sect}^p(T)$, where
\begin{align}
\cE_{\sect}^p(T) &\ceq \left\{\forall\,t\leq T:\;
  \sum_{s,a,h} n_t(s,a,h)
  \KL\!\left(\hat{p}_t(\cdot\mid s,a,h)\,\|\,p(\cdot\mid s,a,h)\right)
  \leq\beta^p(T,1/T^2)\right\}, \\
\cE_{\sect}^\mu(T) &\ceq \left\{\forall\,t\leq T:\;
  \sum_{s,a,h} n_t(s,a,h)
  \KL\!\left(\hat{R}_h^t(s,a)\,\|\,R_h(s,a)\right)
  \leq\beta^\mu(T,1/T^2)\right\}. 
\end{align}

\medskip On $\mathcal{E}^\mu(T)$ we have,
\begin{align}
\label{eq:reward-dev}
	\begin{split}
\frac{\bigl(\muh_T(s,a,h)-\mdqr(s,a,h)\bigr)^2}{2} &\leq \frac{(\mu(s,a, h) - \mdqr(s,a, h))^2}{2} + \frac{(\mu(s,a,h) - \muh_T(s,a, h))^2}{2} + \llvert \mu(s,a, h) - \muh_T(s,a, h) \rrvert
\\
  &\leq \frac{\bigl(\mu(s,a,h)-\mdqr(s,a,h)\bigr)^2}{2}
  + \frac{\beta^\mu(T,1/T^2)}{n_T(s,a,h)}
  + \sqrt{\frac{2\beta^\mu(T,1/T^2)}{n_T(s,a,h)}}.
	\end{split}
\end{align}

\medskip\noindent

Next, for any $q\in\triangle$, we let $\cS_T^+\ceq\{s':\hat{p}_T(s'\mid s,a,h)>0\}$.
For $s'\notin\cS_T^+$, the convention $0\log 0=0$ means the left-hand side contributes $0$. By convexity of $x\mapsto x\log{x}$,

\begin{align*}
&\hat{p}_T(s' \mid s,a,h)\log\frac{\hat{p}_T(s'\mid s,a,h)}{q(s')}
  - p(s'\mid s,a,h)\log\frac{p(s'\mid s,a,h)}{q(s')}
  \leq \\ & \left(1+\log\frac{\hat{p}_T(s'\mid s,a,h)}{q(s')}\right)
    \bigl(\hat{p}_T(s'\mid s,a,h)-p(s' \mid s,a,h)\bigr),
    \quad s'\in\cS_T^+.	
\end{align*}

Summing over $s'\in\cS_T^+$ yields  
\begin{align}
\label{eq:conv-trans}
\begin{split}
& \sum_{s'}\hat{p}_T(s' \mid s,a,h)\log\frac{\hat{p}_T(s'\mid s,a,h)}{q(s')}
  \leq \\ & \sum_{s'} p(s' \mid s,a,h)\log\frac{p(s'\mid s,a,h)}{q(s')}
  + \sum_{s'\in\cS_T^+}
    \left(1+\log\frac{\hat{p}_T(s'\mid s,a,h)}{q(s')}\right)
    \bigl(\hat{p}_T(s'\mid s,a,h)-p(s'\mid s,a,h)\bigr).
    \end{split}
\end{align}
We note that for $s'\in\cS_T^+$, either the transition $s'$ has been observed at least once from $(s,a,h)$, or $(s,a,h)$ has never been visited and $\hat p_T(s' \mid s,a,h) = 1/S$, by initialization. In all cases  $\hat{p}_T(s'\mid s,a,h)\geq \min (1/T, 1/S)$.
Therefore, 
$$
\left|1+\log\frac{\hat{p}_T(s'\mid s,a,h)}{q(s')}\right| \leq  1+\log(ST).
$$ 
Substituting into \eqref{eq:conv-trans} and using the triangle inequality,
\begin{equation}
\label{eq:trans-dev-final}
\sum_{s'}\hat{p}_T(s'\mid s,a,h)\log\frac{\hat{p}_T(s'\mid s,a,h)}{q(s')}
  \leq \sum_{s'} p(s'\mid s,a,h)\log\frac{p(s'\mid s,a,h)}{q(s')}
  + \bigl(1+\log(ST)\bigr)
    \left\|\hat{p}_T(\cdot\mid s,a,h)-p(\cdot\mid s,a,h)\right\|_1.
\end{equation}

Applying Pinsker's inequality on $\cE_{\sect}^p(T)$,
\begin{equation}
\label{eq:trans-dev}
\sum_{s'}\hat{p}_T(s' \mid s,a,h)\log\frac{\hat{p}_T(s' \mid s,a,h)}{q(s')}
  \leq \sum_{s'} p(s' \mid s,a,h)\log\frac{p(s' \mid s,a,h)}{q(s')}
  + \bigl(1+\log(ST)\bigr)
    \sqrt{\frac{2\beta^p(T,1/T^2)}{n_T(s,a,h)}}.
\end{equation}

\medskip\noindent
Substituting \eqref{eq:reward-dev} and 
\eqref{eq:trans-dev} into \eqref{eq:glr-restricted}, applying 
Cauchy--Schwarz to the error terms, and using 
$\sum_{s,a,h}n_T(s,a,h)=  TH$,
\begin{align}
\label{eq:glr-conc}
\begin{split}
&\inf_{(\mdqr,\mdqp)\in\alt(\mdp)}\Upsilon_t(\mdqr,\mdqp) \leq \\ 
  &\inf_{(\mdqr,\mdqp)\in\alt(\mdp,T^{1/4})}
    \sum_{s,a,h} n_T(s,a,h)\left[
      \frac{\bigl(\mu(s,a,h)-\mdqr(s,a,h)\bigr)^2}{2}
      +\sum_{s'} p(s'\mid s,a,h)\log\frac{p(s'\mid s,a,h)}{\mdqp(s'\mid s,a,h)}
    \right]\;  +  \\
  &\quad  SAH \beta^\mu(T, 1/T^2) + \sqrt{2SAH^2T\,\beta^\mu(T,1/T^2)} 
    \; + \; \bigl(1+\log(ST)\bigr) \sqrt{2SAH^2T\,\beta^p(T,1/T^2)}.
  \end{split}
\end{align}
Since $\beta^\mu$ and $\beta^p$ are at most logarithmic in $T$, the three error terms in \eqref{eq:glr-conc} are all $o(T)$.

\medskip\noindent

\bigskip 
We then invoke \myeqref{eq:good-evt-op-xc}, which proves that on the event $\cE_{\sect}(T) $ 
we have 
\begin{align*}
	\inf_{(\mdqr, \mdqp) \in \alt(\mdp, T^{1/4})} &\sum_{s,a,h} n_T(s,a, h) \lsb \frac{(\mu(s,a, h) - \mdqr(s,a,h))^2}{2} + \sum_{s'} p(s'\mid s,a,h) \log \frac{ p(s'\mid s,a,h)}{\mdqp(s'\mid s,a,h)}\rsb \leq  \\
	& \inf_{(\mdqr, \mdqp) \in \alt(\mdp, T^{1/4})}\sum_{s,a,h} \sum_{t\in \lbrack T \rbrack} w^{\bpol_t}_h(s,a) \lsb \frac{(\mu(s,a, h) - \mdqr(s,a, h))^2}{2} + \sum_{s'} p(s'\mid s,a,h) \log \frac{p(s'\mid s,a,h)}{\mdqp(s'\mid s,a,h)}\rsb \\&+ SAH \sqrt{2T \log(2T^2 SAH)} \bigl ( \frac12 + T^{1/4}\bigr), 
\end{align*}
\medskip\noindent
Combining with \myeqref{eq:glr-conc} yields  
\begin{align*}
&\inf_{(\mdqr,\mdqp)\in\alt(\mdp)}\Upsilon_t(\mdqr,\mdqp) \leq  \\& 	\inf_{(\mdqr, \mdqp) \in \alt(\mdp, T^{1/4})}\sum_{s,a,h} \sum_{t\in \lbrack T -1 \rbrack} w^{\bpol_t}_h(s,a) \lsb \frac{(\mu_h(s,a) - \mdqr_h(s,a))^2}{2} + \sum_{s'} p(s'\mid s,a,h) \log \frac{ p(s'\mid s,a,h)}{\mdqp(s' \mid s,a,h)}\rsb + \notag\\& \lsb (1+ \log(ST))  \sqrt{ {2 SAH^2T  \beta^p(T, 1/T^2)}} \rsb \notag\\& +\lsb { SAH \beta^\mu(T, 1/T^2)} + \sqrt{{ 2SAH^2T \beta^\mu(T, 1/T^2)}}\rsb + SAH \sqrt{2T \log(2T^2 SAH)} \bigl( \frac12 + T^{1/4}\bigr) \:. 
\end{align*}
\medskip 
\noindent 
Next, noting that $\beta^\mu, \beta^p$'s dependency on $T$ is at most logarithmic (see e.g. \citet{pmlr-v201-tirinzoni23a}), the second order term in the equation above is $o(T)$, which we rewrite as: on the event $\cE_\sect(T)$, 

\begin{align*}
	&\inf_{(\mdqr,\mdqp)\in\alt(\mdp)}\Upsilon_t(\mdqr,\mdqp) \leq  \\ &\inf_{(\mdqr, \mdqp) \in \alt(\mdp, T^{1/4})}\sum_{s,a,h} \sum_{t\in \lbrack T-1\rbrack} w^{\bpol_t}_h(s,a) \lsb \frac{\bigl (\mu(s,a, h) - \mdqr(s,a, h)\bigr)^2}{2} + \sum_{s'} p(s'\mid s,a,h) \log \frac{ p(s'\mid s,a,h)}{\mdqp(s'\mid s,a,h)}\rsb   \; + \; o(T)\\
	&=  (T-1) \cdot \inf_{(\mdqr, \mdqp) \in \alt(\mdp, T^{1/4})}\sum_{s,a,h} \frac1{T-1} \sum_{t\in \lbrack T-1 \rbrack} w^{\bpol_t}_h(s,a) \lsb \frac{\bigl(\mu(s,a, h) - \mdqr(s,a, h)\bigr)^2}{2} + \right. \\& \left. \sum_{s'} p(s'\mid s,a,h) \log \frac{ p(s'\mid s,a,h)}{\mdqp_h(s,a,s')}\rsb + o(T) \\
	&\leq (T-1) \; \cdot \sup_{w \in \Omega_\mdp} \inf_{(\mdqr, \mdqp) \in \alt(\mdp, T^{1/4})}\sum_{s,a,h}w_h(s,a) \lsb \frac{(\mu(s,a, h) - \mdqr(s,a, h))^2}{2} + \sum_{s'} p(s'\mid s,a,h) \log \frac{ p(s'\mid s,a,h)}{\mdqp(s'\mid s,a,h)}\rsb  + o(T)
\end{align*}
which follows as, from the convexity of $\Omega_\mdp$ we have $\frac1{T-1} \sum_{t=1}^{T-1} w_{\cdot}^{\bpol_t}\in \Omega_\mdp$. 
\medskip
\noindent
Finally, remark that the sequence $(u_T)_{T\geq 1}$ defined by 
$$ u_T \,\ceq\, \sup_{w \in \Omega_\mdp} \inf_{(\mdqr, \mdqp) \in \alt(\mdp, T^{1/4})}\sum_{s,a,h}w_h(s,a) \lsb \frac{\bigl(\mu(s,a, h) - \mdqr(s,a, h)\bigr)^2}{2} + \sum_{s'} p(s' \mid s,a,h) \log \frac{ p(s'\mid s,a,h)}{\mdqp(s'\mid s,a,h)}\rsb $$
is non-increasing and $\inf\lb u_T : T\geq 1\rb = \Gamma_\mdp$. Thus, noting  $\bP(\cE_{\sect}(T) \cap \cE^{\mu +p}_{\sect}(T)) \geq 1- O(1/T^2)$, thanks to Borel-Cantelli's lemma, we have with probability one,  
\begin{equation*}
	\limsup_{T\to \infty} \frac1T \lsb \inf_{(\mdqr,\mdqp)\in\alt(\mdp)}\Upsilon_T(\mdqr,\mdqp)  \rsb  \leq \Gamma_\mdp. 
\end{equation*}
\end{proof} 

\subsection{Proof of Theorem~\ref{thm:opt_post_contra}}

\begin{proof}
The result follows from Proposition~\ref{prop:upbd-glr}, 
Lemma~\ref{lem:utl_lower_bound}, and an anti-concentration bound on 
the posterior. We proceed in two steps: first expressing the posterior 
probability of $\alt(\mdp)$ in terms of the GLR, then lower-bounding it.

\medskip\noindent
On can  verify that the posterior probability of $\alt(\mdp)$ is proportional to
\begin{align}
\label{eq:post-alt}
\Pr_{\mdq\sim\nu_t\mid\cH_{t-1}}
  \!\left(\mdq \in\alt(\mdp)\right)
  \propto \int_{\alt(\mdp)}\myexp{ 
    - \Upsilon_t(\mdqr, \mdqp) } \diff \mdqr \,\diff \mdqp.
\end{align}

Defining
\begin{equation}
\label{eq:def-Lambda}
\Lambda_t \ceq \int_{\alt(\mdp)}\myexp{ 
    - \Upsilon_t(\mdqr, \mdqp) } \diff \mdqr \,\diff \mdqp ,
\end{equation}
we obtain
\begin{equation}
\label{eq:post-ratio}
\Pr_{\mdq \sim\nu_{t}\mid\cH_{t-1}}
  \!\left(\mdq \in\alt(\mdp)\right)
  = \frac{\Lambda_t}
    {\displaystyle\int_{\Mdp}
      \myexp{-\Upsilon_t(\mdqr,\mdqp)}\diff \mdqr \,\diff \mdqp}.
\end{equation}

\medskip\noindent

Since $\Upsilon_t$ is non-negative (it is a sum of KL divergences and 
squared differences) and $\Mdp $ is bounded,
\[
\int_{\Mdp}\myexp{-\Upsilon_t(\mdqr,\mdqp)}\diff \mdqr \,\diff \mdqp
  \leq \vol(\Mdp).
\]
Substituting into \eqref{eq:post-ratio},
\begin{equation}
\label{eq:post-lb}
\Pr_{\mdq \sim\nu_{t}\mid\cH_{t-1}}
  \!\left(\mdq \in\alt(\mdp)\right)
  \geq \frac{\Lambda_t}{\vol(\Mdp)}.
\end{equation}

\medskip
\noindent

By Lemma~\ref{lem:utl_lower_bound} (with $\eta=1$), with probability one,
\begin{align}
\log\Lambda_t
  &\geq -\inf_{(\mdqr,\mdqp)\in\alt(\mdp)}
    \sum_{k=1}^{t-1}\sum_{h=1}^H\left[
      \frac{\bigl(\muh_t(s_h^k,a_h^k,h)-\mdqr(s_h^k,a_h^k,h)\bigr)^2}{2}
      +\log\frac{\hat{p}_t(s_{h+1}^k\mid s_h^k,a_h^k,h)}
               {\mdqp(s_{h+1}^k \mid s_h^k,a_h^k,h)}
    \right] - o(t) \nonumber\\
  &= -{
      \inf_{(\mdqr,\mdqp)\in\alt(\mdp)}\sum_{s,a,h} n_t(s,a,h)
      \left[
        \frac{\bigl(\muh_t(s,a,h)-\mdqr(s,a,h)\bigr)^2}{2}
        +\sum_{s'}\hat{p}_t(s'\mid s,a,h)
         \log\frac{\hat{p}_t(s'\mid s,a,h)}{\mdqp(s'\mid s,a,h)}
      \right]
    }
    - o(t).
  \label{eq:Lambda-lb}
\end{align}

\medskip\noindent
Combining \eqref{eq:post-lb} and \eqref{eq:Lambda-lb},
\[
-\log\Pr_{\mdq\sim\nu_{t}\mid\cH_t}
  \!\left(\mdq \in\alt(\mdp)\right)
  \leq \inf_{(\mdqr,\mdqp)\in\alt(\mdp)}\Upsilon_t(\mdqr,\mdqp)  + \log\vol(\Mdp) - o(t).
\]
Since $\log\vol(\Mdp)= O(1)$, dividing by $t$ and taking 
the $\limsup$, Proposition~\ref{prop:upbd-glr} gives, with probability one,
\[
\limsup_{t\to\infty}
  -\frac{1}{t}\log\Pr_{\mdq \sim\nu_t\mid\cH_{t-1}}
  \!\left(\mdq \in\alt(\mdp)\right)
  \leq \limsup_{t\to\infty}\frac{1}{t}\, \lsb \inf_{(\mdqr,\mdqp)\in\alt(\mdp)}\Upsilon_t(\mdqr,\mdqp) \rsb 
  \leq \Gamma_\mdp,
\]
which completes the proof of Theorem~\ref{thm:opt_post_contra}.
\end{proof}

We now prove the main result of this section, showing that\algoname{} reaches the optimal posterior contraction rate with probability one. 
\postMain*
\begin{proof}
	Thanks to Proposition~\ref{prop:saddle-point},  with probability at least $1-O(1/t^2)$, for $t$ large enough  
	\begin{eqnarray*} 
	\glr(t) \geq t \cdot \Gamma_\mdp - O(\log t)- O(t^{1+\alpha - \gamma}) - o(t) - O(t^{ \alpha+\frac{\gamma +1}2} \sqrt{\log t}). 	
	\end{eqnarray*}
	
	\medskip\noindent

	Thanks to Lemma~\ref{lem:concentr} we have for
	$$ \bP_{\wt \mdp \sim \nu_{t} \mid \cH_{t-1}}(\wt \mdp \in \alt(\mdp)) \leq h(t) \myexp{- \glr(t))}, $$
	where $h$ is polynomial in $t$, where due to Lemma~\ref{lem:forced_exp}, we will have $\alt(\empmdp_t) = \alt(\mdp)$ (up to a zero measure set), for $t$ large enough. 
	 
	Therefore, with probability $1-O(1/t^2)$,
	$$ \log \bP_{\wt \mdp \sim \nu_{t} \mid \cH_{t-1}}(\wt \mdp \in \alt(\mdp)) \leq -t \Gamma_\mdp  + \log h(t) +  O(\log t)+ O(t^{1+\alpha - \gamma}) + o(t) + O(t^{ \alpha+\frac{\gamma +1}2} \sqrt{\log t}), $$
then, thanks to Borel-Cantelli's lemma, with probability one, 
	$$\liminf_{t\to \infty}-\frac 1t \log \bP_{\wt \mdp \sim \nu_{t} \mid \cH_{t-1}}(\wt \mdp \in \alt(\mdp)) \geq \Gamma_\mdp, $$
	which combined with Theorem~\ref{thm:opt_post_contra} completes the proof. 
	\end{proof}

\section{SAMPLE COMPLEXITY GUARANTEES}
\label{appx:sample_complexity}

In this section, we establish sample complexity guarantees regardless of the algorithm's correctness for the specified threshold, showing that when $B(t,\delta)$ is correctly calibrated, the posterior sampling rule combined with the sampling rule of \algoname{} attains the optimal expected sample complexity.  

\scMain*
\begin{proof} Let $(\cE(t))_{t\geq 1}$ be some events that will be explicit below. 

	We have  
	\begin{eqnarray*}
		\bE_\mdp[\tau_\delta] &\leq& 1 + \sum_{t=1}^\infty \bP_\mdp(\tau_\delta>t)\:, \\
		&\leq& 1 + \sum_{t=1}^\infty \bP_\mdp(\exists \;k\,\leq\, B(t, \delta) : \hat \pi_t  \notin \Pi^\star(\wt\mdp_{t, k}) )\:,\\
		&\leq& 1 + \sum_{t=1}^\infty \bE_\mdp \lsb \indp{ \cE(t) } \bP_\mdp(\exists \; k\, \leq \, B(t, \delta) : \hat \pi_t  \notin \Pi^\star(\wt\mdp_{t, k}) | \cH_{t-1}) \rsb + \sum_{t=1}^\infty \bP_\mdp(\cE(t)^c)\:,\\ 
		&\leq& 1 +  \sum_{t=1}^\infty \bP_\mdp(\cE(t)^c) +\sum_{t=1}^\infty \bE_\mdp \lsb B(t, \delta)\indp{\cE(t)} \bP_{\wt \mdp \sim \nu^{\eta_t}_t|\cH_{t-1}}\!\lp \hat \pi_t  \notin \Pi^\star(\wt\mdp)\rp \rsb,
		\end{eqnarray*}
		which follows since conditionally on $\cH_{t-1}$, $\wt\mdp_{t, 1},\wt\mdp_{t, 2}, \dots $ are i.i.d. samples from distribution with conditional density $\nu^{\eta_t}_t(\cdot)$.  
		\medskip \noindent
		
		Thus taking $k_0$ such that $\hat \pi_t = \pi^\star$ for any $t\geq k_0$, we have 
		\begin{eqnarray*}
		\bE_\mdp[\tau_\delta]  &\leq& 1 + k_0 + \sum_{t=1}^\infty \bP_\mdp(\cE(t)^c) +\sum_{t=1}^\infty \bE \lsb B(t, \delta)\bP_{\wt \mdp \sim \nu^{\eta_t}_t |\cH_{t-1}}\lp \wt\mdp \in \alt(\mdp)  \rp \rsb	\:. \end{eqnarray*}
		\medskip
		\noindent 
		Next, thanks to Lemma~\ref{lem:concentr} we have 
		$$ \bP_{\wt \mdp \sim \nu^{\eta_t}_t | \cH_{t-1}}(\wt\mdp \in \alt(\mdp)  ) \leq f(t) \myexp{ - \eta_{t} \glr(t) }, $$ where $f$ is at most polynomial in $t$. 
		
	By Proposition~\ref{prop:saddle-point}, there exists an event $\cE(t)$, which holds with probability $1-O(1/t^2)$ such that we have 
		$$ \glr(t) \geq \Gamma_\mdp \cdot  t - h(t)$$ where $h$ is a deterministic sublinear function of $t$ satisfying $h(t) = o(t)$.  
		
		\medskip
		\noindent
		Combining the above results, we have 
		$$ \indp {\cE(t)}\bP_{\wt \mdp \sim \nu^{\eta_t}_t | \cH_{t-1}}(\wt\mdp \in \alt(\mdp)  ) \leq f(t) \exp \lp - \Gamma_\mdp t \eta_t +  \eta_t h(t) \rp.$$
	Thus we define 
	$$ T(\delta) \ceq \sup\lb t\geq 1 : { -\log(B(t, \delta)) - \log f(t) + \eta_t t \cdot  \Gamma_\mdp } - \eta_t h(t) \leq \log(t\log(t)^{1+\sigma})\rb\:,$$ so we have 
	\begin{eqnarray}
	\label{eq:zz-oo}
		\bE_\mdp[\tau_\delta] 
		&\leq& T_0 + T(\delta) +   \sum_{t>2} \frac{1}{t\log(t)^{1+\sigma}},\end{eqnarray} 
		where $T_0 <\infty$ is a constant. 
		
	We note  that $$ \sum_{t>2} \frac{1}{t\log(t)^{1+\sigma}} < \infty$$
	for $\sigma>0$.
	\medskip
	\noindent
	
Therefore, it remains to control the quantity $T(\delta)$ and for this remark that 
\begin{eqnarray*}
	T(\delta) &=& \sup \lb t \geq 1 \middle | t\Gamma_\mdp \eta_t    - \log(B(t, \delta)) - \log f(t)  - \eta_t h(t) \leq \log(t\log(t)^{1+\sigma} ) \rb, \\
	&=& \sup \lb t\geq 1 \middle |  t \leq    \Gamma_\mdp^{-1} \lp \frac{1}{\eta_t}\log(B(t, \delta)) + \frac{1}{\eta_t}\log f(t)  + h(t) - \frac{1}{\eta_t}\log(t\log(t)^{1+\sigma} ) \rp \rb, 
\end{eqnarray*}
Now we remark in the above the leading term in $\delta$ dependency is the term in $\frac{1}{\eta_t}\log(B(t, \delta))$. 
\end{proof}

\paragraph{Bounding $T(\delta)$}
First note that $f$ is at most polynomial and $h(t)$ is sublinear in $t$, as well as $\frac{1}{\eta_t}$ is sublinear in $t$. Thus, $$ q(t) \;\ceq\; \frac{1}{\eta_t}\log B(t,\delta) \;+\; \frac{1}{\eta_t}\log f(t) \;+\; h(t) \;-\; \frac{1}{\eta_t}\log\!\bigl(t\log (t)^{1+\sigma}\bigr) $$ 
is sublinear in $t$. Hence, there exists $\veps \in (0, 1)$ such that $q(t) = o_{t\rightarrow \infty}(t^{\veps})$. Next, observe that  $q(\log(1/\delta)^{1/\veps}) = o(\log(1/\delta))$. Thus, we have for $t_\delta = \log(1/\delta)^{1/\veps}$ 
\begin{eqnarray*}
\limsup_{\delta \to 0} \frac{\frac{1}{\eta_{t_\delta}} \log(B(t_\delta, \delta)) + q(t_\delta) }{\log \frac 1\delta}  \leq 1. 
\end{eqnarray*} 
Let us introduce 
$$ b(t_\delta) \ceq {\frac{1}{\eta_{t_\delta}} \log(B(t_\delta, \delta)) + q(t_\delta)} \:.$$
Let $\delta_{\min} \in (0, 1)$ be defined as 
\begin{eqnarray*}
\delta_{\min} &\ceq &\inf \left\{ \delta \in (0, 1) \mid b(t_\delta) > \log(1/\delta)^{1/\veps} \Gamma_\mdp \right\}	\\
&=& \inf \left\{ \delta \in (0, 1) \mid \Gamma_\mdp^{-1}b(t_\delta) > t_\delta \right\}\:, 
\end{eqnarray*}
which is well defined as $\limsup_{\delta \to 0}\frac{b(t_\delta)}{\log \frac 1\delta} \leq 1$ and $\veps \in (0, 1)$. Let $T_{\max} = \log(1/\delta_{\min})^{1/\veps}$. For all $t\geq T_{\max}$, there exists $(0, 1) \ni \delta' \leq \delta_{\min}$  such that $\lfloor t_{\delta'} \rfloor = t$ and $\Gamma_\mdp^{-1} b(t_{\delta'}) < t_{\delta'}$. Therefore, for all $\delta\leq \delta_{\min}$
\begin{equation}
\label{eq:zz-mm}
	T(\delta) \leq \log(1/\delta)^{1/\veps}\:. 
\end{equation}
Moreover, by definition $T(\delta) \leq \Gamma_\mdp^{-1} b(T(\delta))$ and $b$ is increasing for small $\delta$, it follows that 
\begin{eqnarray}
\label{eq:zz-pp}
	T(\delta) \leq \Gamma_\mdp^{-1} \cdot b(\log(1/\delta)^{1/\veps}).
\end{eqnarray}
Combining \eqref{eq:zz-oo}, \eqref{eq:zz-mm} and \eqref{eq:zz-pp}, it follows that for all $\delta \leq \delta_{\min}$, 
\begin{eqnarray}
	\bE_\mdp[\tau_\delta] &\leq& \Gamma_\mdp^{-1}\cdot b(\log(1/\delta)^{1/\veps}) + O(1). 
\end{eqnarray}
Finally, since $\limsup_{\delta \to 0} \frac{b(\log(1/\delta)^{1/\veps})}{\log \frac1\delta} \leq 1$, we conclude that 
\begin{equation*}
	\limsup_{\delta \rightarrow 0}\frac{\bE_\mdp[\tau_\delta]}{\log(1/\delta)} \leq \Gamma_\mdp^{-1}. 
\end{equation*}

\paragraph{Almost-sure upper bound}
The almost-sure bound on the sample complexity is derived from an application of Borel-Cantelli's lemma. Indeed, introducing for any $T$ the event 
$$ E(T) \ceq \lb \sum_{t=1}^T \indp{\tau_\delta > t}-  \sum_{t=1}^T \bP(\tau_\delta > t \mid \cH_{t-1}) \leq \sqrt{2T\log T^2}\rb\:,$$
we have when $E(T)$ holds 
\begin{eqnarray*}
	\min(\tau_\delta, T) &\leq& 1 + \sum_{t=1}^T \indp{ \tau_\delta > t }  \\
	&\leq& 1 + \sum_{t=1}^T \bP (\tau_\delta > t \mid \cH_{t-1}) + \sqrt{2T \log T^2} \\
	&\leq&  T(\delta) + \sqrt{2T \log T^2} + T_0 \:,  
\end{eqnarray*}
where the last inequality follows from the derivations above on the event $\cE(T)$ and holds with probability larger than $1-O(1/T^2)$. Next, assume $\cE(T)$ holds. 
Then, if $$T \geq T'(\delta) \ceq \sup\lb t\geq 1: t< T_0 + T(\delta) + \sqrt{2t\log t^2 }\rb\:,$$ and $E(T)$ holds then the stopping rule would trigger before episode $T$. Thanks to Azuma-Hoeffding $E(T)$ holds with probability at least $1-1/T^2$ for any $T\geq 1$.  Thus $E(T)^c$ holds with probability at most $1/T^2$, similarly for $\cE(T)^c$ and, as $\sum_{T\geq 1} \frac 1{T^2} <\infty$, we invoke Borel-Cantelli's lemma to justify that $\bP(\limsup_{T\to\infty} (E(T) \cap \cE(T))^c) = 0$ so with probability $1$, there exists $\wt T$ (possibly random) such that for $T\geq \wt T$, $E(T)\cap \cE(T)$ holds. 

\medskip\noindent
From the calculations above, we remark that the leading term in $T'(\delta)$ is at most of order $\log(1/\delta)$. Let $\delta_{\min}'$ be such that $\forall \;\delta \,\leq \,\delta_{\min}'$, $$\lceil \log(1/\delta)^{3/2} \rceil > T'(\delta), $$ which is well-defined. Let $\delta \leq \delta_{\min}'$ such that additionally $\lceil \log(1/\delta)^{3/2} \rceil > \wt T$. We then have 
\begin{eqnarray*}
	\tau_\delta &\leq& T_0 + T(\delta) + \sqrt{ 2\lceil \log(1/\delta)^{3/2} \rceil \log \lceil \log(1/\delta)^{3/2} \rceil^2}  \\
	&\leq& T_0 + T(\delta) +  \log(1/\delta)^{3/4} \sqrt{8\log (1+ \log(1/\delta)^{3/2})}\:.
\end{eqnarray*}
 Finally recalling that $$\limsup_{\delta \to 0} \frac{T(\delta)}{\log \frac1\delta} \leq \Gamma_\mdp^{-1}\:.$$ 
 We conclude by noting that 
 \begin{eqnarray*}
 	\limsup_{\delta \to 0} \frac{\tau_\delta}{\log \frac1\delta} &\leq& \limsup_{\delta \to 0} \frac{T(\delta)}{\log \frac1\delta} + \lim_{\delta \to 0} \frac{T_0 + \log(1/\delta)^{3/4} \sqrt{8\log (1+ \log(1/\delta)^{3/2})}}{\log \frac1\delta } \\
 	&\leq& \Gamma_\mdp^{-1} \:.
 \end{eqnarray*}

\section{CONCENTRATION RESULTS}
\label{appx:concentration}

\subsection{Self-noramlized Concentration}
We prove some concentration results used in other sections. We define the following concentration events of the self-normalized quantities related to the GLR
\begin{align*}
	\cE^p(\delta) &\ceq \lb \forall t \geq 1,\; \sum_{s,a, h} n_t(s,a,h)\KL(\hat p_t(\cdot \mid s, a,h) \;\|\; p(\cdot\mid s,a, h))\leq  \beta^p(t, \delta)\rb \\
	\cE^{\mu}(\delta) &\ceq \lb \forall t\geq 1,\; \sum_{s,a, h} n_t(s,a,h)\KL(\hat R_h^t(s,a)\;\|\; R_h(s,a))\leq  \beta^\mu(t, \delta)\rb \\
	\cE^{\mu+p}(\delta) &\ceq \lb \forall t\geq 1,\; \sum_{s,a, h} n_t(s,a, h) \KL(\hat R_h^t(s,a) \otimes \hat p_t(\cdot \mid s, a, h)\;\|\; R_h(s,a)\otimes p(\cdot\mid s,a, h)) \leq  \beta^{\mu +p}(t, \delta) \rb \end{align*}

It is possible to choose $\beta^p, \beta^{\mu}, \beta^{\mu+p}$ with logarithmic dependency to ensure that each event hold with probability $1-\delta$ (cf e.g.,  \citet{kaufmann_mixture_2021}). Since our focus is on characterizing the sample complexity in the asymptotic regime, we do not make these terms explicit. The interested reader may refer to \citet{kaufmann_mixture_2021} for explicit expressions of $\beta^p$, $\beta^{\mu}$, and $\beta^{\mu+p}$. 

Moreover, the event 
\begin{align*}
\cE^{\text{cnt}}(T)
&\ceq
\Bigl\{
\forall t \geq 1,\;\forall (s,a,h):\;
n_t(s,a,h)
\geq \tfrac{1}{2}\,\bar n_t(s,a,h) - \log(2SAH/\delta)
\Bigr\}.	
\end{align*}
holds with probability at least $1-\delta$, thanks to Lemma~\ref{lem:dan_cool}.

\begin{lemma}[\cite{NIPS2017_17d8da81}] 
\label{lem:dan_cool}
	Let $X_1, X_2, \dots$ be a sequence of Bernoulli random variables adapted to a filtration $\{ \cH_t\}_{t\geq 1}$ and such that for any $i$, 
	$\bP(X_i=1 \mid \cH_{i-1}) = P_i$. We have 
	$$ \bP \Big( \exists n \geq 1: \sum_{t\leq n} X_t  \leq \frac12 \sum_{t\leq n} P_t  - W \Big) \leq \exp(-W) \:.$$
\end{lemma}

\subsection{Gaussian and Dirichlet Concentration}
\newcommand{\ent}{\mathrm{Ent}}

We leverage Lemma~\ref{lem:concentr-mvn} to prove concentration of distribution over MDPs in $\Mdp$.

\begin{lemma}[\citet{lu_mvn_upb}]
\label{lem:concentr-mvn}
	Let $X \sim \cN(\theta, \Sigma)$ be a multivariate normal vector in dimension $d$. For any convex set $C \C \bR^d$ 
	$$ \bP_{X \sim \cN(\theta, \Sigma)}(X \in C) \leq \frac12 \myexp{ - \inf_{\lambda \in C} \frac12{ \left\| \lambda - \theta\right\|^2_{\Sigma^{-1}}}}.$$
\end{lemma}
 
 We note that from the  proof of Lemma~\ref{lem:concentr-mvn} in \citet{lu_mvn_upb}, the result extends to multivariate normals truncated to to a convex set. 
 \medskip
 
 We prove the following crucial concentration. 
 \medskip
\begin{lemma}
\label{lem:concentr}
	Let $P \ceq (P({s,a, h}))_{s,a, h}$ be a distribution whose marginals $(P({s,a, h}))_{s,a,h}$ are independent  $\dir((\alpha({s'\mid  s, a,h}) +1)_{s' \in \cS})$ and $R\ceq (R({s,a, h}))_{s,a,h}$ where each $R({s,a, h})$ is an independent $\cN(\mu({s,a, h }), \sigma^2(s,a, h)$, possibly truncated to a convex set $\cC$. Let $\cX \subset \Mdp$ such that for any $q \in \cP$, $\cX_q \ceq \lb \mdqr \in (0,1)^{SAH} : (\mdqr,q) \in \cX \rb$ is convex. Introducing $(\kappa_h)_h$ satisfying $\kappa(s'\mid s,a,h) \alpha(s,a, h) = {\alpha({s'\mid s, a,h})}\;, \forall (s,a,h,s')$,   we have 
\begin{align*}
	&\bP((R, Q) \in \cX)  \\ &\leq \frac 12 \lp \prod_{s,a h}\maxp{1}{ S \alpha(s,a, h)^{S-1} \myexp{-\ent(\kappa(\cdot \mid s,a,h))}} \rp \exp\Biggl\{  - \inf_{(\mdqr,q) \in \cX} \Biggl[ \sum_{s,a,h}\frac{(\mu(s,a, h) - \mdqr(s,a, h))^2}{2 \sigma^2({s,a, h })} + \\&   \quad  \sum_{s,a,h} \alpha(s,a, h) \KL(\kappa(\cdot \mid s,a,h) \| q(\cdot \mid s,a,h)) \Biggr] \Biggr\} \:. 
\end{align*}

\end{lemma}

\begin{proof}
We have $\cX \subset  \Mdp \ceq (0, 1)^{SAH} \times \triangle^{SAH}$  and by assumption, for any $q \in \triangle^{SAH}$, the set $\cX_q  = \lb r \in \cR : (r,q) \in \cX \rb$ is convex. Next, $P$ has the same distribution as a multivariate normal in dimension $SAH$ with independent marginals. Noting that $\bP((R, P) \in \cX) = \bE \lsb \bP((R, P) \in \cX \mid P) \rsb$, and, since $R,P$ are independent, leveraging Lemma~\ref{lem:concentr-mvn} we have 
	\begin{eqnarray}
		\bP((R, P) \in \cX) &=& \bE \lsb \bP((R, P) \in \cX   \big | P) \rsb \notag \\
		&\leq& \frac12 \bE \lsb  \exp\lb - \inf_{\mdqr \text{ s.t } (\mdqr, P) \in \cX} \sum_{s,a,h}\frac{(\mu(s,a, h) - \mdqr(s,a, h))^2}{2 \sigma^2({s,a, h })} \rb \rsb \label{eq:ze-xs-dc}
		\end{eqnarray}
which invokes Lemma~\ref{lem:concentr-mvn} since   given $P \in \cP$, $\cX_P$ as defined above is convex and $R$ is a multivariate normal vector. 

\medskip
\noindent
We recall that for $\alpha \in \bR^d$, the density of $\dir(\alpha)$ is proportional to 
$$ \myexp{ - \sum_{i} (\alpha_i - 1) \log \frac1{x_i}}  \:.$$

\medskip
Next, we have for any $(s,a,h)$, letting $q(\cdot \mid s,a ,h) \in \triangle$, 
\begin{align*}
\sum_{s'} \alpha(s' \mid  s,a,h)\log\frac{1}{q(s' \mid  s,a,h)} &= \alpha(s,a, h) 	\sum_{s'} \kappa(s'  \mid s,a, h)\log\frac{1}{q(s' \mid s,a,h)} \\
&= \alpha(s,a, h) 	\sum_{s'} \kappa(s'\mid s,a,h)\log\frac{\kappa(s' \mid s,a,h)}{q(s' \mid s,a,h)} - \\ &\quad  \alpha(s,a, h) \sum_{s'} \kappa(s'\mid s,a,h)\log{\kappa(s' \mid s,a,h)} \\
&= \alpha(s,a, h) 	\sum_{s'} \kappa(s' \mid s,a,h)\log\frac{\kappa(s' \mid s,a,h)}{q(s'\mid s,a,h)} -  \sum_{s'} \alpha(s'\mid s,a,h)\log{\kappa(s'\mid s,a,h)}\:, \\
&= \alpha(s,a, h)  \KL(\kappa(\cdot \mid s,a,h) \,\|\, q(\cdot \mid s,a,h)) - \sum_{s'} \alpha(s'\mid s,a,h)\log{\kappa(s' \mid s,a,h)}\:.
\end{align*}

Thus, the density of $\dir(1 + \alpha(\cdot \mid s,a,h))$ is also proportional to $$ \myexp{-\alpha(s,a, h)  \KL(\kappa(s'\cdot \mid s,a,h) \,\| \,q(s'\cdot \mid s,a,h)) )}.$$

Combining with Eq.\eqref{eq:ze-xs-dc} we have 
\begin{align*}
\bP((R, P) \in \cX) 		&\leq \frac1{2 N} \bigint_{\cP}  \exp\lb - \inf_{\mdqr \text{ s.t } (\mdqr, q) \in \cX} \sum_{s,a,h}\frac{(\mu(s,a, h) - \mdqr(s,a, h))^2}{2 \sigma^2({s,a, h })} \rb \cdot \\ & \exp\lb - \sum_{s,a,h} \alpha(s,a, h) \KL(\kappa(\cdot \mid s,a,h) \,\| \,q(\cdot\mid s,a,h)) \rb \diff q \:, 
\end{align*} 
where 
\begin{equation}
	N \ceq \bigint_{\cP}  \exp\lb  - \sum_{s,a,h} \alpha(s,a, h) \KL(\kappa(\cdot\mid s,a,h) \,\| \,q(\cdot \mid s,a,h)) \rb \diff q \:.
\end{equation}
Next, 
\begin{eqnarray*}
\bP((R, P) \in \cX) &\leq& \frac{\vol(\cP)}{2 N} \cdot\exp\lb  - \inf_{q \in \cP }  \lsb \inf_{\mdqr \text{ s.t. } (\mdqr, q) \in \cX} \sum_{s,a,h}\frac{(\mu(s,a, h) - \mdqr(s,a, h))^2}{2 \sigma^2({s,a,h })} + \right. \right.\\ && \left. \left.  \quad \sum_{s,a,h} \alpha_h(s,a) \KL(\kappa(\cdot \mid s,a,h) \,\|\, q(\cdot \mid s,a,h)) \rsb\rb, 
	\end{eqnarray*} 
then applying Lemma~\ref{lem:double_inf} yields 
\begin{align*}
	\inf_{q \in \cP}  \lsb \inf_{\mdqr \text{ s.t } (\mdqr, q) \in \cX} \sum_{s,a,h}\frac{(\mu(s,a,h) - \mdqr(s,a,h))^2}{2 \sigma^2({s,a , h })} + \sum_{s,a,h} \alpha(s,a, h) \KL(\kappa(\cdot \mid s,a,h) \,\| \,q(\cdot \mid s,a,h)) \rsb = \\
	\inf_{(\mdqr,q) \in \cX} \lsb \sum_{s,a,h}\frac{(\mu(s,a, h) - \mdqr(s,a,h))^2}{2 \sigma^2({s,a,h })} + \sum_{s,a,h} \alpha(s,a, h) \KL(\kappa(\cdot \mid s,a,h)\, \| \,q(\cdot \mid s,a,h)) \rsb\:,  
\end{align*}
thus we have 
\begin{eqnarray*}
	\bP((R, Q) \in \cX)  \leq \frac{\vol(\cP)}{2 N} \exp\lb - \inf_{(\mdqr,q) \in \cX} \lsb \sum_{s,a,h}\frac{(\mu(s,a, h) - \mdqr(s,a, h))^2}{2 \sigma^2({s,a,h })} + \sum_{s,a,h} \alpha(s,a, h) \KL(\kappa(\cdot \mid s,a,h) \,\|\, q(\cdot \mid s,a,h)) \rsb \rb, 
\end{eqnarray*}
We remark that $N$ rewrites as 
$$ \prod_{s,a,h} \lsb \bigint_{\triangle}  \myexp{  - \alpha(s,a, h) \KL(\kappa(\cdot \mid s,a,h) \| q) } \diff q  \rsb \:,$$
then we invoke Lemma~\ref{lem:dirichlet_bound} which yields

$$ \bigint_{\triangle}  \myexp{ - \alpha(s,a, h) \KL(\kappa(\cdot \mid s,a,h) \| q) } \diff q \geq \min \lb  \frac{\myexp{\ent(\kappa(\cdot \mid s,a,h))}{\alpha(s,a,h)}^{-(S-1)}}{S! }, \frac{1}{(S-1)!}\rb, $$
where $\ent(p)$ denotes the entropy. 

Combining the above displays and noting that $\vol(\triangle) = \frac{1}{(S-1)!}$ and $\vol(\cP) = (\vol(\triangle))^{SAH}$ yields 

\begin{align*}
	&\bP((R, Q) \in \cX)  \\ &\leq \frac 12 \lp \prod_{s,a h}\maxp{1}{ S \alpha(s,a, h)^{S-1} \myexp{-\ent(\kappa(\cdot \mid s,a,h))}} \rp \exp\Biggl\{  - \inf_{(\mdqr,q) \in \cX} \Biggl[ \sum_{s,a,h}\frac{(\mu(s,a, h) - \mdqr(s,a, h))^2}{2 \sigma^2({s,a, h })} + \\&   \quad  \sum_{s,a,h} \alpha(s,a, h) \KL(\kappa(\cdot \mid s,a,h) \| q(\cdot \mid s,a,h)) \Biggr] \Biggr\} \:. 
\end{align*}

\end{proof}

The following lemma bounds the normalizing constant of the posterior Dirichlet distribution when its density is expressed in terms of KL divergence.

    \begin{lemma}
    \label{lem:dirichlet_bound} Let $q\in \triangle$ (the probability simplex of $\bR^d$) and denote by $\ent(q) \ceq \sum_i - q_i \log {q_i}$ the entropy of $q$. 
    We have 
    	$$ \int_\triangle \myexp{ - n \KL(q\,\|\, x)} \diff x \geq \min \lb  \frac{\myexp{\ent(q)}n^{-(d-1)}}{d! }, \frac{1}{(d-1)!}\rb  \:.$$
    \end{lemma}
    
      \begin{proof}
    Below $\diff x$ denotes a Hausdorff-measure differential element and, (only) below, $\Gamma$ is the gamma function. 
    For $n=0$, the bound follows as $\int_\triangle  \diff x = \mathrm{Vol}(\triangle) = \frac{1}{\Gamma(d)} =  \frac{1}{(d-1) !}$. Next, we assume $n\geq 1$. Let $\zeta \in (0, 1)$ and observe that 
    	\begin{eqnarray*}
    		\int_\triangle \myexp{  - n \KL(q\,\|\,  x)} dx &\geq& \int_{(1-\zeta)q +\zeta \triangle} \myexp{  - n \KL(q \,\|\,  x)} \diff \diff x \\
    		&=& \zeta^{d-1} \int_{\triangle} \myexp{  - n \KL(q\,\|\, (1-\zeta)q +\zeta x)} \diff x \\
    		&\geq& \zeta^{d-1} \int_{\triangle} \myexp{  - n (1-\zeta) \KL(q\,\|\,q) - n\zeta \KL(q\,\|\, x) } \diff x
    	\end{eqnarray*}
    	which follows by convexity of $x\mapsto \KL(q\,\|\, x)$. Combining with $\KL(q\,\|\, q) =0$ and letting $\zeta = 1/n$ it follows that 
    	\begin{eqnarray*}
    		\int_\triangle \myexp{ - n \KL(q\,\|\,x)} dx &\geq& n^{-(d-1)} \int_\triangle \myexp{ - \KL(q\|  x)} \diff x.
    	\end{eqnarray*} 
    	Next, we bound the right-hand integral. We observe that 
    	\begin{eqnarray*}
    		\KL(q\|x) = \sum_i q_i \log \frac{q_i}{x_i}=  -\ent (q) - \sum_i q_i  \log {x_i}\:, 
    	\end{eqnarray*}
        where  $\ent(q) \ceq \sum_i - q_i \log {q_i}$. Combining above displays yield  
    	\begin{eqnarray*}
    		\int_\triangle \myexp{  - \KL(q\|x)} \diff x = \myexp{\ent (q)}\int_\triangle \prod_i x_i ^{q_i} \diff x 
    	\end{eqnarray*}
    	Next, observe that  $\int_\triangle \prod_i x_i ^{q_i} \diff x$ is the normalizing constant of $\mathrm{Dir}(1+q)$ thus	\begin{eqnarray*}
    		\int_\triangle \prod_i x_i ^{q_i} \diff x &=& \frac{\prod_i \Gamma(1+q_i)}{\Gamma(\sum_i 1+q_i)}\\
    		&=& \frac{\prod_i  \Gamma(1+q_i)}{\Gamma(d+1)} = \frac{\prod_i \Gamma(1+q_i)}{d ! }, 
    	\end{eqnarray*}
    	where $\Gamma$ is the gamma function. As proven in \citet{gamma_function},  
    	$\Gamma(1+x)\geq x^x$ for any $0\leq x\leq 1$. Combining the above displays, we have 
    	\begin{eqnarray*}
    		\int_\triangle \myexp{ - \KL(q \,\|\,x)} \diff x &\geq& \frac{\myexp{\ent (q)} }{d !}   \prod_i \frac{\Gamma(1+q_i)}{q_i^{q_i}}   \\
    		&\geq& \frac{\myexp{\ent (q)} }{d ! } \prod_i \frac{q_i^{q_i}}{q_i^{q_i}} = \frac{\myexp{\ent (q) }}{d! } \:. 
    	\end{eqnarray*}
    	Therefore 
    	\begin{equation*}
    	    \int_\triangle \myexp{  - n \KL(q\,\|\,  x)} \diff x \geq \frac{\myexp{\ent (q)}n^{-(d-1)}}{d! }\:.
        \qedhere
        \end{equation*}
    \end{proof}
    
\section{TECHNICAL RESULTS}
\label{appx:miscellaneous}
We state different technical results used in the other sections. 

The following result is known and appears in \citet{pmlr-v37-krichene15}; we include a proof here for completeness.  
\begin{lemma}
\label{lem:pos-derv}
	Let $\rho$ be a density wrt canonical measure of $\cX$ and $g : \cX \to \bR$
	$$ \eta  \mapsto \frac1\eta \log \int \rho(x) \exp(-\eta g(x)) \diff x $$
	is non-decreasing. 
\end{lemma}
\begin{proof}
	Let $\Lambda_\eta \ceq \int \rho(x) \exp(-\eta g(x))\diff x$ and $f(\eta) = \frac{1}{\eta} \log \Lambda_\eta$ and let the density 
	$w_\eta(x) = \frac{\rho(x) \exp(-\eta g(x))}{\Lambda_\eta}$  and $X \sim w_\eta$. By derivation under the integral sign 	\begin{eqnarray*}
		f'(\eta) &=& - \frac{1}{\eta^2} \log \Lambda_\eta  + \frac{1}{\eta}\frac{\int -g(x) \rho(x)\exp \lp - \eta  g(x)  \rp \diff x }{\int\rho(x)\exp \lp - \eta  g(x)  \rp \diff x } \\
		&=&  - \frac{1}{\eta^2} \log \Lambda_\eta  + \frac{1}{\eta^2} \bE_{X}\lsb \log \exp(-\eta g(X))\rsb  \\
		&=& \frac 1 {\eta^2} \bE_{X}\lsb \log \frac{\exp(-\eta g(X))}{\Lambda_{\eta}}\rsb
			\\ 
			&=& \frac 1 {\eta^2} \bE_{X}\lsb \log \frac{w_\eta(X)}{\rho(X)}\rsb \\
			&=& \frac1{\eta^2} \KL(\nu_{\eta}, \nu_\rho) \geq 0\:,
\end{eqnarray*}  
where $\nu_{\eta}$ is the distribution with density $w_\eta$ and  $\nu_\rho$ the distribution with distribution $\rho$. 
\end{proof}

\begin{lemma}
\label{lem:conv_vol}
	For any mesaurable set $A \subset \Mdp$ there exists an aboluste constant $k_0<\infty$ such that for $t\geq k_0$, 
	$$ \vol(A \cap \Mdp^{1/S\cdot \sqrt{t}}) \geq \frac12 \vol(A)\:.$$
\end{lemma}
\begin{proof}
	The sequence $a_t \ceq \vol(A \cap \Mdp^{1/S\cdot \sqrt{t}})$ is non-decreasing and upper bounded by $\vol(A)$, which is also its supremum, thus it converges to this value, and the conclusion follows from this observation. 
\end{proof}

\begin{lemma}
\label{lem:double_inf} 
	 Let $\Delta$ be a subset of $\cX \times \cY$ and $f :  \cS \to \bR$ be  such that for all $(x,y) \in \cX \times \cY, f(x,y) = g(x) + h(y)$ with $g: \cX \to \bR$ and $h: \cY \to \bR$.  Then 
 \[   \inf_{x \in \cX} \lsb g(x) + \inf_{y \in \cY \atop \text{s.t } (x, y) \in \Delta} h(y)\rsb  = \inf_{(x , y) \in \Delta}f(x , y) \]
	 \end{lemma} 
\begin{proof}
	\begin{eqnarray*}
		\inf_{(x,y) \in \Delta} f(x, y) &=& \inf_{x\in \cX} \inf_{y \in \cY \atop \text{s.t } (x, y) \in \Delta}f(x, y) \\
		&=& \inf_{x\in \cX} \inf_{y \in \cY \atop \text{s.t } (x, y) \in \Delta}\lsb g(x) + h(y) \rsb \\
		&=& \inf_{x\in \cX} \lsb  g(x) +  \inf_{y \in \cY \atop \text{s.t } (x, y) \in \Delta} h(y) \rsb. 
	\end{eqnarray*}
\end{proof}

\begin{lemma}
\label{lem:dev-induc}
Let $\mdp,\mdq\in\Mdp$ satisfy $\|\mdp-\mdq\|_1\leq\gamma$.
Then for any $(s,a,h)\in\mathcal{S}\times\mathcal{A}\times[H]$ and any 
$\pi\in\Pi_{\mathrm{det}}$,
\[
\bigl|\tilde{Q}_h^\pi(s,a) - Q_h^\pi(s,a)\bigr| \leq (H-h+1)\,\gamma.
\]
\end{lemma}

\begin{proof}
Conditioning on $s_h=s$, $a_h=a$ and following $\pi$ thereafter,
the value difference lemma (Lemma~\ref{lem:val-diff-lem}) applied 
between $\mdp$ and $\mdq$ gives
\[
Q_h^\pi(s,a) - \tilde{Q}_h^\pi(s,a)
  = \bE_\mdp^\pi\!\left[
    \sum_{k=h}^H \delta_k(s_k,a_k)
    \,\middle|\, s_h=s,\,a_h=a
  \right],
\]
where
\[
\delta_k(s,a) \ceq
  \bigl(\mu(s,a, k)-\tilde\mu(s,a, k)\bigr)
  +\bigl(p_k(\cdot\mid s,a)-\tilde{p}_k(\cdot\mid s,a)\bigr)^\top
   \tilde{V}_{k+1}^\pi(\cdot).
\]
Since rewards lie in $(0,1)$, $\|\tilde{V}_{k+1}^\pi\|_\infty\leq H-k$,
so
\[
|\delta_k(s,a)|
  \leq |\mu(s,a, k)-\tilde\mu(s,a, k)|
  +(H-k)\|p_k(\cdot\mid s,a)-\tilde{p}_k(\cdot\mid s,a)\|_1
  \leq (H-k+1)\,\gamma_{s,a,k},
\]
where $\gamma_{s,a,k}\ceq|\mu(s,a, k)-\tilde\mu(s,a, k)|
+\|p_k(\cdot\mid s,a)-\tilde{p}_k(\cdot\mid s,a)\|_1$.
Therefore,
\begin{align*}
\bigl|Q_h^\pi(s,a)-\tilde{Q}_h^\pi(s,a)\bigr|
  &\leq \sum_{k=h}^H
    \bE_\mdp^\pi\!\left[|\delta_k(s_k,a_k)|\,\middle|\,
    s_h=s,\,a_h=a\right] \\
  &\leq \sum_{s',a'} \sum_{k=h}^H (H-k+1)
    \gamma_{s',a',k} \\
  &\leq (H-h+1)\,\|\mdp-\mdq\|_1 \\
  &\leq (H-h+1)\,\gamma.
\end{align*}
\end{proof}

\medskip
The purpose of the following lemma is to show that if $\mdq \in \alt(\mdp)$ does not have enough enclosing volume around it, then it is a difficult instance, i.e., the sup-optimality gaps are small. 

\begin{lemma}
\label{lem:non-isolated}
Let $\mdp \in \Mdp$ and let $\pi \in \Pi_\text{det}$ be its unique globally optimal policy.  
Denote by $\Delta^\pi(\mdp)$ the smallest value gap in $\mdp$ for policy $\pi$, and let $$ A^\pi \;\coloneqq\; \bigl\{\mdp' \in \Mdp : \pi \text{ is globally optimal in } \mdp'\bigr\}.$$
If the $\ell_1$-ball of radius $\tfrac{\delta}{2(H+1)}$ around $\mdp$ is not entirely contained in $A^\pi$, i.e.,
$$ \left\{\, \mdp' \in \Mdp : \left\| \mdp - \mdp'\right\|_1 \leq \frac{\delta}{2(H+1)} \,\right\} \;\not\C\; A^\pi,$$
then the smallest gap of $\mdp$ must satisfy $$ \Delta^\pi(\mdp) \;\leq\; \delta\:.$$ 
\end{lemma}

\begin{proof} 
Fix $\delta, \delta' \geq 0$.  
Assume that $\Delta^\pi(\mdp) > \delta'$, and let $\mdp' \in \Mdp$ satisfy $\| \mdp - \mdp'\|_1 \leq \delta$.  

Denote by $\wt Q_h$ and $\wt V_h$ the state--action and state value functions in $\mdp'$, and by $Q_h$ and $V_h$ their counterparts in $\mdp$. 

Then, for any action $a \neq \pi_h(s)$, we have
\begin{eqnarray*}
	\wt Q_{h}^\pi(s,\pi_h(s)) - \wt Q_h^\pi(s,a) &=& Q_h^\pi(s,\pi_h(s)) - Q_h^\pi(s,a)  + \lp \wt V_h^\pi(s)  - V_h^\pi(s)\rp + \lp Q_h^\pi(s,a) - \wt Q_h^\pi(s,a)\rp \\
	&\geq& \delta' + \lp \wt V_h^\pi(s)  - V_h^\pi(s)\rp + \lp Q_h^\pi(s,a) - \wt Q_h^\pi(s,a)\rp \\
	&\geq& \delta' - \delta - \lvert \wt V_h^\pi(s)  - V_h^\pi(s) \rvert  - \lvert \wt V_{h+1}^\pi(s)  - V_{h+1}^\pi(s) \rvert \\
	&\geq& \delta' -(2H-2h +2)\delta, 
\end{eqnarray*}	
where the last inequality follows by value  difference lemma  as $\| \mdp - \mdp'\|_1 \leq \delta$. Thus, it suffices to take $$\delta' \ceq  2(H+1)\delta $$ to guarantee that $\pi$ is the unique globally optimal policy in the MDP $\mdp'$. Indeed, for any other policy $\rho$, by induction, we have at step $h$, 
\begin{eqnarray*}
	V_h^\rho (s) &=& \mu(s, \rho_h(s), h) + p_hV^\rho_{h+1}(s, \rho_h(s)) \\
	&<& \mu(s, \rho_h(s), h) + p_hV^\pi_{h+1}(s, \rho_h(s)) \\
	&<& \mu(s, \pi_h(s), h) + p_hV^\pi_{h+1}(s, \pi_h(s)) = V_h^\pi(s), 
\end{eqnarray*}
which concludes the proof. 

\end{proof}

\bigskip \noindent

The goal of the next lemma is to show that from an instance $\mdp$, one can pick a policy $\pi \in \Pi^\star_\text{det}(\mdp)$ and build an easier instance $\mdp'$ (with larger value gaps) close to $\mdp$ and such that $\pi$ is still optimal in $\mdp'$. This will be used to show that if an instance $\mdp$ lies on the boundary of $\alt(\mdq)$ (so a difficult instance with nearly zero gap) of another instance $\mdq$, one can build an easier instance $\mdp'$ close to $\mdp$, which is still in $\alt(\mdq)$, showing that there is enough volume around any point in the interior of the alternative set of an MDP. 

\bigskip \noindent

\begin{lemma}
\label{lem:build_easier}
Let $\mdp\in\Mdp $ and let $\pi\in\Pi_{\mathrm{det}}$ be a globally 
optimal policy for $\mdp$. For any 
$\varepsilon\leq\frac{1}{3(1+2^{H-1})}$, there exists an MDP 
$\tilde\mdp\ceq (\tilde\mu,\tilde{p})$ and coefficients $(\gamma_{s,a,h})_{s,a,h}\in[0,1]$ such that for all $(s,a,h)$,
\begin{equation}
\label{eq:tilde-mdp-def}
\tilde{p}_h(\cdot\mid s,a) \ceq
  (1-\gamma_{s,a,h})p_h(\cdot\mid s,a)
  +\gamma_{s,a,h}\,p_h(\cdot\mid s,\pi_h(s)),
\qquad
|\mu(s,a, h)-\tilde\mu(s,a, h)|\leq d_{s,a,h}\ceq 2^{H-h+1}\veps,
\end{equation}
with
\[
\gamma_{s,a,h}\in\left[0,\;
  \frac{(1+2^{H-h})\varepsilon}{1-(2^{H-h+1}+3)\veps}
\right].
\]
Moreover, $\pi$ is strongly globally optimal in $\tilde\mdp$: for all 
$(s,h)$,
\[
\tilde{Q}_h^\pi(s,\pi_h(s)) - \max_{a\neq\pi_h(s)}\tilde{Q}_h^\pi(s,a)
  \geq \varepsilon.
\]
\end{lemma}

\begin{proof}
We construct $\tilde\mdp$ bottom-up from stage $H$ to stage $1$. Throughout,
$Q,V$ (resp. $\tilde{Q},\tilde{V}$) denote value functions in $\mdp$ 
(resp. $\tilde\mdp$), and we maintain the induction hypothesis
\begin{equation}
\label{eq:ind-hyp}
0 \leq \tilde{V}_h^\pi(s) - V_h^\pi(s) \leq \Omega_h(\varepsilon)
  \ceq (2^{H-h+1}-1)\varepsilon,
\qquad\forall\,s.
\end{equation}

\medskip\noindent
\textbf{Base case $h=H$.}
If $\mu(s,\pi_H(s), H)-\mu(s,a, H)\leq\varepsilon$, set
\begin{equation}
\label{eq:reward-update}
\tilde\mu(s,a, H) \ceq \max(0,\mu(s,a, H)-\varepsilon),
\qquad
\tilde\mu(s,\pi_H(s), H) \ceq \min(1,\mu(s,\pi_H(s), H)+\varepsilon),
\end{equation}
and $\gamma_{s,a,H}=0$. This ensures 
$\tilde{Q}_H^\pi(s,\pi_H(s))-\tilde{Q}_H^\pi(s,a)\geq\varepsilon$ 
and $0\leq\tilde{V}_H^\pi(s)-V_H^\pi(s)\leq\varepsilon=\Omega_H(\varepsilon)$,
so \eqref{eq:ind-hyp} holds at $h=H$.

\medskip\noindent
\textbf{Induction step.}
Assume \eqref{eq:ind-hyp} holds at stage $h+1$. At stage $h$, for each 
$(s,a)$ we consider three cases.

\medskip\noindent
\textit{Case 1.} If
\begin{equation}
\label{eq:case1}
\mu(s,\pi_h(s), h)-\mu(s,a, h)
  + p_h\tilde{V}_{h+1}^\pi(s,\pi_h(s))
  - p_h\tilde{V}_{h+1}^\pi(s,a) \geq \varepsilon,
\end{equation}
set $\tilde{p}_h(\cdot\mid s,a)\ceq p_h(\cdot\mid s,a)$ and 
$\tilde\mu(s,a, h)\ceq\mu(s,a, h)$, i.e., $\gamma_{s,a,h}=0$. The 
gap is already at least $\varepsilon$ without any perturbation.

\medskip\noindent
\textit{Case 2.} If \eqref{eq:case1} does not hold but
\begin{equation}
\label{eq:case2}
\tilde\mu(s,\pi_h(s), h)-\tilde\mu(s,a, h)
  \geq \mu(s,\pi_h(s), h)-\mu(s,a, h)+\Omega_{h+1}(\varepsilon)+\varepsilon,
\end{equation}
set $\gamma_{s,a,h}=0$. Shifting rewards alone is sufficient to 
guarantee the gap.

\medskip\noindent
\textit{Case 3.} Otherwise, neither \eqref{eq:case1} nor \eqref{eq:case2} 
holds. Since \eqref{eq:case2} does not hold, we have 
\begin{equation}
\label{eq:mu-bounds}
\mu(s,\pi_h(s), h)+\Omega_{h+1}(\varepsilon)+\varepsilon \geq 1
\quad\text{and}\quad
\mu(s,a, h)-\Omega_{h+1}(\varepsilon)-\varepsilon \leq 0,
\end{equation}

Indeed, if any of the statements above above didn't hold, we would have Eq.\eqref{eq:case2}. This further implies 
	\begin{equation}
	\label{eq:op-ol-pm}
	\mu_h(s,\pi_h(s))  - \mu_h(s,a)  \geq 1- 2\lp \Omega_{h+1}(\veps) + \veps\rp	, 
	\end{equation}
and since \eqref{eq:case1} does not hold either, we have 
\begin{equation}
\label{eq:V-diff-bound}
p_h\tilde{V}_{h+1}^\pi(s,\pi_h(s)) - p_h\tilde{V}_{h+1}^\pi(s,a)
  \leq \varepsilon + 2(\Omega_{h+1}(\varepsilon)+\varepsilon) - 1.
\end{equation}
By the mixture definition of $\tilde{p}_h$ and the induction hypothesis,
\begin{align*}
\tilde{Q}_h^\pi(s,\pi_h(s)) - \tilde{Q}_h^\pi(s,a)
  &\geq Q_h^\pi(s,\pi_h(s)) - Q_h^\pi(s,a) - \Omega_{h+1}(\varepsilon) \\
  &\quad + \gamma\bigl(
      p_h\tilde{V}_{h+1}^\pi(s,a)
      - p_h\tilde{V}_{h+1}^\pi(s,\pi_h(s))
    \bigr) \\
  &\geq -\Omega_{h+1}(\varepsilon)
    + \gamma\bigl(1-\varepsilon-2(\Omega_{h+1}(\varepsilon)+\varepsilon)\bigr),
\end{align*}
where the first inequality uses $\pi\in\Pi_{\mathrm{det}}^\star(\mdp)$ 
(so $Q_h^\pi(s,\pi_h(s))-Q_h^\pi(s,a)\geq 0$) and \eqref{eq:ind-hyp},
and the second uses \eqref{eq:V-diff-bound}. Setting
\begin{equation}
\label{eq:gamma-def}
\gamma_{s,a,h} \ceq
  \frac{\varepsilon+\Omega_{h+1}(\varepsilon)}
       {1-\varepsilon-2(\Omega_{h+1}(\varepsilon)+\varepsilon)}
  = \frac{(1+2^{H-h})\varepsilon}{1-(2^{H-h+1}+3)\varepsilon}
\end{equation}
gives $\tilde{Q}_h^\pi(s,\pi_h(s))-\tilde{Q}_h^\pi(s,a)\geq\varepsilon$.
The condition $\gamma_{s,a,h}\in[0,1]$ requires 
$1-(2^{H-h+1}+3)\varepsilon\geq(1+2^{H-h})\varepsilon$, which holds 
whenever $\varepsilon\leq\frac{1}{3+2^{H-h}+2^{H-h+1}}$. This is 
satisfied for all $h\in[H]$ if $\varepsilon\leq\frac{1}{3(1+2^{H-1})}$.

\medskip\noindent
\textbf{Verifying the induction hypothesis at stage $h$.}
Since $\tilde{p}_h(\cdot\mid s,\pi_h(s))=p_h(\cdot\mid s,\pi_h(s))$ 
and $\tilde\mu(s,\pi_h(s), h)\leq\mu(s,\pi_h(s), h)+\varepsilon$,
\begin{align*}
\tilde{V}_h^\pi(s)
  &= \tilde\mu(s,\pi_h(s),h) + p_h\tilde{V}_{h+1}^\pi(s,\pi_h(s)) \\
  &\leq \mu(s,\pi_h(s),) + \varepsilon
    + \Omega_{h+1}(\varepsilon) + p_h V_{h+1}^\pi(s,\pi_h(s)) \\
  &= V_h^\pi(s) + \underbrace{2\Omega_{h+1}(\varepsilon)+\varepsilon}_{
      =\,\Omega_h(\varepsilon)},
\end{align*}
and, it is simple to verify that $\tilde{V}_h^\pi(s)\geq V_h^\pi(s)$ since both rewards and 
transitions are translated in favour of $\pi$. This confirms \eqref{eq:ind-hyp} 
at stage $h$, with the closed form $\Omega_h(\varepsilon)=(2^{H-h+1}-1)\varepsilon$.

\medskip\noindent
From the update rule \eqref{eq:reward-update} and the definition of 
$\Omega_{h+1}$,
\[
|\tilde\mu(s,a, h)-\mu(s,a, h)|
  \leq \Omega_{h+1}(\varepsilon)+\varepsilon
  = 2^{H-h+1}\varepsilon
  = d_{s,a,h},
\]
which holds for all $(s,a,h)$, completing the proof.
\end{proof}

\begin{lemma}
\label{lem:convex_vol}
	Let an MDP $\mdp \in \Mdp$ with a unique stage-optimal policy $\pi \in \Pi_\text{det}$. If $\pi$ is a global optimal policy in $\mdp$ and $\Delta^\pi(\mdp)>0$, then for any $\gamma\in \bigl(0, \frac{\Delta}{6SAH(H+1)} \bigr)$, $\pi$ is a global optimal policy for any MDP in the set 
	$(1-\gamma) \mdp + \gamma \Mdp$. 
\end{lemma}
\begin{proof}
Let $\pi \in \Pi_\text{det}$ be a global optimal policy in $\mdp$.  
From Lemma~\ref{lem:non-isolated} we have for $\veps \leq \Delta$ and for any
$\mdp' \in \bB_1(\mdp, \frac{\veps}{2(H+1)})$, $\pi$ is an optimal policy in $\mdp'$.  Now it suffices to take $\gamma$ such that $(1 - \gamma)\mdp  + \gamma \Mdp \C \bB_1(\mdp ,  \gamma')$. For this, observe that for any $\mdp' \in \Mdp$ we have 
\begin{eqnarray*}
	\left\| \mdp - ((1-\gamma)\mdp + \gamma \mdp') \right\|_1  &=& \gamma \left \| \mdp - \mdp'  \right\|_1 \\
	&\leq&   2 \gamma SAH.
\end{eqnarray*}

Thus we have $(1 - \gamma)\mdp  + \gamma \Mdp \C \bB_1(\mdp ,  3\gamma SAH)$ and taking  $\gamma$ such that $3\gamma SAH \leq \frac{\Delta}{2(H+1)}$ and using Lemma~\ref{lem:non-isolated} concludes the proof. 
\end{proof}

The goal of the following lemma is to, given an MDP $\mdp$, generate $\mdp' \in \Mdp$ such that the transition probabilities in $\mdp'$ are larger than some threshold while $\mdp$ and $\mdp'$ still have the same optimal policy. 
\begin{lemma}
\label{lem:incr-trans} 
	Let an MDP $\mdp \ceq (\mu, p)$ with optimal policy $\pi$ such that its minimum policy gap $\Delta^\pi(\mdp) >0$. Let $u$ denote the uniform transition kernel and define for any $\gamma \in (0, 1)$ the MDP $\mdp_\gamma =(\mu, (1-\gamma)p+  \gamma u)$. For $\gamma\in (0, \Delta_\mdp(\pi)/(3H)$, $\pi$ is state-wise striclty optimal in $\mdp_\gamma$.  
\end{lemma}

\begin{proof}
Let $\gamma$ be fixed and simplify notation by letting $\mdq = \mdp_\gamma$, $\Delta_\mdp(\pi) = \Delta$, and define $\wt Q, \wt V$ the $Q$-function and value function in $\wt\mdp$. We have for any $a \neq \pi_h(s)$  
	\begin{eqnarray*}
		\wt Q_h^\pi(s,\pi_h(s)) - \wt Q_h^\pi(s, a )  &\geq& \Delta + \wt Q_h^\pi(s,\pi_h(s)) - Q_h^\pi(s, \pi_h(s)) + Q_h^\pi(s,a) - \wt Q_h^\pi(s,a) \\
		&=&  \Delta + \wt V_h^\pi(s) - V_h^\pi(s) + p_hV_{h+1}^\pi(s, a) - \tilde p_h\wt V_{h+1}^\pi(s,a ) \\
		&=& \Delta + \wt V_h^\pi(s) - V_h^\pi(s) + p_hV_{h+1}^\pi(s, a) - p_h\wt V_{h+1}^\pi(s,a ) + \gamma ( p_h - u)\wt V_{h+1}^\pi(s,a) \\
		&\geq& \Delta - 3H \gamma 
	\end{eqnarray*}
	so it suffices for $\gamma$ to be smaller than $\Delta/(3H)$ to guarantee the claimed property. 
\end{proof}

\bigskip 
The next lemma is proven in \citet{pmlr-v139-menard21a}. We restate it for completeness. 
\begin{lemma}
\label{lem:elliptic-potential}
	For a sequence $(u_t)_{t\geq 1} \in (0,1)^\bN$ and $U_t \ceq \sum_{l=1}^t u_l$ we have 
	$$ \sum_{t=1}^T \frac{u_{t+1}}{U_t \lor 1} \leq 4 \log(U_{T+1} +1).$$
\end{lemma}
 
\section{BEYOND EXACT POLICY IDENTIFICATION}
\label{sec:non_exact}
\label{appx:extension}

In this section, we discuss the extension of our algorithm to the setting with $(\veps, \delta)$-PAC policy identification with $\veps >0$. In this case, we let $\Pi_\veps^\star(\mdp)$ be the set of (Markov, deterministic and stochastic) policies $\pi$ such that 
$$ \max_{\pi'} V^{\pi'}_0 \leq V^\pi_0 + \veps \:.$$ The goal of the agent is then to identify a policy $\pi^\star \in \Pi_\veps^\star(\mdp)$  with high probability, i.e., given an error rate $\delta$, her final recommendation $\hat \pi^\tau$ at stopping time $\tau$ should satisfy $$ \bP_\mdp(\tau <\infty, \hat\pi_\tau \notin \Pi_\veps^\star(\mdp)) \leq \delta.$$ 

\citet{almarjani2023instanceoptimalityonlinepacreinforcement} derived a lower bound for this setting, which we state below. To introduce this lower bound, we introduce the following notation. For a policy $\pi$ and given $\veps\geq 0$ let $\alt_\veps^\pi(\mdp) \ceq \Bigl \{  \mdp' : \mdp \ll \mdp' \text{ and } \pi \notin \Pi^\star_\veps(\mdp') \Bigr\}$. In words, it is the set of $\mdp'$, for which $\mdp\ll \mdp'$ and where the $\pi$ is not a near-optimal policy. 

\begin{theorem}
Fix an $\mdp \in \Mdp$. For any adaptive stopping time $\tau$ for $(\veps, \delta)$-PAC policy identification 
	$$ \liminf_{\delta \to 0} \frac{\bE_\mdp[\tau]}{ \log \frac 1\delta} \geq \Gamma_\veps(\mdp) \:, $$ where 
	$$ \Gamma_\veps^{-1} (\mdp)\ceq \max_{\pi \in \Pi_\veps^\star(\mdp)} \max_{w \in \Omega_\mdp} \inf_{\mdp' \in \alt_\veps^\pi(\mdp)} \lsb \sum_{s,a,h} w_h(s,a) \KL(\mdp, \mdp')_{s,a,h} \rsb\:.$$
\end{theorem}
Intuitively, the lower bound states that the optimal sample complexity should scale with the easiest to identify policy among the $\veps$-optimal set of policies. From previous works in bandits, asymptotically matching this bound often required "sticking procedures'' to stick to an answer to verify. Indeed, for $\veps=0$ and $\lvert \Pi_\veps^\star(\mdp)\rvert = 1$, sticking at each episode, the algorithm's sampling rule is meant to collect evidence for assessing the true optimality of its current best policy guess $\hat\pi_t$. For $\veps>0$, it is likely that at episode $t$, $\lvert \what \Pi_{t, \veps} \rvert >1$. The challenge is then to pick which answer in this set the sampling rule should aim to verify. 

Following the approach developed in Appendix~\ref{appx:posterior_convergence} it's simple to prove that 

\begin{theorem}
\label{thm:opt_post_contra_veps}
Fix an MDP $\mdp \in \Mdp$ and let $\nu_t$ denote the posterior distribution introduced in Eq.\eqref{eq:posterior_def}. For any adaptive algorithm for best policy identification, with probability one,  
$$
    \limsup_{t \to \infty} 
    -\frac{1}{t} 
    \log \lp \min_{\pi\in \Pi_\veps^\star(\mdp)}\bP_{\wt\mdp \sim \nu_t \mid \cH_{t-1}}\!\bigl(\pi \notin \Pi_\veps^\star(\mdq)\bigr) \rp
    \;\leq\; \Gamma_\veps(\mdp).
$$
\end{theorem}

\bigskip 

To state the modification of the Algorithm~\ref{alg:repit}, let a generic sticking procedure take as input the empirical $\wh\mdp$ and output a policy in $\what \Pi_{t, \veps}$. We require the procedure to be deterministic and stable, that is if the sticking recommends an answer at time $t$, then as long as this answer belongs to $\what \Pi_{t', \veps}$ for $t'\geq t$, it should keep recommending it. In case of ties, the sticking break rules with consistency, so that this defines a strict total order on $\Pi^\star(\mdp) $. 

\medskip

Because of the smooth forced exploration, we should have that on a good event $\what \Pi_{t, \veps} \to  \Pi_{\veps}^\star(\mdq) $ so that the sticking rule is simply a rule to select an answer to collect samples to verify it. We restate below the slight modification of Algorithm~\ref{alg:repit} to account for the $\veps>0$ case. Remark that compared to the initial algorithm, only Line~\ref{eq:line_stick} has changed, to pick an answer to verify among the set of empirically admissible answers.  

Intuitively, the sticking procedure transforms the problem into verifying a single answer, as opposed to verifying multiple correct answers.

As the algorithm sticks to an answer $\pi$, the conditional posterior sampling is now a no-regret learner on $\alt_\veps^\pi(\mdp)$. Following the same lines as in previous sections, one could derive the no-regret property of the conditional posterior sampling on $\alt_\veps^\pi(\mdp)$. Combining these observations, the following will hold with probability one 


\paragraph{A posterior-based stopping time for the multiple correct answers setting.} We propose the following stopping rule for the $(\veps, \delta)$-PAC policy identifcation with $\veps>0$.  
\begin{equation}
\label{eq:sto-mca}
\tau  \ceq \inf \lb t\geq 1 : \exists \pi \in \wh\Pi_{t, \veps} \text{ such that } \forall \; k\; \leq B(t, \delta), \pi \in \Pi^\star_\veps(\mdq_{t, k}) \rb . 
\end{equation}
In words, the stopping rule triggers when there exists an answer for which finding a challenger by conditional resampling takes more than $B(t,\delta)$ samples.  
Observe that for this stopping rule, we have 
\begin{eqnarray*}
	 \bP(\tau > t \mid \cH_{t-1})  &\leq&  \bP(\forall  \; \pi\, \in \,\wh\Pi_{t, \veps}, \exists\; k\, \leq\, B(t, \delta) :  \pi \notin \Pi^\star_\veps(\mdq_{t, k}) \mid \cH_{t-1} ) \\
	 &\leq& \min_{\pi \in \wh\Pi_{t,\veps}}  \bP(\exists\; k\, \leq \,B(t, \delta) :  \pi \notin \Pi^\star_\veps(\mdq_{t, k}) \mid \cH_{t-1} ) \\
	 &\leq& B(t,\delta) \min_{\pi \in \wh\Pi_{t, \veps}}  \bP(\mdq_{t} \in \alt^\pi_\veps(\empmdp_{t}) \mid \cH_{t-1} ) \:,  
\end{eqnarray*} 
where $\mdq_{t}, \mdq_{t, 1}, \ldots, \mdq_{t, k}  $ are \iid. 

By the concentration of the posterior distribution, we would have with probability one, 
\begin{eqnarray*}
	\bP(\tau > t \mid \cH_{t-1})  &\lessapprox &\wt O\!\lp B(t,\delta) \min_{\pi \in \wh\Pi_{t,\veps}} \exp \lb - \inf_{\mdq \in \alt^\pi_\veps(\empmdp_{t}) } \lsb \sum_{s,a,h} n_{t}(s,a, h) \KL(\empmdp_{t}\| \mdq)_{s,a,h}\rsb\rb \rp \\ 
	&\lessapprox &\wt O\!\lp B(t,\delta)  \exp \lb - \max_{\pi \in \wh\Pi_{t,\veps}}\inf_{\mdq \in \alt^\pi_\veps(\empmdp_{t}) } \lsb \sum_{s,a,h} n_{t}(s,a, h) \KL(\empmdp_{t}\| \mdq)_{s,a,h}\rsb\rb\rp
\end{eqnarray*} 
where $\tilde\cO$ hides constant and terms that are at most polynomial in $t$. As the empirical estimates converge, the term in the exponent will be equivalent (up to logarithmic terms in $t$)  
$$  \max_{\pi \in \Pi^\star_\veps(\mdp)}\inf_{\mdq \in \alt^\pi_\veps(\mdp) } \lsb \sum_{s,a,h} {n_{t}(s,a, h)} \KL(\mdp\| \mdq)_{s,a,h}\rsb $$

which is larger than evaluating it at any response computed by the sticking procedure. Combining the convergence of the algorithm with "sticking'' with the stopping time $\tau$, one can prove results similar to the previous section in terms of sample complexity. 
It is also possible to ensure that if the sampling rule guarantees saddle-point convergence, then plugging in this stopping rule will ensure optimality. We leave as an open question the design of a converging, computationally efficient sampling rule for the setting with $\veps > 0$.

Given $m$ posterior samples, checking if the stopping triggers require solving the feasibility problem 

 $$ \max_{ \pi \in \wh\Pi_{t,\veps}} \min_{k\leq m} (\wt V_{t,0}^{k, \pi} + \veps - \wt V_{t,0}^{k, \star})> 0 $$
 and solving this might ultimately require enumerating $\wh\Pi_{t, \veps}$, which is a major caveat.  For the computation of the stopping rule.

\begin{algorithm}[h!]
\DontPrintSemicolon
\LinesNumbered
\SetAlgoLined
\SetKwInOut{Input}{Require}
\SetKwInOut{Output}{Output}
\Input{Initial exploration $\bpol_1 = \{\bpol_1(\cdot | s, h)\}_{s, h}$; initial forced exploration $c_1$; mixing parameter $\gamma \in (0,1)$; clipping parameter $\alpha \in (0,\tfrac12)$; posterior inflation $(\eta_t)_{t\geq 1}$; initial MDP $\empmdp_{0} \in \Mdp$ (arbitrary); $\veps\geq 0$} 
\BlankLine
\For{episode $t = 1,2,\ldots$}{
\tcp*{Stick to an answer among current empirical $\veps$-best guesses}

  Compute $\hat\pi_{t} \gets \textsc{sticking}(\empmdp_{t}, \veps)$\; \label{eq:line_stick}
  \BlankLine
  \tcp{Computing a challenger MDP}
  \For{$k = 1,2,\ldots$}{\label{algo:sample-challenger}
    Sample candidate challenger $\mdq_{t, k} \sim \nu^{\eta_t}_{t}$\;
    Compute the optimal $Q$-function $\;\wt Q_{t}^{k,\star}$ of $\mdq_{t, k}$ by backward induction\;
    \If{$\hat\pi_{t}$ is $\veps$-suboptimal for $\;\wt Q_{t}^{k,\star}$}{
      Set challenger $\mdq_{t} \gets \mdq_{t, k}$ and \textbf{break}\;
    }
  }
  \BlankLine
  
 Compute $c_{t}$ as in Algorithm~\ref{alg:appx_forcexp}\; 
    
  Draw $Z_t \sim \mathrm{Bernoulli}(t^{-\gamma})$ and define $\forall (s,h)\in \cS \times [H]$ 
  $$
     \bpolmix_t(\cdot | s, h) \;\leftarrow\;
     \begin{cases}
       c_t(\cdot | s, h), & \text{if } Z_t = 1,\\
       \bpol_t(\cdot | s, h), & \text{if } Z_t = 0,
     \end{cases}
  $$
  \BlankLine
  \tcp{Rollout and data collection}
  Set $s_1^t \gets s_{\mathrm{init}}$\;
  \For{$h = 1$ \KwTo $H$}{
    Sample action $a_h^t \sim \bpolmix_t(\cdot | s_h^t, h)$\;
    Observe reward $r_h^t$ and transition to $s_{h+1}^t$\;
  }
  \BlankLine
  \tcp{KL-based bonuses for the online learner}
  Construct bonus kernel $\loss_t(s,a, h)$ from $\empmdp_{t}$ and $\mdq_{t}$ as in Eq.\eqref{eq:main-loss-def}\; 
  \BlankLine
  \tcp{Behaviour policy improvement}
  Update main exploration policy $\bpol_{t+1} \gets \textsc{ComputeExplorationPolicy}\bigl(t, \bpol_t,\, \loss_t,\, \hat p_t,\, n_t \bigr)$\; 

}
\caption{$\veps$-\textsc{PIPS}: Policy Identification via Posterior Sampling in Tabular MDPs  \label{alg:repit_veps}}
\end{algorithm}

\section{IMPLEMENTATION DETAILS AND ADDITIONAL EXPERIMENTS}
\label{appx:add_exp_impl_details}
This section contains implementation details and additional experimental results comparing\algoname{} to PSRL. 

\subsection{Implementation Details}
\label{appsubsec:impl_details}
In this section, we provide details on the implementation and experimental setup for reproducibility.  

\paragraph{Experimental setup.}  
For each environment shown in Figure~\ref{fig:mocamdp} and Figure~\ref{fig:riverswim}, rewards are modeled as Gaussian random variables with means specified in the figures. In the \textsc{Riverswim} environment, the reward variance is set to $1/20$, while in the MOCA instance it is set to $1/4$.  Although both environments are stationary, we did not include this information in our algorithm, which allows us to estimate the model and maintain distributions over non-stationary MDPs. 

\paragraph{Online learner.}  
For the policy improvement step of the online learner, we use AdaHedge, implementing the pseudo-code described in \citet{rooij_ada_hedge}, which automatically adjusts its learning rate at each round. We recall that the initialization of Algorithm~\ref{alg:appx-ol} with a generic online learner is described and analyzed in Appendix~\ref{appx:sup_player}, with its pseudo-code in Algorithm~\ref{alg:appx-ol}. 

In practice, we observe little difference in performance compared to the natural Hedge update. As argued in Appendix~\ref{appx:sup_player}, both algorithms ensure sublinear regret and thus preserve the theoretical guarantees of Algorithm~\ref{alg:repit}. 

\paragraph{Posterior inflation and thresholds calibration.}  
The posterior is inflated using $\eta_t = 1/(t+1)^4$, which provides a good trade-off between empirical performance and runtime. Running without inflation (i.e., $\eta_t = 1$) tends to yield slower experiments without noticeable benefits. However, for simplicity, we recommend using $\eta_t \gets 1$ with $B(t,\delta) \approx \tfrac{SAH}{\delta}\log(t SAH/\delta)$ which experimentally ensure $\delta$-correctness. 

\paragraph{Other algorithms implementation.}
For each algorithm, we refer to the pseudo-code in the original paper for the implementation. For the UCBVI algorithm \citep{pmlr-v70-azar17a} and SSR \citep{NEURIPS2022_298c3e32} we implement the following idealized Bernstein confidence bonuses 
$$ \bonus_t(s,a, h) = \min \lp \sqrt{\frac{\var_{\hat p_t}[\bar V_{t, h+1}](s,a) \log(t+1)}{\max(1, n_{t}(s,a, h))}} + (H-h+1)\frac{\log(t+1)}{\maxp{1}{ n_{t}(s,a, h)}}, H-h+1\rp\:. $$ 
For PSRL, the prior distribution is the same as for Algorithm~\ref{alg:repit}.

\subsection{Additional Experiments}
\label{appsubsec:add_exp}

We present further experimental results on the MOCA instance. Specifically, we study the number of posterior samples required before identifying an alternative model. This experiment evaluates the behavior of the posterior stopping time. Since both\algoname{} and PSRL rely on posterior sampling, we compare their performance in this setting.  

Each algorithm is run on the MOCA instance for $T=10{,}000$ episodes. At every episode, we record the number of posterior resamples needed to obtain an alternative model. For PSRL, this quantity is immaterial, but for \algoname{} it directly influences the sampling rule. We report the average (over $50$ independent runs). 

The results in Figure~\ref{fig:appx-samples-mocaexp} show that for \algoname{}, the number of resamples grows exponentially with $t$, as further confirmed by the log-plot in Figure~\ref{fig:appx-samples-mocaexp-log}. This behavior is consistent with theory: conditionally on the history, the number of resamples follows a geometric law with parameter scaling with $\myexp{-t\Gamma}$, so its expectation grows exponentially. This indicates that evidence accumulates rapidly, enabling the algorithm to certify the optimal policy.  

In contrast, PSRL exhibits only polynomial growth in the number of resamples (also evident from the log-log plot in Figure~\ref{fig:appx-samples-mocaexp-loglog}). This suggests that its posterior concentrates at a polynomial rate. 

\begin{figure}[H]
  \centering
  \begin{minipage}[b]{0.32\textwidth}
    \centering
    \includegraphics[width=\linewidth]{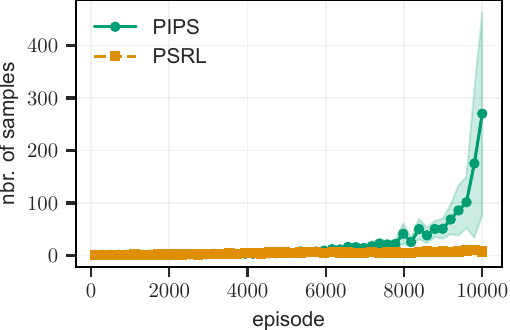}
    \caption{Number of samples for finding an alternative model averaged on 100 runs.}
    \label{fig:appx-samples-mocaexp}
  \end{minipage}
  \begin{minipage}[b]{0.32\textwidth}
    \centering
    \includegraphics[width=\linewidth]{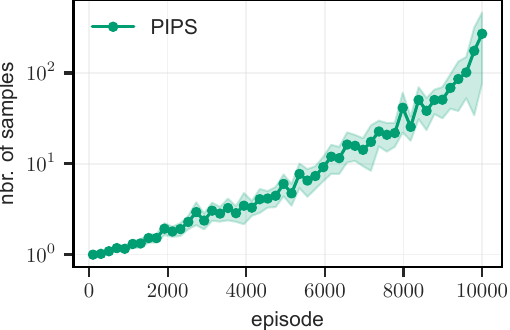}
    \caption{Log number of samples for finding an alternative model averaged on 100 runs.}
    \label{fig:appx-samples-mocaexp-log}
  \end{minipage}
  \hspace{0.015cm}
   \begin{minipage}[b]{0.32\textwidth}
    \centering
 \includegraphics[width=\linewidth]{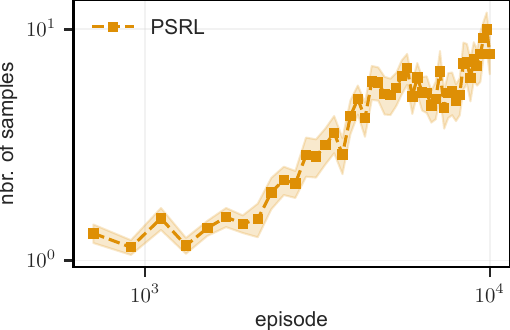}
    \caption{Log-log plot of number of samples for finding an alternative model averaged on 100 runs. }
    \label{fig:appx-samples-mocaexp-loglog}
  \end{minipage}
\end{figure}

\begin{figure}[H]
	\centering
        \includegraphics[width=0.9\linewidth]{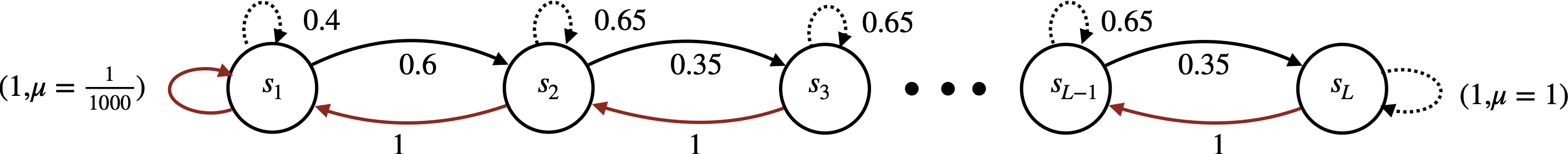}
    \caption{The $L$-state RiverSwim environment. Each state allows two actions: \textcolor{red}{\textsc{left}} (red arrows) and \textsc{right} (black arrows). Rewards are placed at the two ends of the chain. Transition probabilities are annotated above the corresponding arrows.}
    \label{fig:riverswim}
\end{figure}

\end{document}